\documentclass[runningheads]{llncs}
\RequirePackage{silence}  
\usepackage{graphicx}
\usepackage{comment}
\usepackage{amsmath,amssymb}
\usepackage{color}
\usepackage{url}
\usepackage{hyperref}
\usepackage{cleveref}

\usepackage{graphicx}
\usepackage{booktabs}
\usepackage{url}

\usepackage[dvipsnames]{xcolor}
\usepackage{graphicx} 
\usepackage{subcaption} 
\usepackage{amsmath,bm}
\usepackage{amssymb}
\usepackage{subcaption}
\usepackage{wrapfig}
\usepackage{caption}
\usepackage{multirow}
\usepackage{booktabs}
\usepackage{makecell}

\usepackage[export]{adjustbox}
\usepackage{array}

\newtheorem{assumption}{Assumption}

\newcommand{\myparagraph}[1]{\par\noindent\textbf{#1}}

\usepackage{longtable}

\newif\ifreview
\reviewfalse

\ifreview
	\usepackage{lineno}

	\linenumbers
\fi

\begin{document}

\def\SubNumber{26}

\def\GCPRTrack{Main Track}

\title{Unlocking Pretrained Vision Transformers for Time Series Classification}

\ifreview
\else

	\author{Simon Roschmann\inst{1,2,3,4}\thanks{Corresponding Author} \and
	Quentin Bouniot\inst{1,2,3,4,5} \and
	Vasilii Feofanov\inst{6,7} \and \\
    Ievgen Redko\inst{6,8} \and
    Zeynep Akata\inst{1,2,3,4}}
	
	\authorrunning{S. Roschmann et al.}
	
	\institute{Helmholtz Munich, Germany \and Technical University of Munich, Germany \and Munich Center for Machine Learning, Germany \and Munich Data Science Institute, Germany \and LTCI, Télécom Paris, Institut Polytechnique de Paris, Palaiseau, France \and Noah’s Ark Lab, Paris, France \and 42.com, Amsterdam, Netherlands \and Schneider Digital - AI Hub, Schneider Electric, Paris, France
    \email{\{simon.roschmann,quentin.bouniot,zeynep.akata\}@tum.de},
    \email{ievgen.redko@se.com},
    \email{vfeofanov@42.com}}
\fi

\maketitle              %

\begin{abstract}
Adapting vision models for time series analysis is compelling, yet all existing approaches are falling short of dedicated time series foundation models (TSFMs) in classification.
In this work, we propose \textbf{Ti}me \textbf{Vi}sion \textbf{T}ransformer (\textbf{TiViT}), the first framework that successfully unlocks the representational power of frozen Vision Transformers~(ViTs) pretrained on large-scale image datasets for time series classification. TiViT achieves state-of-the-art performance without any finetuning by utilizing the hidden representations of OpenCLIP models.
We explore the structure of TiViT representations and find that intermediate ViT layers with high intrinsic dimension are the most effective for time series classification. Furthermore, we assess the alignment between TiViT and TSFM representation spaces and identify a strong complementarity, with additional performance gains achieved through feature concatenation.
Finally, we unfreeze the ViT backbone of TiViT for continual pretraining and contrastive alignment with TSFMs on time series, enhancing the performance of lightweight TiViT variants. Our findings reveal a new direction for the domain and task adaptation of vision foundation models. Code is available at \href{https://github.com/ExplainableML/TiViT}{https://github.com/ExplainableML/TiViT}.
\end{abstract}

\section{Introduction}

Time series capture critical information in healthcare, finance, and manufacturing.
Inspired by the success of foundation models in natural language processing (NLP) and computer vision (CV), similar models have recently been developed for the analysis of time series, following two different approaches. The first one is to pretrain time series foundation models (TSFMs) in a self-supervised way \cite{ansari2024chronos,das2023decoder,feofanov2025mantis,goswami2024moment,lin2023nutime} using a large-scale real-world time series dataset. The second one is to repurpose foundation models from other domains, such as NLP \cite{jin2023time,zhou2023one} and CV \cite{chen2024visionts,li2023time}, for time series tasks. The idea behind this approach is to benefit from the vast amount of samples that large vision and language models are trained on, and which are often unavailable in the time series domain.

\begin{figure}[t]
\centering
\includegraphics[width=0.98\textwidth]{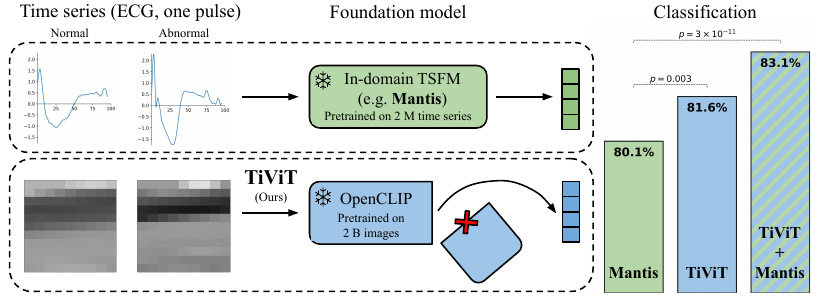}
\caption{While TSFMs operate directly on the 1D time series signal, TiViT transforms time series into images to leverage pretrained ViTs for feature extraction. 
We display the average time series signal of two classes from ECG200 \cite{ecg200} and their corresponding 2D representations. 
Utilizing the hidden representations of OpenCLIP, TiViT significantly outperforms Mantis in classification on the UCR benchmark. Combining both models further improves accuracy.}
\label{fig: teaser}
\end{figure}

Adapting vision models to time series analysis is particularly compelling, since time series can be visualized as line plots, heatmaps, or spectrograms \cite{ni2025harnessingvisionmodelstime}. TimesNet~\cite{WuHLZ0L23} has been trained end-to-end on heatmaps generated from time series, pretrained Masked Autoencoders have been applied in zero-shot time series forecasting  \cite{chen2024visionts}, and pretrained SwinTransformers have been finetuned on line plots of irregularly sampled time series \cite{li2023time}. However, these approaches are either restricted to the task of forecasting or require costly per-dataset training and finetuning. No prior work has shown that vision and vision-language models trained on billions of images can be state-of-the-art in time series classification.

\myparagraph{TiViT Overview.} In this work, we demonstrate that frozen vision foundation models such as CLIP \cite{ilharco2021openclip,radford2021clip}, SigLIP~2 \cite{tschannen2025siglip2}, and DINOv3 \cite{simeoni2025dinov3}, pretrained solely on natural images or image-text pairs, can serve as feature extractors for time series classification without any pretraining or finetuning on time series data. We further show that these vision representations are complementary to the representations of TSFMs and that they can be strengthened through continual pretraining or contrastive alignment with TSFMs on time series.

Our main contributions are summarized as follows:
\begin{itemize}
\item We introduce the Time Vision Transformer (TiViT), leveraging hidden representations of pretrained and frozen ViTs for time series classification. TiViT surpasses not only previous vision-based approaches but also conventional TSFMs without any finetuning across 158 time series datasets.
\item We study the alignment of TiViT and TSFMs and find that they extract complementary information from time series. By merging their representations, we achieve an average improvement of $+3\%$ over TSFMs. 
\item We unfreeze the ViT backbone of TiViT for continual pretraining and contrastive alignment with TSFMs on synthetic time series. Through such adaptation, we enable the most lightweight TiViT variant (36 M parameters) to close the gap to state-of-the-art classification performance.
\end{itemize}

\section{Related Work}
\label{sec:related}
\myparagraph{Time Series Foundation Models.} Over the past few years, the number and variety of TSFMs has rapidly increased. At first, such models were based on repurposing large language models (LLMs) for time series tasks \cite{cao2023tempo,chang2023llm4ts,gruver2024large,jin2023time,xue2023promptcast,zhou2023one}, exploiting the ability of LLMs to efficiently handle sequential data. More recently, a complementary line of research has trained TSFMs from the ground up on extensive and diverse time series datasets \cite{ansari2024chronos,bhethanabhotla2024mamba4cast,das2023decoder,feofanov2025mantis,gao2024units,goswami2024moment,lin2023nutime,liu2024moirai,liu2024timer,rasul2023lag,wang2024rose}. While most of these models were designed for time series forecasting, several of them also tackled time series classification \cite{feofanov2025mantis,gao2024units,goswami2024moment,lin2023nutime,zhou2023one}. TSFMs are pretrained once and can then be transferred across various datasets and tasks. They are on par with or exceed the performance of earlier deep learning models such as TimesNet~\cite{WuHLZ0L23}, which required training from scratch on every dataset.

\myparagraph{Vision Approaches for Time Series.} Time series can be transformed into images in various ways, e.g., by rendering the signal in the time domain (line plot) or frequency domain (spectrogram), or by using a 2D modeling approach (heatmap, Gramian angular field, recurrence plot) that stacks segments of the signal based on a chosen periodicity. 

Vision models, often based on CNNs and their variations, were applied on such image-based representations of time series as early as 2013 \cite{ni2025harnessingvisionmodelstime}. Most of them, however, were trained in a supervised way to fit the dataset at hand. This work explores how pretrained vision models can be used as powerful feature extractors without pretraining or finetuning on time series data. 
Li et al. \cite{li2023time} showed that pretrained ViTs can be effective in the classification of irregular time series from their line plot representations after full finetuning. In a similar vein, Chen et al.~\cite{chen2024visionts} applied a masked auto-encoder with a pretrained frozen ViT to 2D transformed time series to perform time series forecasting.

Different from these works, our TiViT model surpasses the performance of frontier TSFMs across a broad set of classification benchmarks. Moreover, we show that pretrained ViTs serve as a powerful initialization for time series encoders and can be strengthened through continual pretraining on time series. While prior work has considered 1D and 2D modeling approaches separately, we are the first to identify and leverage the complementarity of vision and time series foundation models.

\section{Time Vision Transformer}
\label{sec:tivit}
We introduce the Time Vision Transformer (TiViT) leveraging pretrained frozen ViTs from the vision or vision-language domain for time series classification. 
We consider a multivariate time series dataset $\mathcal{T} = \{\bm{t}^n | \bm{t}^n \in \mathbb{R}^{T \times D}\}_{n=1}^N$ containing $N$ samples, each of length $T$ and dimensionality $D$. The corresponding targets $\mathcal{Y}=\{y^n\}_{n=1}^N$ are labels $y^n \in \{1, ..., C\}$ from $C$ different classes. We transform the time series into images and apply ViTs on these images to extract representations for linear classification. \Cref{fig: methodology} illustrates our approach.

\begin{figure}[!t]
\centering
\includegraphics[width=0.98\textwidth]{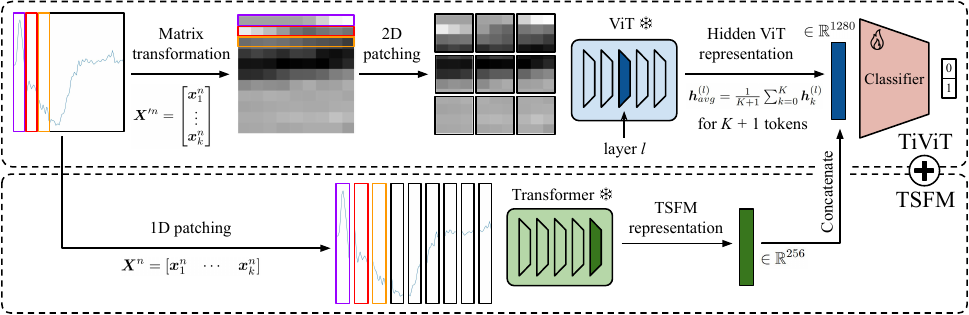}
\caption{Illustration of TiViT on a time series sample from ECG200 \cite{ecg200}. We split the time series into segments and stack them to form a grayscale image. Then, we patch the image in 2D and feed it into a frozen ViT pretrained on large-scale image datasets. We average the representations from a specific hidden layer and pass them to a learnable classification head. Concatenating the representations of TiViT and TSFMs prior to classification further improves accuracy.}
\label{fig: methodology}
\end{figure}

\subsection{Time Series-to-Image Transformation}
\label{subsection: time series to image transformation}
Following the channel independence assumption, proposed by Nie et al.~\cite{nie2022patchtst} and widely adopted in most recent TSFMs \cite{feofanov2025mantis,goswami2024moment}, we first split a multivariate time series $\bm{t}^n \in \mathbb{R}^{T \times D}$ into $D$ univariate time series $\{\bm{t}^n_d \in \mathbb{R}^{T}\}_{d=1}^{D}$. We then normalize each univariate time series $\bm{t}^n_d$ using robust scaling, defined as: $\frac{\bm{t}^n_d - Q_2}{Q_3 - Q_1},$ where $Q_1, Q_2, Q_3$ are the first, second (median), and third quartiles, respectively. We apply padding at the beginning of each time series by replicating its first value and subsequently segment it into $M$ patches $\{\bm{x}_m\}_{m=1}^{M}$ of size $P$.
Given a patch length $P$ and stride $S$, the total number of patches is: $M = \left\lfloor{\frac{T - P}{S}} \right\rfloor + 1$. We stack the patches to generate a 2D representation $\bm{X}' \in \mathbb{R}^{M \times P}$, which we then render into a grayscale image $\bm{X}' \in \mathbb{R}^{M \times P \times 3}$ by replicating its signals across three channels. To align with the square input resolution $(R, R)$ expected by the ViT, we resize the image.

\subsection{Time Series Classification with Frozen ViTs}
\label{subsection:classification_with_frozen_vits}
We feed each grayscale image $\bm{X}'$ representing a univariate time series into a pretrained and frozen ViT $v$ with $L$ hidden layers. The ViT inherent 2D patching yields a sequence $\{{\bm{x}_k'} \in \mathbb{R}^{U^2}\}_{k=1}^K$ of flattened patches where $(U, U)$ is the resolution per patch and $K = R^2/U^2$ is the resulting number of patches. ViTs generally prepend a classification token to this sequence. The ViT consumes all input tokens and produces a sequence of features at every layer: $v(\bm{X}')=\left\{[\bm{h}_0^{(l)}, \bm{h}_1^{(l)}, ..., \bm{h}_K^{(l)}]\right\}_{l=0}^{L}$.
To obtain a single embedding vector $\bm{e}$ per image, we select a specific layer $l$ and average its $K+1$ representations: $\bm{e} = \bm{h}_{avg}^{(l)} = \frac{1}{K + 1} \sum_{k=0}^{K} \bm{h}_k^{(l)}$. For multivariate time series, we feed per-channel image representations $\{\bm{X}'_d\}_{d=1}^D$ separately into the ViT and concatenate the resulting embeddings for a specified layer: $\text{Concat}(\bm{e}_1, ..., \bm{e}_D)$. We only train a linear classifier on the ViT representations and their corresponding class labels. To enhance performance, the embeddings of frozen TSFMs and TiViT can be concatenated prior to classification.

\subsection{Continual Pretraining on Time Series}
\label{subsection:continual_pretraining}

To further reduce the domain gap between natural images and time series, we continually adapt TiViT in a self-supervised fashion on synthetic time series data generated by CauKer \cite{xie2026cauker}. %
We follow the self-supervised contrastive pretraining of Mantis~\cite{feofanov2025mantis} to make TiViT invariant to specific distortions of the time series. Two augmentations of the same time series should yield similar embeddings (positive pair) while embeddings of different time series should be well separated (negative pairs).

We truncate the pretrained ViT $v$ at layer $l$, selected via linear probing of the frozen encoder (see \Cref{subsection: hidden representations}), and attach a non-linear projection head $h:\mathbb{R}^{(K+1)\times d}\to\mathbb{R}^{d'}$.
This head consists of a lightweight Transformer with a learnable \texttt{[CLS]} token followed by an MLP with one hidden layer. Given a batch of time series $\mathcal{B}=\{\bm{t}_i\}_{i=1}^{B}$, we sample two augmentation functions $\phi,\psi\sim\mathcal{A}$ and apply them on each time series. The set of augmentations $\mathcal{A}$ is based on RandomCropResize with crop rate $c \in [0, 0.2]$. We transform the augmented time series into images and compute $\ell_2$-normalized embeddings:

\begin{equation}
\bm{z}_i^{\phi} \;=\; 
\frac{h\!\left(v\!\left(f(\phi(\bm{t}_i))\right)\right)}{\left\|h\!\left(v\!\left(f(\phi(\bm{t}_i))\right)\right)\right\|_2},
\qquad
\bm{z}_i^{\psi} \;=\; 
\frac{h\!\left(v\!\left(f(\psi(\bm{t}_i))\right)\right)}{\left\|h\!\left(v\!\left(f(\psi(\bm{t}_i))\right)\right)\right\|_2}.
\end{equation}

We construct the similarity matrix $S\in\mathbb{R}^{B\times B}$, where $S_{ij}=(\bm{z}_i^{\phi})^\top \bm{z}_j^{\psi}$, and minimize the symmetric InfoNCE \cite{oord2018infonce} objective:

\small
\begin{equation}
\mathcal{L}
\;=\;
\frac{1}{2}\Big(
\mathrm{CE}(S/\tau) \;+\; \mathrm{CE}(S^\top/\tau)
\Big),
\quad
\mathrm{CE}(A) \;=\; -\frac{1}{B}\sum_{i=1}^{B}\log\frac{\exp(A_{ii})}{\sum_{j=1}^{B}\exp(A_{ij})}.
\label{eq:symmetric_infonce}
\end{equation}
\normalsize

Here, $\tau$ is a learnable temperature initialized as $\tau=0.1$. We keep the ViT backbone largely frozen and update only a small set of parameters: the projection head $h$, LoRA adapters inserted into the ViT blocks, and the layer normalization parameters. After continual pretraining, we discard $h$ and use the aggregated backbone representation $\bm{h}_{avg}^{(l)}$ in downstream tasks.

\subsection{Contrastive Alignment with TSFMs}
\label{subsection: alignment and distillation}

TiViT and TSFM representations can be combined via concatenation at the classifier level (see \Cref{subsection:classification_with_frozen_vits}). However, such an ensemble approach requires two forward passes at inference. To reduce this overhead, we align the representations of TiViT and TSFM through training such that a single model alone produces enriched features and can be deployed independently. Since TiViT and TSFM impose different inductive biases, their original representations can be considered different views of the same underlying time series signal. We map these views into a joint embedding space and optimize the contrastive InfoNCE objective.

Let $v$ and $m$ denote TiViT and TSFM, and let $h_v$ and $h_m$ be their non-linear projection heads. For a batch of time series $\mathcal{B} = \{\bm{t}_i\}_{i=1}^{B}$, we compute
$\bm{z}_i^{\mathrm{TiViT}} = h_v\!\left(v(f(\bm{t}_i))\right)$ and $\bm{z}_i^{\mathrm{TSFM}} = h_m\!\left(m(\bm{t}_i)\right)$, and apply $l_2$-normalization. Then, we construct the similarity matrix $S_{ij}=(\bm{z}_i^{\mathrm{TiViT}})^\top \bm{z}_j^{\mathrm{TSFM}}$ and minimize the symmetric InfoNCE loss $\mathcal{L}$ from \Cref{eq:symmetric_infonce}.

\section{Experimental Evaluation}\label{sec:exps}

We evaluate TiViT in time series classification, showing its state-of-the-art performance compared to supervised, self-supervised, and TSFM baselines. We further investigate the factors underlying its performance: the role of hidden ViT representations, the complementarity with TSFMs, the time series characteristics it captures, and the data regimes in which it excels.

\subsection{Setup}

\myparagraph{Data.} The evaluation comprises 128 univariate time series datasets from the UCR benchmark~\cite{dau2019ucr}, 27 multivariate datasets from the UEA benchmark~\cite{bagnall2018uea}, and 3 multivariate EEG datasets (PCL, CAP, SEDFx) from WOODS \cite{gagnonaudet2023woods}.

\myparagraph{TiViT.} Our study examines three differently pretrained ViTs: OpenCLIP~\cite{Cherti_2023_CVPR,ilharco2021openclip}, SigLIP~2~\cite{tschannen2025siglip2}, and DINOv3~\cite{simeoni2025dinov3}. We consider multiple vision model sizes (ViT-B, ViT-L, ViT-H) with varying layer depth (12, 24, and 32 layers).

\myparagraph{Baselines.} We compare TiViT to the TSFMs Mantis \cite{feofanov2025mantis} and Moment \cite{goswami2024moment} which are exclusively pretrained on time series. We further consider GPT4TS~\cite{zhou2023one} pretrained on textual data, vision-based approaches such as VisionTS \cite{chen2024visionts} and TimesNet \cite{WuHLZ0L23}, forecasting TSFMs such as VisionTS~\cite{chen2024visionts} and Chronos~\cite{ansari2024chronos}, and a wide range of (self-)supervised baselines (pre-)trained per time series dataset: ResNet \cite{wang2017baseline}, FCN \cite{wang2017baseline}, DTW \cite{dau2019ucr}, CNN \cite{zebik2017cnn}, MLP \cite{wang2017baseline}, Encoder \cite{serra2018encoder}, TWIESN \cite{tanisaro2016twiesn}, MCNN \cite{cui2016mcnn}, TimesNet \cite{WuHLZ0L23}, TS2Vec \cite{yue2022ts2vec}, T-Loss \cite{franceschi2019tloss}, TS-TCC \cite{eldele2021tstcc}, TNC \cite{tonekaboni2021tnc}, and TST \cite{zerveas2021TST}.

\myparagraph{Implementation.} To evaluate the effectiveness of TiViT and TSFM representations in time series classification, we train a logistic regressor with the LBFGS solver per dataset. Our evaluation adheres to the standard train-test splits provided by the UCR and UEA archive and reserves 20\% of the train split for validation. All experiments with frozen encoders are performed on a single NVIDIA V100 GPU with 16 GB memory.
A detailed overview of our experimental setup is provided in the Appendix.

\subsection{Transforming Time Series into Images} %
\label{subsection: transforming time series into images for vit feature extraction}

\begin{wrapfigure}[14]{r}{0.4\linewidth}
\centering
\vspace{-0.7cm}
\includegraphics[width=\linewidth]{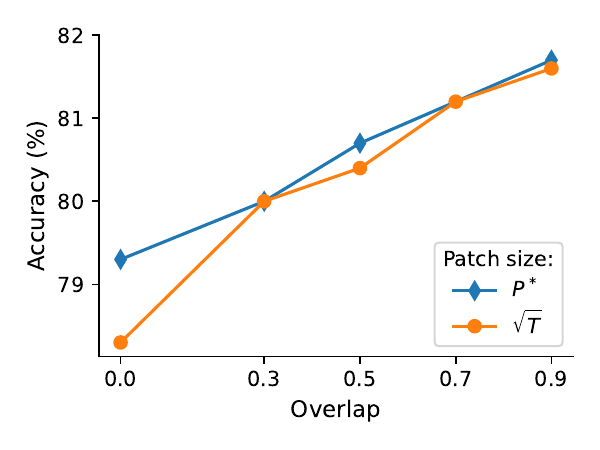}
\caption{Effect of patch size and overlap on UCR classification accuracy.}
\label{fig:overlap_accuracy_sqrt_opt}
\end{wrapfigure}

The performance of our time series-to-image transformation is sensitive to the time series patch size $P$, as extreme values can create redundant visual tokens during resizing to the ViT input resolution. To avoid a computationally expensive hyperparameter search for the optimal patch size $P^*$ per dataset, we propose the heuristic $P = \sqrt{T}$ for any series of length $T$. This choice yields a square-shaped matrix representation prior to resizing, which minimizes horizontal or vertical distortion and thus preserves patch diversity.
While an exhaustive search for $P^*$ offers a marginal accuracy improvement in the case of no overlap, our heuristic provides a strong baseline at a fraction of the computational cost. 
As displayed in \Cref{fig:overlap_accuracy_sqrt_opt}, introducing overlap between patches further boosts performance and makes the impact of the optimal patch size vanish. Details can be found in the Appendix.
Consequently, we use a patch size of $P = \sqrt{T}$ and a stride of $S = P/10$ in the following experiments.

\subsection{Time Series Classification with Frozen ViTs}
\label{subsection: time series classification}

We apply frozen ViTs as feature extractors on images generated from time series. Here, we refer to our best performing model with a ViT-H backbone from OpenCLIP (up to layer 14) as TiViT. Details on ViT layer selection are provided in \Cref{subsection: hidden representations} and the Appendix.
\Cref{tab: ucr uea joint benchmark} compares TiViT and TSFMs on the UCR and UEA test set. The state-of-the-art TSFM Mantis achieves a classification accuracy of 80.1\% on the UCR benchmark.
Our statistical analysis with a paired t-test and a significance level of 0.05 confirms that TiViT significantly outperforms ($p=0.003$) Mantis across the 128 datasets of the UCR benchmark, achieving 81.6\% accuracy. 

On the UEA benchmark, TiViT reaches a classification accuracy of 72.0\%, which is statistically on par with Mantis. The concatenation of per-channel representations, without learning any cross-channel interactions, achieves state-of-the-art performance. In line with prior work \cite{feofanov2025mantis}, we could not observe any benefit with channel-gating or attention pooling of channel-wise representations.

\Cref{fig:ucr_benchmark_91} shows that TiViT outperforms not only other TSFMs, but also a series of (self-)supervised baselines from Goswami et al. \cite{goswami2023unsupervised} (pre-)trained per dataset. The comparison is limited to 91 UCR datasets since most of these models can only handle time series up to $T=512$. 

\myparagraph{Vision Methods.} VisionTS has successfully repurposed pretrained masked autoencoders (MAEs) for time series forecasting. However, when we apply VisionTS in time series classification on the 91 UCR datasets, it achieves only 73.4\% accuracy. VisionTS remains behind Mantis with an accuracy of 82.6\% and particularly behind TiViT with an accuracy of 83.5\%. This performance gap cannot be explained by the choice of the vision backbone alone, but by a crucial modeling insight. We find that the most effective features for classification lie in the hidden layers rather than the output layers of pretrained ViTs (see \Cref{subsection: hidden representations}).
TimesNet is another vision approach trained in a fully supervised manner on each time series dataset. It is computationally expensive and reaches only 57.2\% accuracy. To the best of our knowledge, TiViT is the first vision approach that outperforms TSFMs in time series classification.

\myparagraph{Forecasting TSFMs.} Interestingly, TSFMs such as Chronos and VisionTS, primarily designed for forecasting, perform worse than TiViT or Mantis in time series classification. Specifically, Chronos reaches 82.1\%, which ranks the model behind Mantis with 82.6\% and TiViT with 83.5\%. This highlights that models optimized for forecasting cannot be simply transferred to classification tasks and emphasizes the need for dedicated classification-focused TSFMs such as TiViT.

\myparagraph{Ensemble of TiViT and TSFMs.} Concatenating the representations of TiViT and TSFMs increases the classification accuracy (see \Cref{tab: ucr uea joint benchmark}) by up to $3\%$ over the best TSFM Mantis. In \Cref{subsection: alignment and fusion}, we investigate their representational alignment and complementarity. %

\begin{figure} [tb]
\begin{minipage}[m]{.4\linewidth}
\vspace{-0.2cm}
\captionof{table}{Classification accuracy of TSFMs and frozen TiViT.}
\centering
\scriptsize
\vspace{0.2cm}
\resizebox{\linewidth}{!}{
\begin{tabular}{@{}lccc@{}}
\toprule
\textbf{Model} & \textbf{UCR} & \textbf{UEA} & \textbf{WOODS}\\
\midrule
Moment & 79.0 & 69.9 & 70.4 \\
Mantis & 80.1 & 72.4 & 70.1 \\
\midrule
TiViT \emph{(Ours)} & 81.6 & 72.0 & 72.4 \\
\midrule
TiViT + Moment & 82.7 & 72.6 & 73.5 \\
TiViT + Mantis & \textbf{83.1} & \textbf{73.7} & \textbf{73.7} \\
\bottomrule
\end{tabular}
}
\label{tab: ucr uea joint benchmark}
\end{minipage}
\hfill
\begin{minipage}[m]{0.58\textwidth}
\centering
\includegraphics[width=\linewidth]{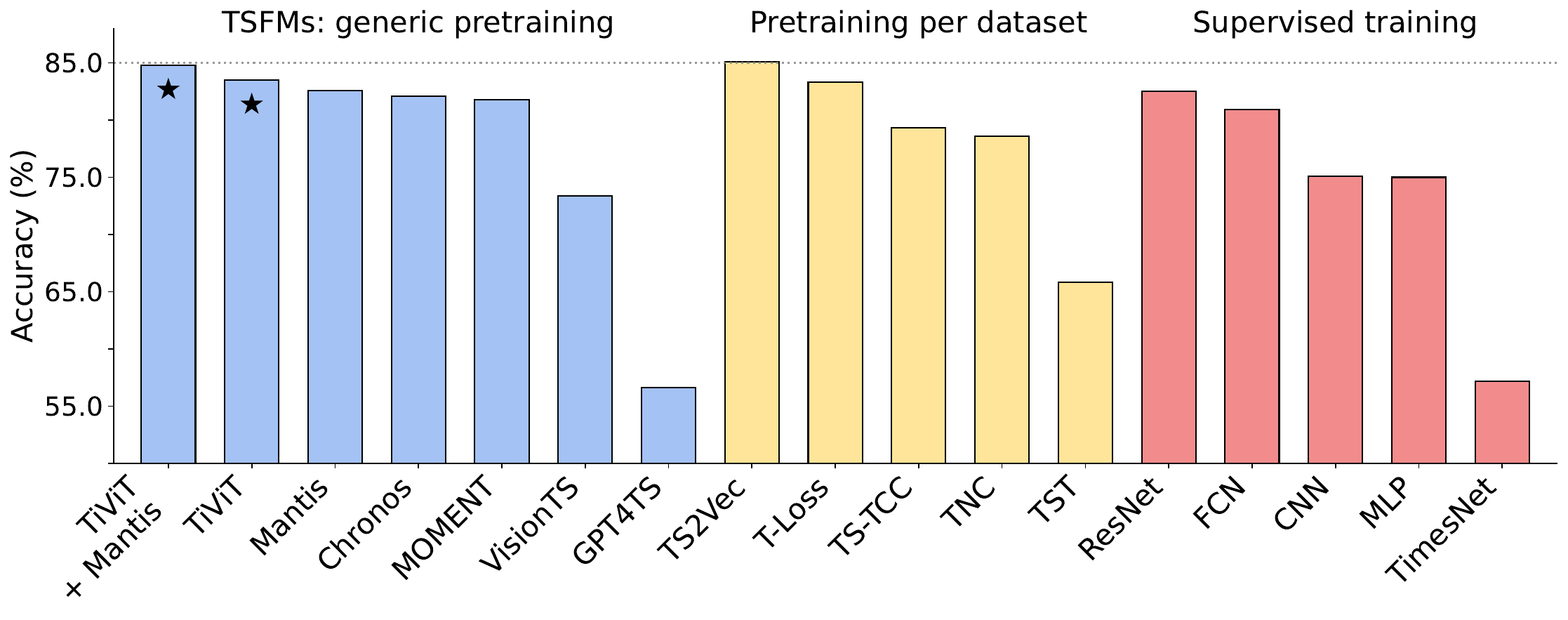}
\captionof{figure}{Classification on 91 UCR datasets.}
\label{fig:ucr_benchmark_91}
\end{minipage}
\end{figure}

\subsection{Effectiveness of Hidden ViT Representations}
\label{subsection: hidden representations}

\myparagraph{Layerwise Probing.} While the final representations of ViTs typically capture high level semantics, intermediate layers encode lower level information \cite{dorszewski2025colors}. Our study reveals that the intermediate representations of ViTs are the most effective for downstream classification. \Cref{fig: layer vs accuracy} reports the classification performance of TiViT with pretrained ViTs from DINOv3, CLIP, and SigLIP 2 on the validation split of the UCR benchmark. For each dataset, we extract representations from the hidden layers of ViTs, average them, and train a linear classifier. 

The intermediate representations of ViTs, between 40\% and 70\% of the layer depth, achieve the highest classification accuracy.
CLIP and SigLIP 2, both optimized with a contrastive loss on image-text pairs, reach best performance in their earlier layers: layer 14 of 33 for CLIP (ViT-H) and layer 12 of 28 for SigLIP 2 (SoViT-400m). In contrast, DINOv3 (ViT-L) trained with contrastive learning and masked modeling on images only, reaches the highest classification accuracy with representations from a later layer (17 of 25).
For each ViT, we determine the optimal hidden layer based on its highest validation accuracy on the UCR benchmark. Note that simply using the final layer like previous vision-based approaches would be far from state-of-the-art classification performance. 

\myparagraph{Intrinsic Dimension.} To better understand the hidden representations of ViTs, we analyze their intrinsic dimension (see \Cref{fig: intrinsic dimensionality}) and principal components (see Appendix).
Valeriani et al.~\cite{valeriani2023geometry} have previously investigated the geometry of hidden representations of Transformers for in-domain vision and language applications.
We measure the intrinsic dimension of ViTs applied on time series from the UCR archive using the DADApy~\cite{dadapy} implementation of the TWO-NN estimator~\cite{facco2017estimating}.

\Cref{fig: intrinsic dimensionality} displays for three different ViT backbones the intrinsic dimensionality of their representations at varying layer depth. Across these three backbones, the mean Pearson correlation coefficient between the intrinsic dimension and validation accuracy is $\rho = 0.704$. The best performing layers exhibit the highest or second highest intrinsic dimension.

\begin{figure}[t]
\centering
\begin{subfigure}[b]{0.49\textwidth}
\centering
\includegraphics[width=0.95\textwidth]{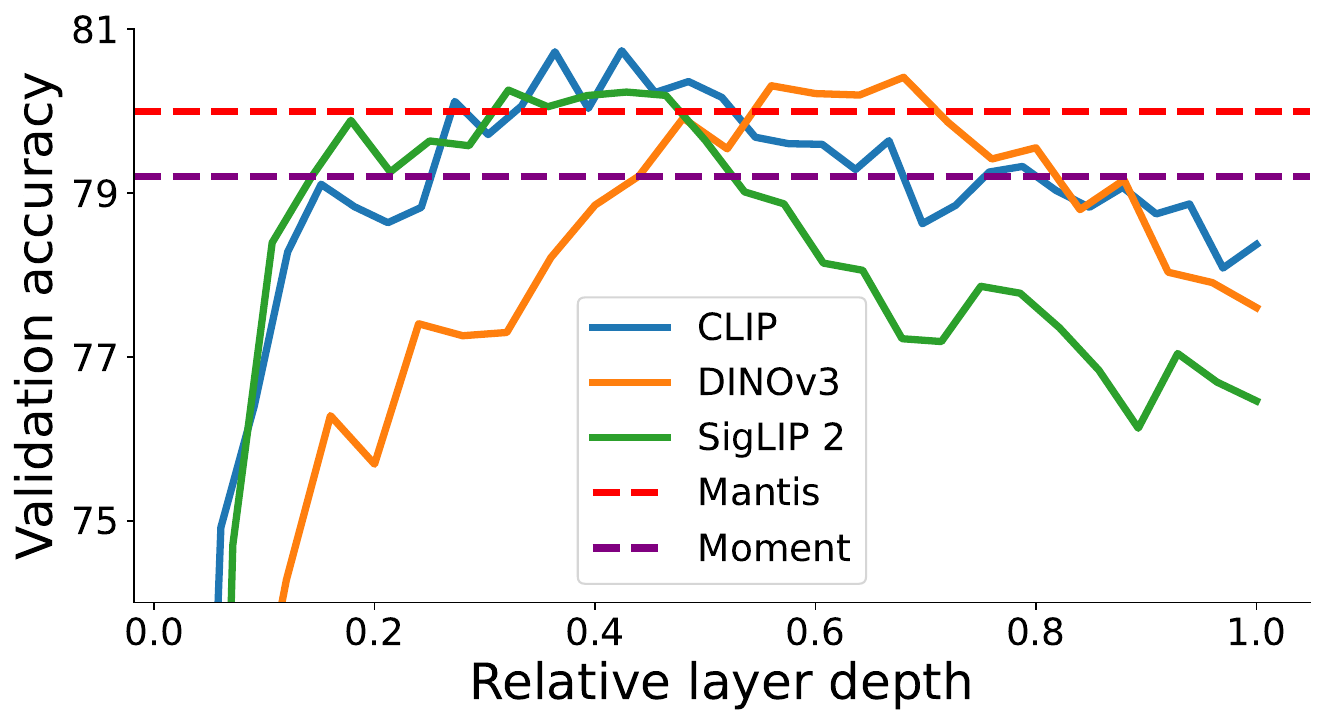}
\caption{Validation accuracy}
\label{fig: layer vs accuracy}
\end{subfigure}
\hspace{0.1cm}
\begin{subfigure}[b]{0.47\textwidth}
\centering
\includegraphics[width=0.95\textwidth]{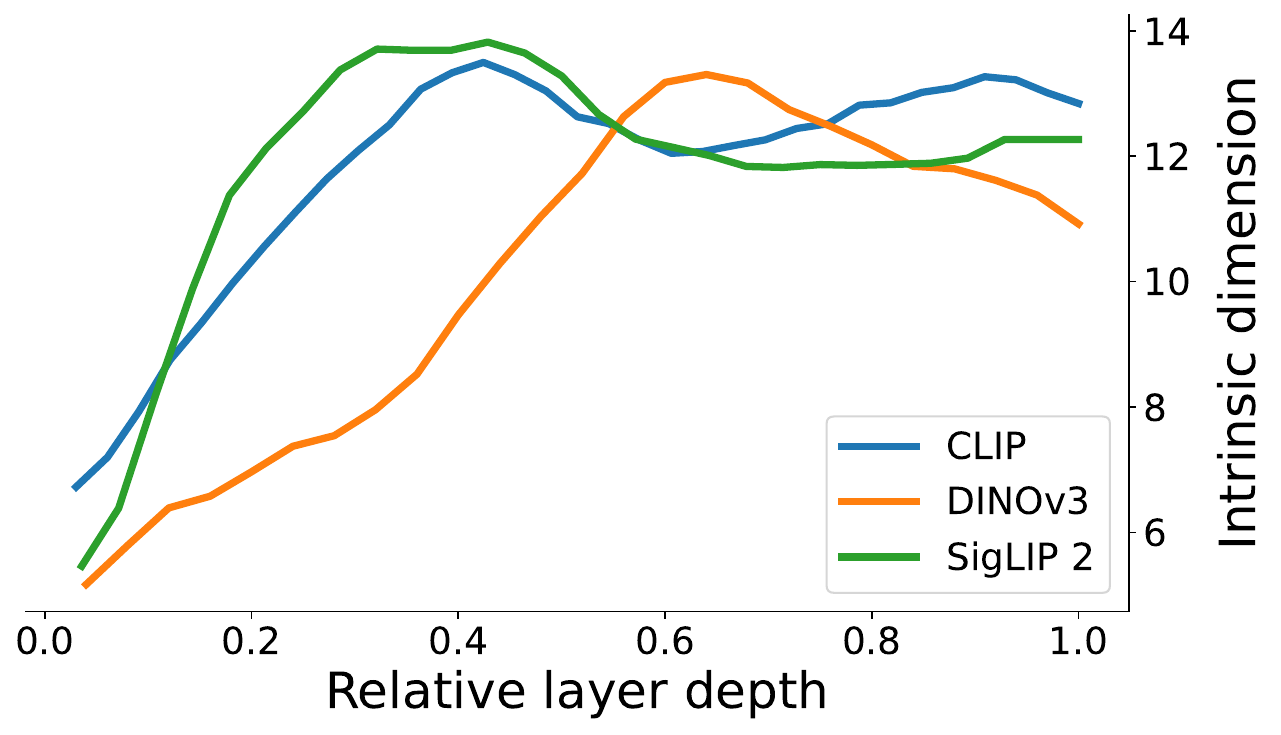}
\caption{Intrinsic dimension}
\label{fig: intrinsic dimensionality}
\end{subfigure}
\caption{(a) Validation accuracy and (b) Intrinsic dimensionality using hidden representations at different depth of pretrained ViTs ({\color{NavyBlue} CLIP}, {\color{Orange} DINOv3}, {\color{ForestGreen} SigLIP 2}). Results are averaged over 128 datasets from the UCR benchmark.}
\label{fig: layer analysis}
\end{figure}

\subsection{Representational Alignment of TiViT and TSFMs}
\label{subsection: alignment and fusion}

We not only compare the effectiveness of TiViT and TSFM representations against each other, but also explore their complementarity and concatenate their features for joint classification. As depicted in \Cref{tab: joint classification and alignment}, the combination of TiViT and TSFMs consistently improves the classification performance over any standalone model. While the combination of two TSFMs yields 81.5\% accuracy, fusing TiViT-CLIP with Moment and Mantis leads to even higher accuracies of 82.7\% and 83.1\%, respectively.
These results underscore the potential of multimodal time series analysis.

To uncover the differences between TiViTs and TSFMs, we assess the alignment of their representation spaces using the mutual k-nearest neighbor metric~\cite{huh2024prh} on the 10 largest UCR datasets. 
\Cref{tab: joint classification and alignment} presents the average alignment scores across datasets for CLIP, DINOv3, Mantis, and Moment. Interestingly, the alignment score of the two TSFMs is relatively low. We hypothesize that this discrepancy arises from their different pretraining paradigms.
A similarly low alignment score is observed between any TiViT and TSFM, which we attribute to their domain gap. TiViT and Mantis extract different representations for the same time series, which is beneficial for joint classification.
The highest alignment is measured between TiViT-CLIP and TiViT-DINOv3, both of which are pretrained contrastively on image datasets.

\begin{table}[tb]
\centering
\small 
\setlength{\tabcolsep}{5pt} 
\caption{Joint classification accuracy and alignment score (mutual-kNN) for TiViTs and TSFMs on the UCR benchmark.}
\label{tab: joint classification and alignment}
\begin{tabular}{@{}clclccc@{}}
\toprule
\multirow{2}{*}{\textbf{Fusion}} & \multicolumn{2}{c}{\textbf{Model 1}} & \multicolumn{2}{c}{\textbf{Model 2}} & \multirow{2}{*}{\textbf{\makecell{Joint\\accuracy}}} & \multirow{2}{*}{\textbf{\makecell{Alignment\\score}}} \\
\cmidrule(lr){2-3} \cmidrule(lr){4-5}
 & Name & Acc & Name & Acc & & \\
\midrule
TSFM + TSFM
& Mantis & 80.1 & Moment & 79.0 & 81.5 {\scriptsize \textcolor{ForestGreen}{+1.4}, \textcolor{ForestGreen}{+2.5}} & 0.222 \\
\midrule
TiViT + TiViT
& CLIP & 81.6 & DINOv3 & 80.2 & 82.2 {\scriptsize \textcolor{ForestGreen}{+0.6}, \textcolor{ForestGreen}{+2.0}} & \bfseries 0.431 \\
\midrule
\multirow{4}{*}{TiViT + TSFM}
& DINOv3 & 80.0 & Moment & 79.0 & 82.0 {\scriptsize \textcolor{ForestGreen}{+2.0}, \textcolor{ForestGreen}{+3.0}} & 0.213\\
& DINOv3 & 80.0 & Mantis & 80.1 & 82.5 {\scriptsize \textcolor{ForestGreen}{+2.5}, \textcolor{ForestGreen}{+2.4}} & 0.243 \\
& CLIP   & 81.6 & Moment & 79.0 & 82.7 {\scriptsize \textcolor{ForestGreen}{+1.1}, \textcolor{ForestGreen}{+3.7}} & 0.241 \\
& CLIP   & 81.6 & Mantis & 80.1 & \textbf{83.1} {\scriptsize \textcolor{ForestGreen}{+1.5}, \textcolor{ForestGreen}{+3.0}} & 0.262 \\
\bottomrule
\end{tabular}
\end{table}

\begin{table}[t]
\centering
\caption{Classification accuracy of TiViT on the UCR benchmark with various ViT-B backbones (6 layers) after continual pretraining or contrastive alignment with TSFMs on synthetic time series from CauKer.}
\label{tab:continual_pretraining_alignment}
\setlength{\tabcolsep}{10pt}
\small
\begin{tabular}{@{}l cccc@{}}
\toprule
\multirow{2}{*}{\textbf{TiViT}} & \multirow{2}{*}{\textbf{\makecell{Frozen\\backbone}}} & \multirow{2}{*}{\textbf{\makecell{Continual\\pretraining}}} & \multicolumn{2}{c}{\textbf{Alignment to}} \\
\cmidrule(lr){4-5}
& & & Mantis & Moment \\
\midrule
DINOv3   & 80.2 & 81.2 \scriptsize \textcolor{ForestGreen}{+1.0} & 81.2 \scriptsize \textcolor{ForestGreen}{+1.0} & 81.6 \scriptsize \textcolor{ForestGreen}{+1.4} \\
SigLIP 2 & 79.1 & 80.5 \scriptsize \textcolor{ForestGreen}{+1.4} & 81.6 \scriptsize \textcolor{ForestGreen}{+2.5} & 81.1 \scriptsize \textcolor{ForestGreen}{+2.0} \\
CLIP     & 80.8 & 81.8 \scriptsize \textcolor{ForestGreen}{+1.0} & 81.9 \scriptsize \textcolor{ForestGreen}{+1.1} & 81.4 \scriptsize \textcolor{ForestGreen}{+0.6} \\
\bottomrule
\end{tabular}
\end{table}

\subsection{Continual Pretraining and Contrastive Alignment}
\label{subsec: continual pretraining}

We adapt the ViT backbone of TiViT, pretrained on natural images, to the time series domain, via a second stage of continual pre-training or alignment with TSFMs on 2M synthetic time series from CauKer \cite{xie2026cauker}.
Here, we focus on the smallest TiViT model with ViT-B backbone and inject LoRA adapters to make contrastive pretraining on 4 GPUs with an effective batch size of 4096 feasible. Details on our experimental setup are provided in the Appendix.

\myparagraph{Continual Pretraining.} \Cref{tab:continual_pretraining_alignment} shows the classification performance of TiViT on the UCR benchmark with ViT-B backbones from OpenCLIP, SigLIP 2, and DINOv3 before and after continual pre-training. The contrastive learning on time series improves the average accuracy of all backbones across the 128 UCR datasets by more than 1\%.
TiViT with ViT-B backbone from OpenCLIP (36 M parameters, 81.8\% accuracy) outperforms the best frozen TiViT model with ViT-H backbone (257 M parameters, 81.6\% accuracy) using 7$\times$ fewer parameters.

\myparagraph{Alignment with TSFMs.} We further explore joint representation learning of TiViT and TSFMs. Specifically, we align TiViT using ViT-B backbones from OpenCLIP, SigLIP 2, and DINOv3 with the TSFMs Mantis and Moment. \Cref{tab:continual_pretraining_alignment} summarizes the results. While such representation learning does not match the gains of the TiViT and TSFM ensemble, it enables faster inference with only a single forward pass through TiViT.

\newcommand{\imp}[1]{\textcolor{green}{#1}}

\begin{figure}[t]
\centering
\begin{minipage}[t]{0.49\textwidth}
\centering
\begin{subfigure}[b]{0.48\linewidth} 
\centering
\includegraphics[width=0.48\linewidth]{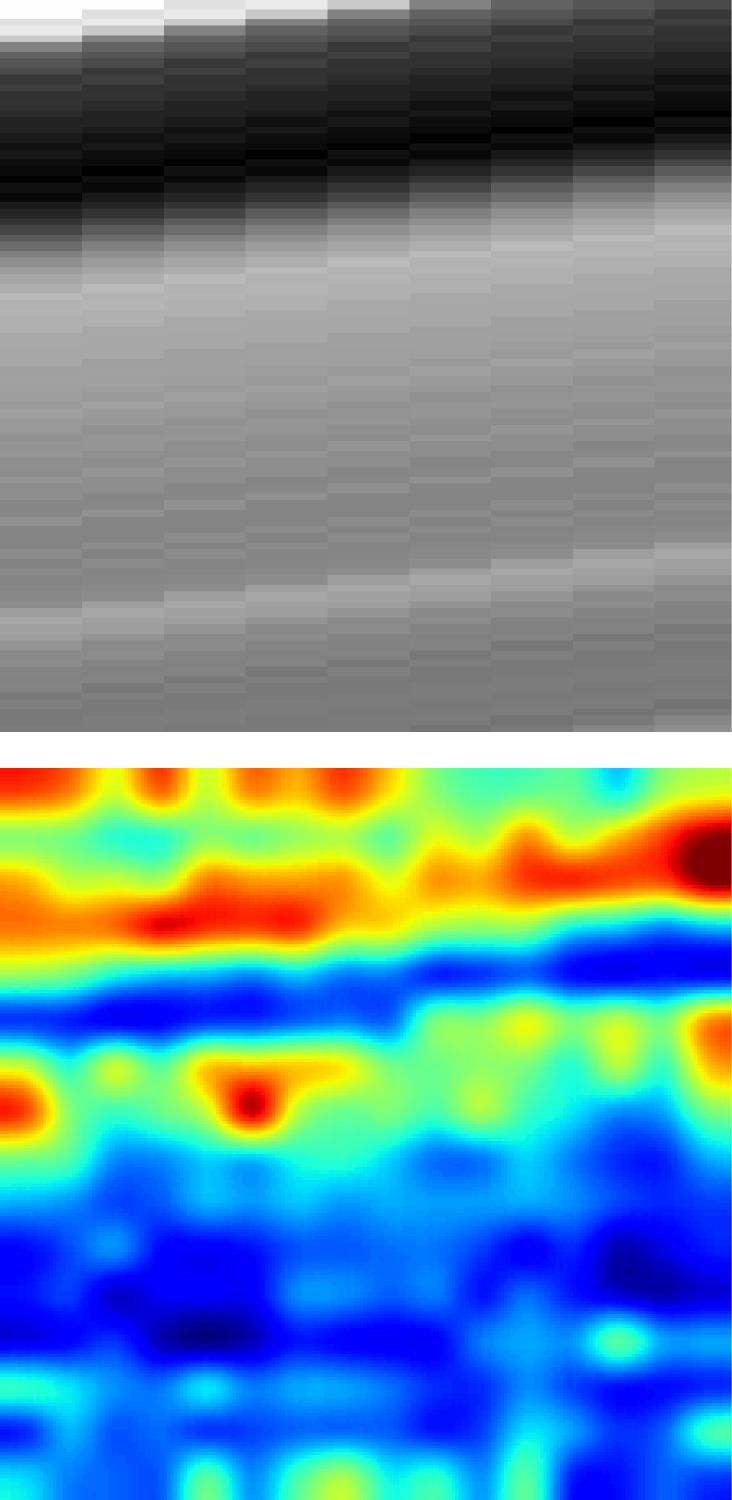}
\includegraphics[width=0.48\linewidth]{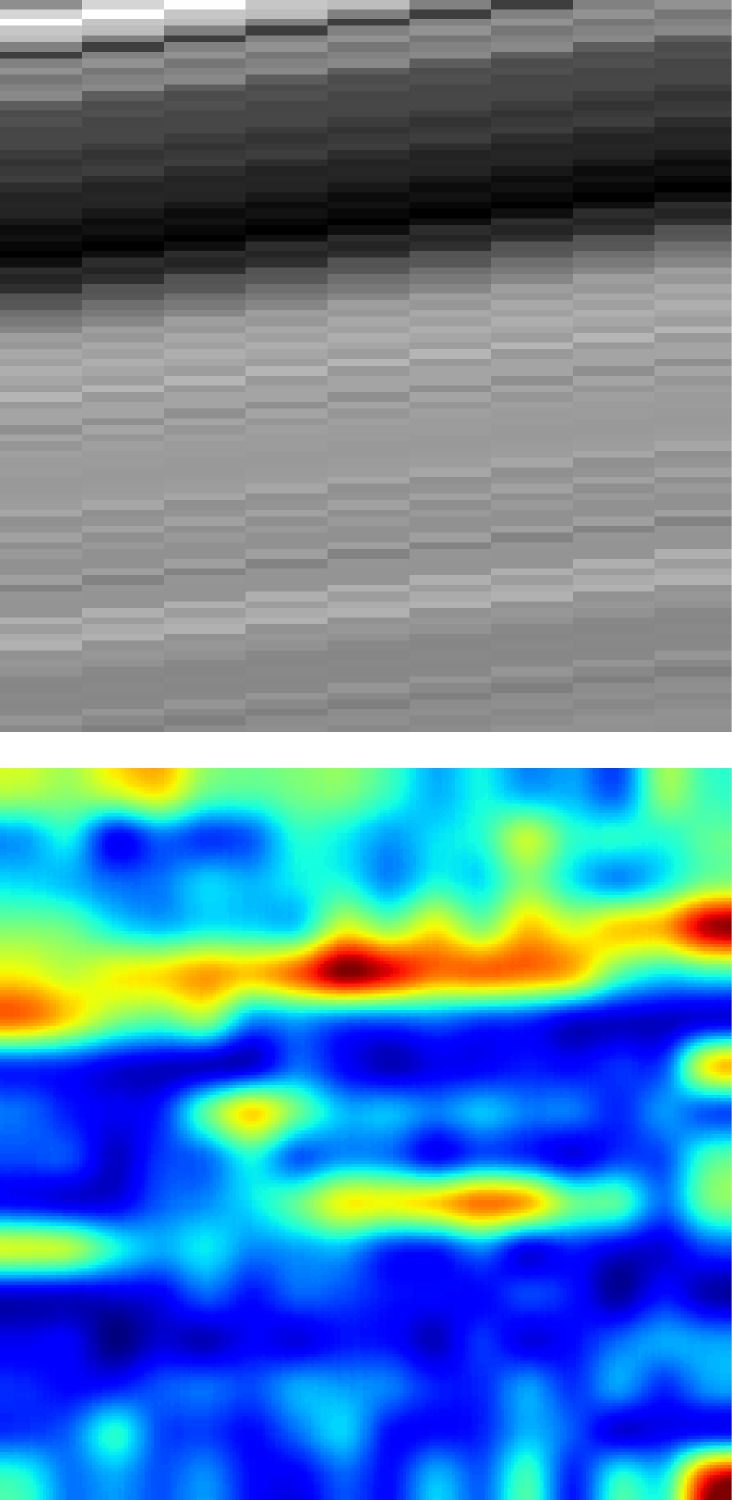}
\caption{ECG200}
\end{subfigure}
\hfill
\begin{subfigure}[b]{0.48\linewidth}
\centering
\includegraphics[width=0.48\linewidth]{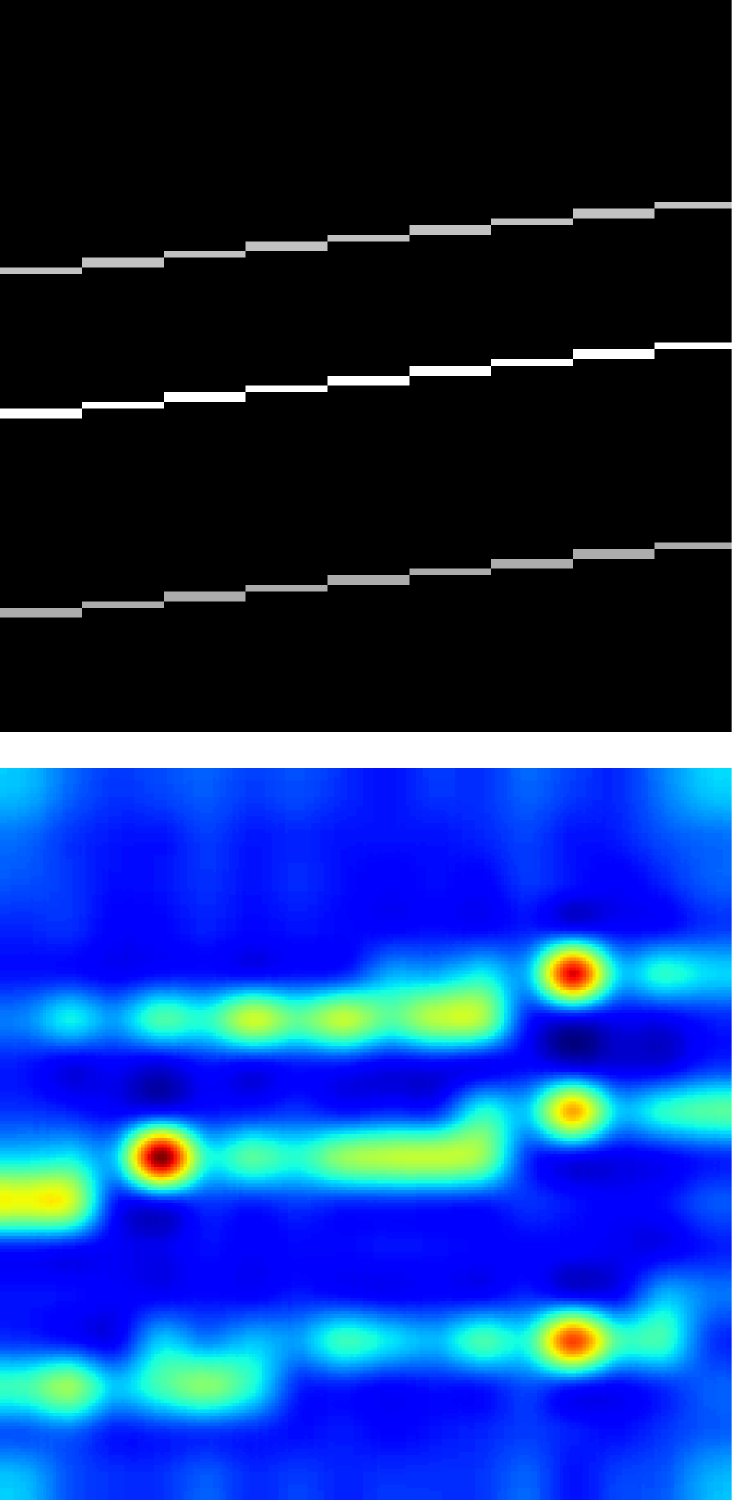}
\includegraphics[width=0.48\linewidth]{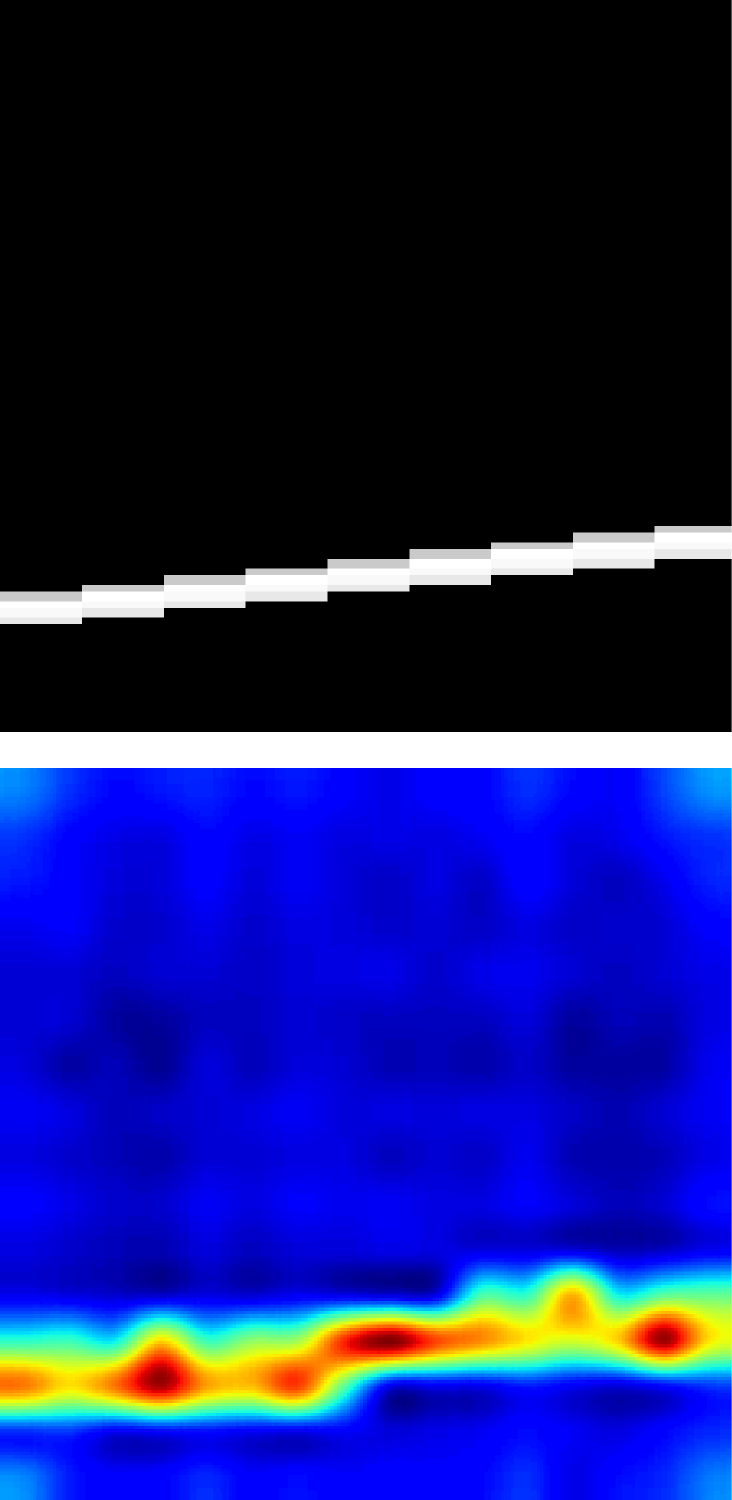}
\caption{ElectricDevices}
\end{subfigure}
\caption{Attention rollout.}
\label{fig:attn_rollout}
\end{minipage}\hfill%
\begin{minipage}[t]{0.48\textwidth}
\centering
\begin{subfigure}[b]{0.49\linewidth}
\centering
\includegraphics[width=\linewidth]{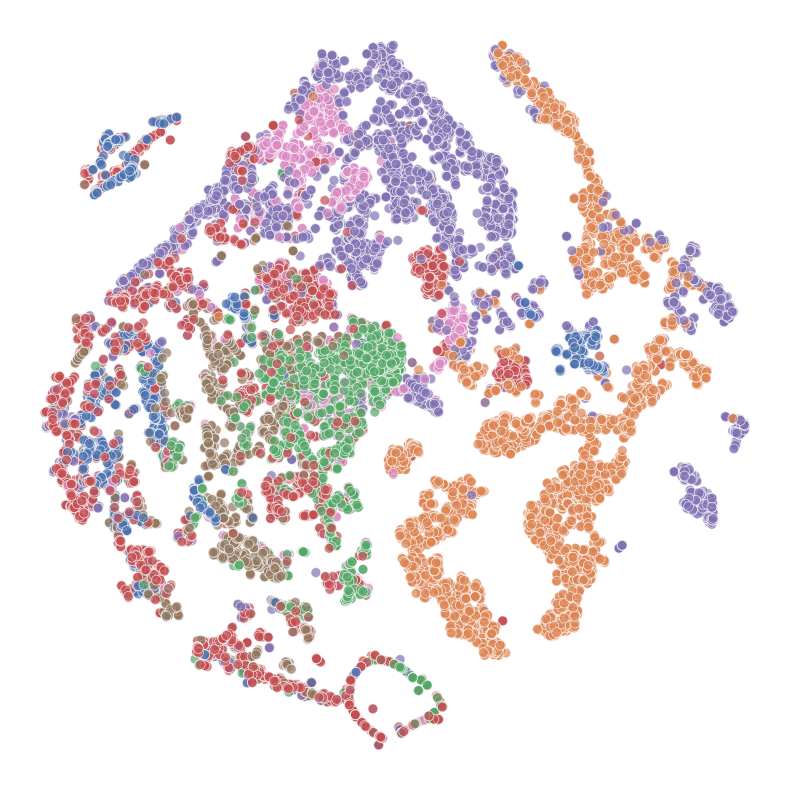}
\caption{Mantis}
\end{subfigure}
\begin{subfigure}[b]{0.49\linewidth}
\centering
\includegraphics[width=\linewidth]{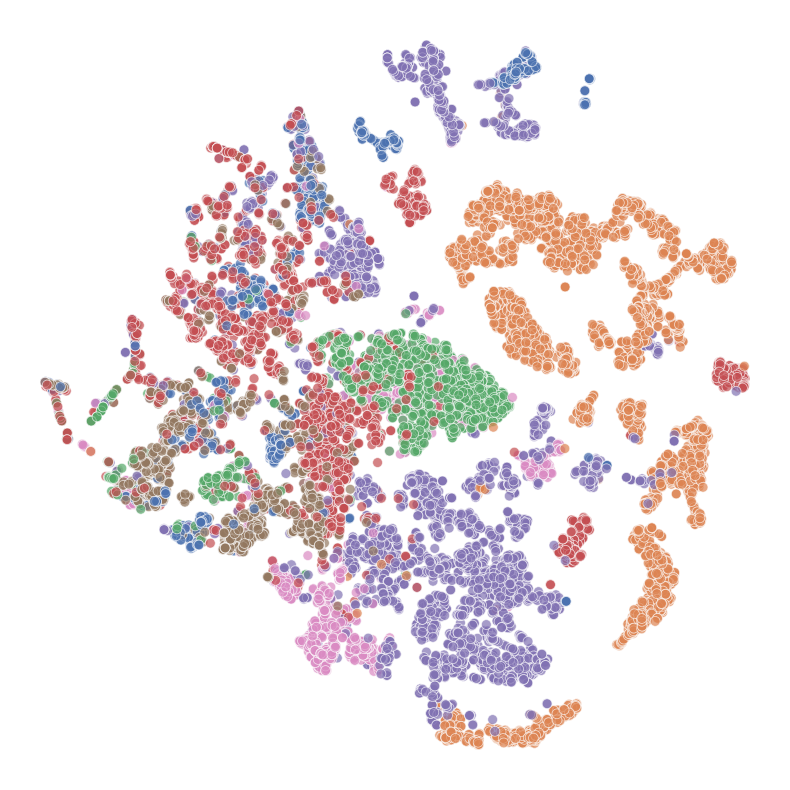}
\caption{TiViT}
\end{subfigure}
\caption{t-SNE visualizations.}
\label{fig:tsne_vis}
\end{minipage}
\end{figure}

\begin{figure*}[t]
\centering
\setlength{\tabcolsep}{2pt}
\begin{tabular}{ccccc}
Trend & Amplitude & Frequency & Baseline & Phase \\
\tiny $\sin(2\pi x)+x^c+\epsilon$
& \tiny $c\sin(2\pi x)+\epsilon$
& \tiny $\sin(2c\pi x)+\epsilon$
& \tiny $\sin(2\pi x)+c+\epsilon$
& \tiny $\sin(2\pi f x + c)+\epsilon$\\
\raisebox{-0.5\height}{\includegraphics[width=0.185\linewidth]{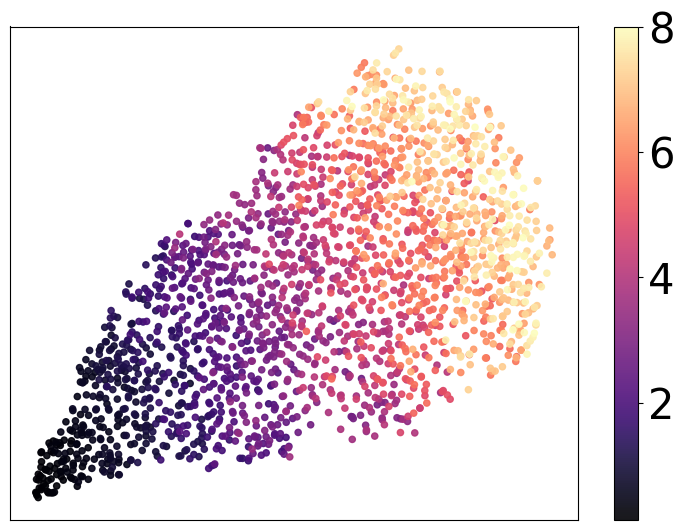}} &
\raisebox{-0.5\height}{\includegraphics[width=0.185\linewidth]{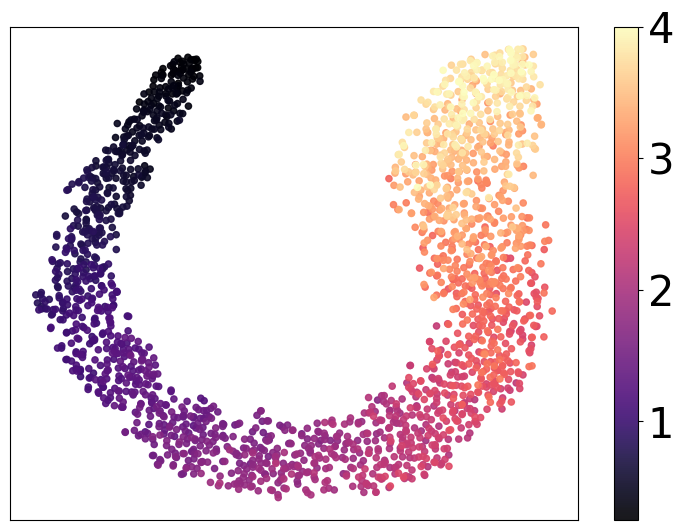}} &
\raisebox{-0.5\height}{\includegraphics[width=0.185\linewidth]{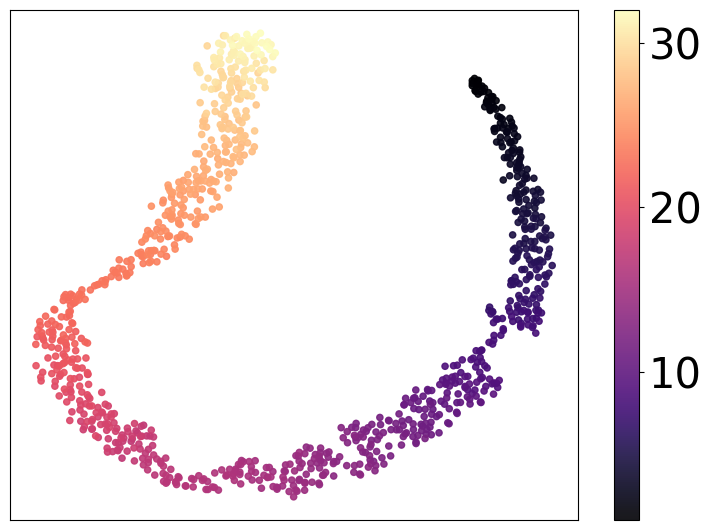}} &
\raisebox{-0.5\height}{\includegraphics[width=0.185\linewidth]{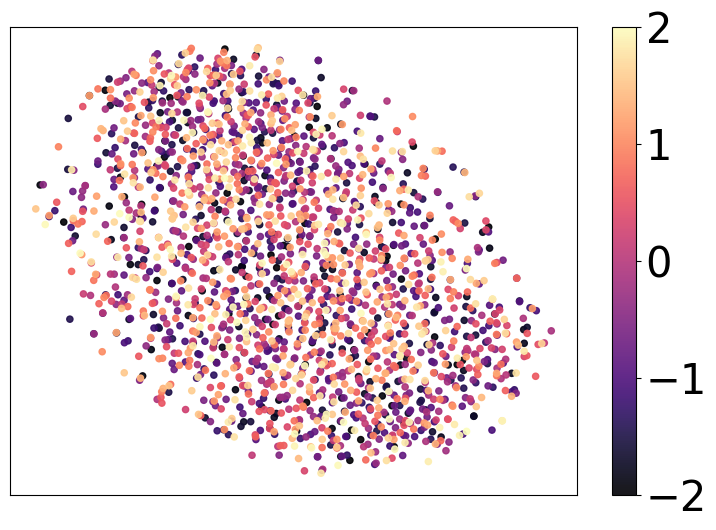}} &
\raisebox{-0.5\height}{\includegraphics[width=0.185\linewidth]{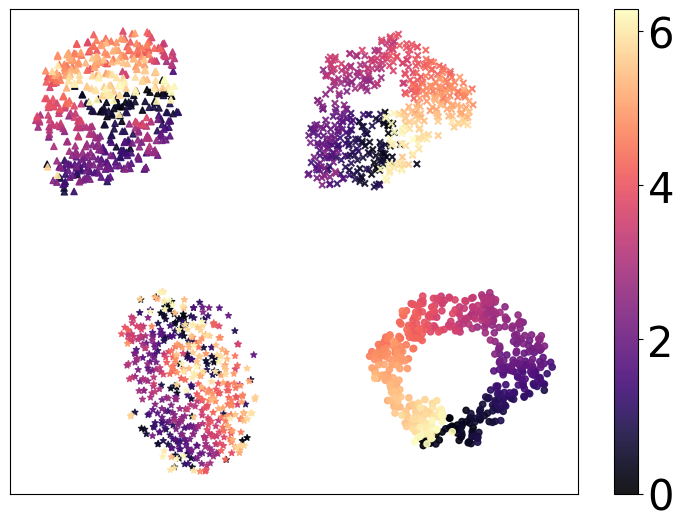}} \\
\end{tabular}
\caption{t-SNE visualizations of representations extracted by TiViT from sine waves with characteristic time series patterns. The parameter $c$ represents the varying pattern and is visualized with varying colors. $f \in \{1, 2, 3, 5\}$ and $\epsilon$ is Gaussian noise with $\mu=0$ and $\sigma=0.1$.}
\label{fig:tivit-mantis-moment-factors}
\end{figure*}

\subsection{Qualitative Feature Analysis}
\label{subsection: feature visualization}

\myparagraph{Attention Rollout.} 
We employ attention rollout \cite{abnar2020quantifying} on images generated from the ECG200 and ElectricDevices datasets of the UCR archive. As shown in \Cref{fig:attn_rollout}, the attention weights aggregated across layers highlight the most salient regions of the input. In particular, high attention weights align with the bright and dark areas in the 2D image representation, which correspond to high and low signals in the original time series. This indicates that TiViT attends to the critical patterns necessary for distinguishing samples of different classes.

\myparagraph{Clustering.} In addition, we compare the clustering of representations from TiViT and Mantis in \Cref{fig:tsne_vis} using t-SNE \cite{JMLR:v9:vandermaaten08a} visualizations on the ElectricDevices dataset. Mantis is trained contrastively and can discover class distinguishable time series representations without label supervision. Our findings go beyond this: even without any training on time series, TiViT generates embeddings that form clusters aligned with the ground-truth classes illustrated in different colors. 

\myparagraph{Time Series Patterns.} Following Goswami et al.~\cite{goswami2024moment}, we examine if TiViT can capture characteristic time series patterns. We generate time series based on sine waves with varying trend, amplitude, frequency, baseline, or phase. \Cref{fig:tivit-mantis-moment-factors} provides t-SNE visualizations of representations extracted by TiViT. Remarkably, TiViT is able to recognize all time series patterns except for the baseline shift due to instance normalization. %

\newcommand*\rot[1]{\rotatebox[origin=l]{60}{\textbf{#1}}}

\begin{table*}[t]
\centering
\caption{Classification accuracy on UCR subsets.}
\label{tab:ucr_subsets}
\scriptsize
\setlength{\tabcolsep}{2pt}
\begin{tabular}{@{}lcccc@{\hspace{6pt}\vrule\hspace{6pt}}*{8}{c}@{}}
\toprule
 & \rot{Small} & \rot{Large} & \rot{Short} & \rot{Long}
& \rot{Device} & \rot{ECG} & \rot{HAR} & \rot{Image}
& \rot{Motion} & \rot{Sensor} & \rot{Simulated} & \rot{Spectro} \\
\midrule
\# Datasets & 10 & 10 & 10 & 10 & 12 & 7 & 21 & 32 & 6 & 17 & 9 & 12 \\
\midrule
Moment & 86.6 & 85.4 & 87.4 & 67.3 & 75.5 & 88.0 & 77.1 & 79.5 & 70.4 & 85.3 & 92.3 & 76.9 \\
Mantis & 87.2 & 82.6 & 88.2 & 71.4 & 77.6 & 86.7 & 78.3 & 80.4 & 69.3 & 85.2 & 92.4 & 78.1 \\
\midrule
TiViT & 90.5 & 85.4 & 87.8 & 75.6 & 82.1 & 91.6 & 80.7 & 80.7 & 72.0 & 86.6 & 95.8 & 76.8 \\
\midrule
TiViT + Moment & 90.7 & \textbf{87.2} & 88.8 & 75.7 & 81.9 & \textbf{92.1} & 81.6 & 81.7 & \textbf{73.2} & \textbf{87.4} & 96.0 & 80.5 \\
TiViT + Mantis & \textbf{91.4} & 86.2 & \textbf{89.3} & \textbf{77.8} & \textbf{82.5} & 91.5 & \textbf{82.3} & \textbf{82.1} & 72.7 & 87.2 & \textbf{96.1} & \textbf{81.1} \\
\bottomrule
\end{tabular}%
\end{table*}

\subsection{Data-Dependent Benefits of TiViT and TSFMs}
\label{subsection: ablation studies}

In \Cref{subsection: time series classification}, we report the performance of TiViT across all 128 UCR datasets. To explore its behavior beyond the aggregated performance, we now select four UCR subsets: 10 small datasets with the fewest training samples ($16 \leq N_{train} \leq 20$), 10 large datasets with the most training samples ($1000 \leq N_{train} \leq 8926$), 10 datasets with the shortest time series ($15 \leq T \leq 80$), and 10 datasets with the longest time series ($1500 \leq T \leq 2844$). As displayed in \Cref{tab:ucr_subsets}, TiViT significantly outperforms Mantis on subsets with a small training set (90.5\% vs. 87.2\%) and long time series (75.6\% vs. 71.4\%). These findings demonstrate that TiViT excels in generalizing from limited training data and in modeling long-range dependencies.

Furthermore, \Cref{tab:ucr_subsets} breaks down the performance of models on UCR subsets grouped by data type. TiViT surpasses TSFMs in the classification of electric devices, electrocardiograms (ECG), human activity and motion profiles, sensor data, and simulated time series. The ensembles of TiViT and TSFMs consistently yield the highest classification accuracy.

\section{Conclusion}\label{sec:ccl}
In this paper, we introduced TiViT, the first method to successfully leverage large pretrained ViTs for time series classification. Our analysis revealed that the most effective features for this task are the hidden representations of ViTs which exhibit high intrinsic dimensionality.
Building on this insight, TiViT with a frozen backbone significantly outperformed TSFMs in time series classification on the UCR benchmark and reached competitive results on UEA. We highlighted the complementarity of TiViT and TSFMs, and by combining their representations, established the new state-of-the-art in time series classification on both benchmarks. Finally, we demonstrated that the continual pretraining of lightweight TiViT models on time series can further improve their performance.

\myparagraph{Limitations and Future Work.}
While we evaluated time series representation learning through linear probing, future work could explore the end-to-end finetuning of TiViT per dataset.
More broadly, our analysis indicates that ViTs trained on billions of natural images are better feature extractors for time series classification than ViT-inspired TSFMs trained on millions of time series. This suggests that domain specific features extracted from the comparably small corpus of public time series cannot reach, let alone outperform, the rich representations of ViTs. Looking forward, the community should scale time series datasets, innovate beyond adopting vision architectures and objectives on time series, and further explore the synergy between 1D and 2D time series modeling.

\begin{credits}
\subsubsection{\ackname}
This work was partially funded by the ERC (853489 - DEXIM) and the Alfried Krupp von Bohlen und Halbach Foundation, which we thank for their generous support. This work was also supported by Hi! PARIS and ANR/France 2030 program (ANR-23-IACL-0005). The authors gratefully acknowledge the scientific support and resources of the AI service infrastructure \textit{LRZ AI Systems} provided by the Leibniz Supercomputing Centre (LRZ) of the Bavarian Academy of Sciences and Humanities (BAdW), funded by Bayerisches Staatsministerium für Wissenschaft und Kunst (StMWK).
\end{credits}

\bibliographystyle{splncs04}
\bibliography{references}

\appendix
\section*{Appendix}

In Section \ref{app: experimental setup}, we detail our experimental setup for classification, continual pretraining and alignment, as well as for anomaly detection and forecasting. In Section~\ref{app: tivit additional analysis}, we analyze design choices of TiViT including the time series-to-image transformation or the size and type of the ViT backbone. In Section~\ref{app: ucr uea full benchmark}, we provide the full benchmark results in classification, anomaly detection, and forecasting per dataset.
In Section~\ref{section: time series as images}, we present a theoretical insight and experiments showing the benefits of 2D patching over conventional 1D patching for time series.
Finally, we discuss the broader impacts of our work in Section~\ref{app: broader impacts}.

\section{Experimental Setup}
\label{app: experimental setup}

\myparagraph{Datasets}
UCR \cite{dau2019ucr} comprises 128 univariate time series datasets of varying sample size ($16 \leq N_\text{train} \leq 8926$) and series length ($15 \leq T \leq 2844$).
UEA \cite{bagnall2018uea} consists of 30 multivariate time series datasets. Following Feofanov et al.~\cite{feofanov2025mantis}, we exclude three datasets (AtrialFibrillation, StandWalkJump, PenDigits) from UEA in our main evaluation due to their short sequence length or small test size. Furthermore, we consider 3 multivariate EEG datasets (PCL, CAP, SEDFx) from the WOODS~\cite{gagnonaudet2023woods} repository. These datasets come in different splits reflecting different research groups (PCL), machines (CAP), or age groups (SEDFx). For in-domain (ID) evaluation, we merge all splits before performing the train-val-test split. For out-of-domain (OOD) evaluation, we hold out one split for validation and another one for testing, while using the remaining splits for training.

\myparagraph{Vision Transformers}
Our study mainly examines three differently pretrained ViTs: 
OpenCLIP \cite{Cherti_2023_CVPR,ilharco2021openclip}, SigLIP~2~\cite{tschannen2025siglip2}, and DINOv3 \cite{simeoni2025dinov3}.
CLIP \cite{radford2021clip} performs contrastive learning of image and text encoders on image-text pairs. We reuse the ViT image encoders of OpenCLIP \cite{Cherti_2023_CVPR,ilharco2021openclip} models trained with the LAION-2B English subset of LAION-5B \cite{schuhmann2022laion}. SigLIP~2~\cite{tschannen2025siglip2} adopts contrastive learning on image-text pairs, but with a Sigmoid loss, complemented by captioning-based pretraining, self-distillation, and masked prediction.
In contrast, DINOv3 \cite{simeoni2025dinov3} is solely pretrained on images through self-distillation with a student-teacher architecture and objectives at both the image and patch level.
For each pretraining approach, we consider multiple vision model sizes (ViT-B, ViT-L, ViT-H) with varying layer depth (12, 24, and 32 layers).
Additionally, we investigate the effectiveness of ViTs from DINOv2 \cite{oquab2024dinov2} and Masked Autoencoders \cite{he2022masked} in the appendix.

\myparagraph{Baselines}
We compare TiViT to two state-of-the-art TSFMs exclusively pretrained on time series.
Mantis~\cite{feofanov2025mantis} is a Transformer model (8 M parameters) comprising 6 layers and 8 heads per layer, pretrained on 2 million time series with contrastive learning. As stated by Feofanov et al.~\cite{feofanov2025mantis}, Mantis is based on the ViT architecture, making it particularly suitable for our comparison with large-scale ViTs trained on natural images.
Moment~\cite{goswami2024moment} is a family of Transformers pretrained on 13 million time series with masked modeling. In our study, we consider Moment-base with 12 layers and 125 M parameters.\\
We further consider GPT4TS \cite{zhou2023one} pretrained on textual data and a wide range of supervised and self-supervised baselines (pre-)trained per time series dataset. The 9 supervised baselines comprise: ResNet \cite{wang2017baseline}, FCN \cite{wang2017baseline}, DTW \cite{dau2019ucr}, CNN \cite{zebik2017cnn}, MLP \cite{wang2017baseline}, Encoder \cite{serra2018encoder}, TWIESN \cite{tanisaro2016twiesn}, MCNN \cite{cui2016mcnn}, and TimesNet \cite{WuHLZ0L23}.
The 5 self-supervised baselines are: TS2Vec \cite{yue2022ts2vec}, T-Loss \cite{franceschi2019tloss}, TS-TCC \cite{eldele2021tstcc}, TNC \cite{tonekaboni2021tnc}, and TST \cite{zerveas2021TST}. 
For all of these baselines, we utilize the classification accuracy reported by Goswami et al.~\cite{goswami2024moment} in our comparison.\\
Furthermore, we evaluate the effectiveness of two state-of-the-art TSFMs that have been designed for time series forecasting in time series classification: Chronos Bolt Base \cite{ansari2024chronos} and VisionTS \cite{chen2024visionts} with MAE Base backbone. We average the sequence of their output representations to obtain a single representation for linear classification.

\myparagraph{Classification}
To assess the effectiveness of TiViT and TSFM representations in time series classification, we train a logistic regressor with the LBFGS solver per dataset. Our evaluation adheres to the standard train-test splits provided by the UCR and UEA archive and reserves 20\% of the train split for validation.
For the time series-to-image transformation, we resize the grayscale images to the resolution expected by the ViT with nearest interpolation and adjust the contrast with a factor of~$0.8$.
To compute the mutual kNN alignment score between models, we select the 10 largest UCR datasets, sample 1024 time series from each dataset, and measure the overlap of their representations for k=10. This setup is in line with Huh et al.~\cite{huh2024prh}.
All linear probing experiments can be performed on a single NVIDIA V100 GPU with 16 GB memory.

\myparagraph{Continual pretraining and contrastive alignment}
To bridge the gap between the image and time series domain, we further adapt the ViT-B backbones through continual pretraining or alignment with TSFMs on time series data. We focus on TiViT with a ViT-B backbone to make this adaptation with a sufficiently large batch size on 4 NVIDIA H100 GPUs feasible. 
We retain the backbone up to the layer that achieves the best linear probing performance. For all ViT-B backbones from OpenCLIP, SigLIP 2, and DINOv3, this corresponds to layer 6. On top of the truncated backbone, we append lightweight projection modules.
For continual pretraining, we add a single Transformer layer with a learnable CLS token, followed by an MLP with one hidden layer that projects the CLS token into a joint embedding space. For contrastive alignment, we use only an MLP with one hidden layer on top of the aggregated representations of TiViT and the TSFM.

We train all models on 2M synthetic time series from CauKer \cite{xie2026cauker} for up to 5 epochs using LoRA \cite{hu2022lora} with rank $r=8$, scaling factor $\alpha=16$, and dropout $0.1$. LoRA is applied primarily to the query and value projections; when these are not separately available, it is applied to the output projection. We additionally adapt the layer normalization parameters. Training uses a cosine annealing learning rate scheduler with a maximum learning rate of $10^{-5}$, a warmup ratio of 0.05, and a batch size of 4096.
For continual pretraining, we apply RandomCropResize as augmentation with a cropping rate $c \in [0, 0.2]$.
After pretraining or alignment, we discard the projection heads and utilize the adapted backbone in downstream classification.

\myparagraph{Anomaly detection} For this task, we equip TiViT with 6 layers of OpenCLIP ViT-B, apply no patch overlap, and flatten the sequence of representations before learning a linear reconstruction head per dataset. TiViT is evaluated across 248 datasets from the UCR Anomaly Archive \cite{wu2023ucranomalyarchive} and compared to the following baselines: Moment \cite{goswami2024moment}, GPT4TS \cite{zhou2023one}, TimesNet \cite{WuHLZ0L23}, Anomaly Transformer \cite{xu2022anomalytransformer}, DGHL \cite{challu2022dghl}, and kNN \cite{ramaswamy200knn} with $k = 5$.
We utilize the adjusted best F1 score~\cite{goswami2023unsupervised,challu2022dghl} and VUS-ROC score \cite{paparrizos2022vusroc} reported for each baseline by Goswami et al.~\cite{goswami2024moment}.

\myparagraph{Forecasting} We further evaluate TiViT in long-horizon time series forecasting on 8 standard datasets \cite{wu2021autoformer,ilbert2024samformer}. Similar to the best setup for anomaly detection, TiViT utilizes 6 layers of OpenCLIP ViT-B as backbone, applies no patch overlap, and flattens the sequence of representations. A linear forecasting head is learned per dataset and forecasting horizon in $\{96, 192, 336, 720\}$. Our comparison considers 8 baselines.
There are 2 TSFMs evaluated with linear probing: Moment \cite{goswami2024moment} and GPT4TS \cite{zhou2023one}.
Moreover, there are 6 supervised methods: PatchTST \cite{nie2022patchtst}, DLinear \cite{zeng2023dlinear}, TimesNet \cite{WuHLZ0L23}, FEDformer \cite{zhou2022fedformer}, N-BEATS \cite{oreshkin2020nbeats}, and Stationary.
The Mean Squared Error (MSE) and Mean Absolute Error (MAE) per baseline have been reported by Goswami et al.~\cite{goswami2024moment}.

\section{Additional Analysis on TiViT}
\label{app: tivit additional analysis}

\subsection{Patch Size and Overlap}
\label{app: patch size and overlap}

In Section 4.1, we report for TiViT that a patch size $P=\sqrt{T}$ and a stride $S = \frac{P}{10}$ yields high classification accuracy on any time series of length $T$.
The patch size parameter $P$ affects the visual appearance of the image representation provided to the ViT for feature extraction. Figure~\ref{fig:patch_size} displays a time series sample from the ECG200 \cite{ecg200} dataset along with its corresponding image representations for three different patch sizes. After patching and stacking, the 2D matrix is resized to the quadratic image resolution required by ViTs. Using very small (Figure~\ref{fig:patch2}) or very large (Figure~\ref{fig:patch4}) patch sizes results in redundant tokens representing the same input signal.
To avoid a computationally expensive hyperparameter search to find the best patch size $P^*$ per dataset, we propose to select $P = \sqrt{T}$ for any dataset of length $T$. A patch size of $\sqrt{T}$ yields a square-shaped image prior to resizing and thus the most diverse set of patches without any horizontal or vertical distortion (Figure~\ref{fig:patch3}). Moreover, this setting is in line with our theoretical consideration in Section~\ref{subsection: theoretical motivation}.

Table \ref{tab: periodicity search} presents the classification accuracy for TiViT with a CLIP backbone (TiViT-CLIP) and both non-overlapping and overlapping patches. To provide an upper bound on the classification performance, we perform a hyperparameter search for the best patch size $P^*$. Specifically, for each dataset of length $T$, we consider 20 equally spaced values in $[1, \frac{T}{2}]$ and identify the patch size that maximizes classification accuracy on the validation set. Note that, while there is a small decline in accuracy in the case of no overlap, when consistently applying $P = \sqrt{T}$, the computational cost is reduced by a factor of 20. The impact of the correct patch size vanishes with increasing overlap.

\begin{figure}
\centering
\begin{subfigure}[b]{0.24\textwidth}
\label{subfigure: ecg 200 sample}
\centering
\includegraphics[width=0.95\textwidth]{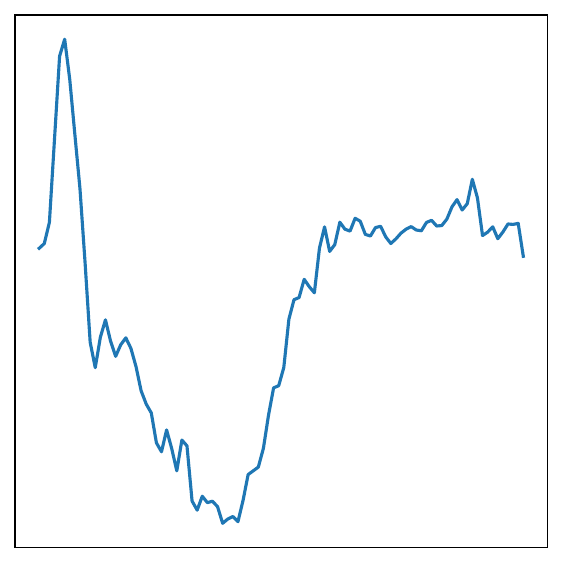}
\caption{ECG200 sample}
\label{fig:patch1}
\end{subfigure}%
\hfill
\begin{subfigure}[b]{0.24\textwidth}
\label{subfigure: patch size small}
\centering
\includegraphics[width=0.95\textwidth]{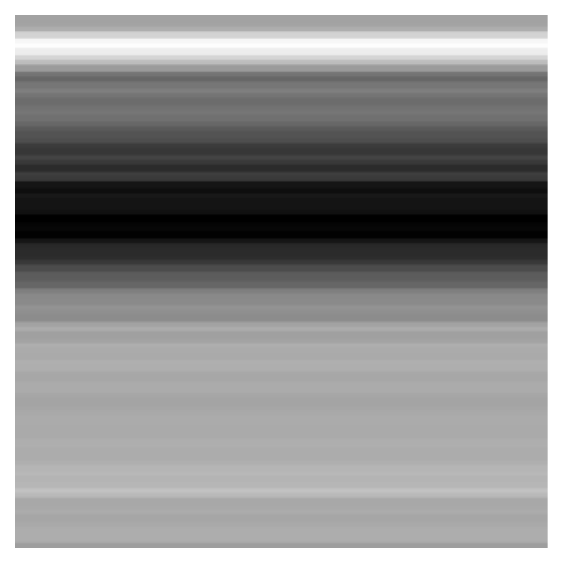}
\caption{$P = 1$}
\label{fig:patch2}
\end{subfigure}%
\hfill
\begin{subfigure}[b]{0.24\textwidth}
\label{subfigure: patch size sqrt}
\centering
\includegraphics[width=0.95\textwidth]{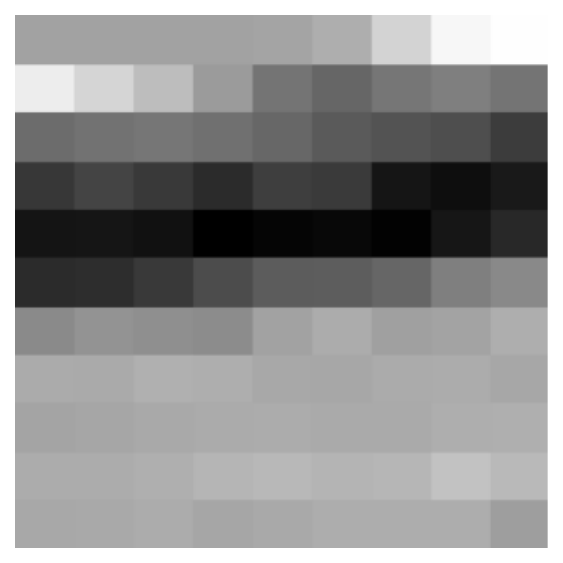}
\caption{$P = \sqrt{T}$}
\label{fig:patch3}
\end{subfigure}%
\hfill
\begin{subfigure}[b]{0.24\textwidth}
\label{subfigure: patch size large}
\centering
\includegraphics[width=0.95\textwidth]{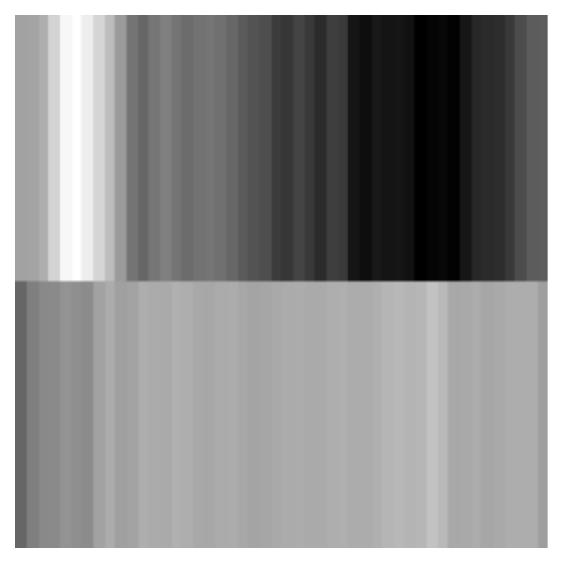}
\caption{$P = \frac{T}{2}$}
\label{fig:patch4}
\end{subfigure}
\caption{Effect of patch size $P$ on the time series-to-image transformation of a sample from the ECG 200~\cite{ecg200} dataset. To match the ViT input resolution, a small patch size ($P=1$) requires horizontal stretching, while a large patch size ($P = \frac{T}{2}$) requires vertical stretching. Both scenarios result in redundant tokens.}
\label{fig:patch_size}
\end{figure}

\begin{table}[t]
\centering
\setlength{\tabcolsep}{8pt}
\caption{Effect of patch size and overlap on UCR classification accuracy.}
\label{tab: periodicity search}
\begin{tabular}{@{}lcccccc@{}}
\toprule
Patch size & \multicolumn{3}{c}{$\sqrt{T}$} & \multicolumn{3}{c}{$P^*$} \\
\cmidrule(lr){2-4}
\cmidrule(lr){5-7}
Overlap & 0.0 & 0.5 & 0.9 & 0.0 & 0.5 & 0.9 \\
\midrule
Val accuracy & 78.0 & 80.3 & 80.7 & 88.1 & 88.9 & 89.7 \\
Test accuracy & 78.3 & 80.4 & 81.6 & 79.3 & 80.7 & 81.7 \\
\bottomrule
\end{tabular}
\end{table}

\subsection{Interpolation Algorithm for Image Resizing}

In our time series-to-image transformation, we resize the grayscale images to the resolution expected by the ViT with nearest interpolation by default. To further investigate the impact of the resizing method, we conduct additional experiments using bilinear and bicubic interpolation, both with and without antialiasing, and Lanczos interpolation. Table \ref{tab:ucr_interpolation} summarizes our results on the UCR benchmark and indicates that nearest interpolation yields the highest classification accuracy. We hypothesize that nearest interpolation is optimal for TiViT since it preserves the raw time series signals without introducing any smoothing artifacts.

\subsection{Imaging Method for Time Series}

In Section 3.1, we describe the transformation of time series into grayscale heatmaps. Here, we explore two alternative image representations. Specifically, we visualize the time series as line plots, similar to \cite{li2023time}, and Gramian Angular Fields (GAF). We provide these 2D representations to TiViT and evaluate their effectiveness for classification on the UCR benchmark. For the two new imaging methods, we perform a hyperparameter search on the hidden layers ([10, 14, 18]) and choose the best configuration based on validation accuracy. The test accuracy is shown in Table \ref{tab:ucr_imaging_methods}. Our results indicate that TiViT achieves the highest classification accuracy using the heatmap-based representations.

\begin{table}
\centering
\caption{Comparison of interpolation methods on the UCR benchmark.}
\label{tab:ucr_interpolation}
\begin{tabular}{@{}lccc@{}}
\toprule
Interpolation & Antialias & Accuracy \\
\midrule
\multirow{2}{*}{Bilinear} & False     & 81.2 \\
& True      & 80.9 \\
\midrule
\multirow{2}{*}{Bicubic}  & False     & 79.1 \\
& True      & 79.1 \\
\midrule
Lanczos  & -         & 80.6 \\
Nearest  & -         & 81.6 \\
\bottomrule
\end{tabular}
\end{table}

\begin{table}
\centering
\caption{Comparison of imaging methods on the UCR benchmark.}
\label{tab:ucr_imaging_methods}
\begin{tabular}{@{}lcccc@{}}
\toprule
Imaging method & Backbone & Layer & Accuracy \\
\midrule
Gramian Angular Field & ViT-H/14 & 14 & 76.4 \\
\midrule
Lineplot & ViT-H/14 & 14 & 80.7 \\
\midrule
Heatmap  & ViT-H/14 & 14 & 81.6 \\
\bottomrule
\end{tabular}
\end{table}

\begin{table}
\centering
\caption{Linear classification accuracy of TiViT on the UCR dataset with different ways of aggregating the hidden representations per layer. We report the total number of layers including the output layer and the index of the best performing layer starting from 0.}
\label{tab: aggregation ablation}
\begin{tabular}{@{}lccccc@{}}
\toprule
\multirow{2}{*}{Model} & \multirow{2}{*}{\# Layers} & \multicolumn{2}{c}{Average of tokens} & \multicolumn{2}{c}{CLS token}\\
\cmidrule(lr){3-4}
\cmidrule(lr){5-6}
&&  Layer & Accuracy & Layer & Accuracy\\
\midrule
TiViT-DINOv2 & 25 & 15 & 80.0 & 17 & 79.1 \\
TiViT-SigLIP 2 & 28 & 10 & 80.6 & 14 & 71.7\\
TiViT-CLIP & 33 & 14 & \textbf{81.6} & 18 & 78.6 \\
\bottomrule
\end{tabular}
\end{table}

\subsection{Aggregation of Hidden Representations}
As described in Section 3.2, we obtain a single embedding for each time series by averaging the ViT hidden representations in a particular layer. We now evaluate the performance of TiViT when using the CLS token from each layer instead. Table \ref{tab: aggregation ablation} compares the linear classification performance on the UCR dataset using either the CLS token or the mean of all tokens. To ensure a fair comparison, we determine the best performing layer for each approach based on the validation accuracy. Across all backbones, the CLS token consistently results in lower test accuracy, confirming our choice to use the mean hidden representation in TiViT. Interestingly, the best performing CLS tokens appear in later layers compared to the best performing mean tokens. Therefore, utilizing the mean representations not only enhances classification accuracy, but also reduces computational cost.

\begin{figure}
\centering
\begin{subfigure}[b]{0.33\textwidth}
\centering
\includegraphics[width=0.95\textwidth]{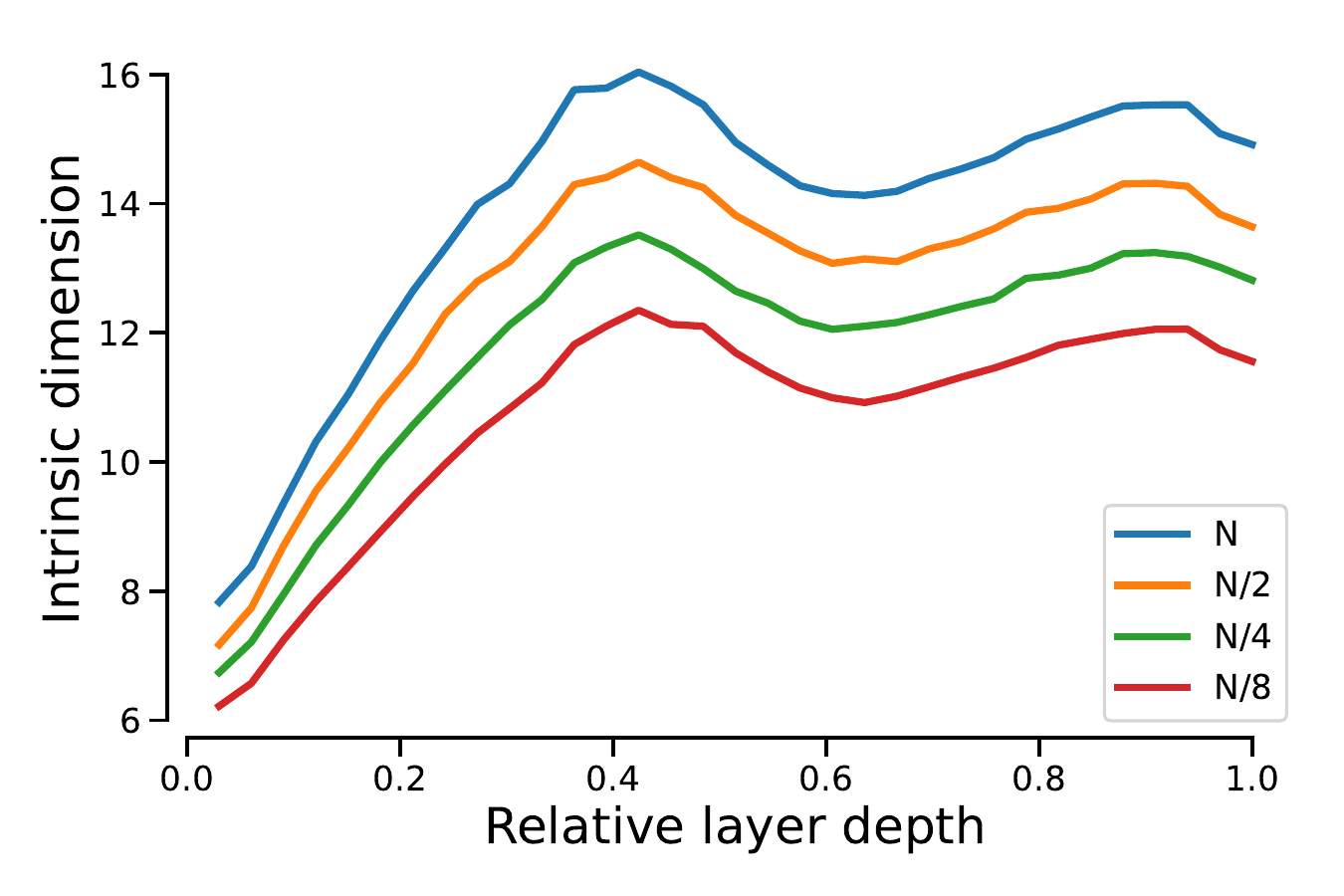}
\caption{TiViT-CLIP}
\label{fig: clip id}
\end{subfigure}%
\hfill
\begin{subfigure}[b]{0.33\textwidth}
\centering
\includegraphics[width=0.95\textwidth]{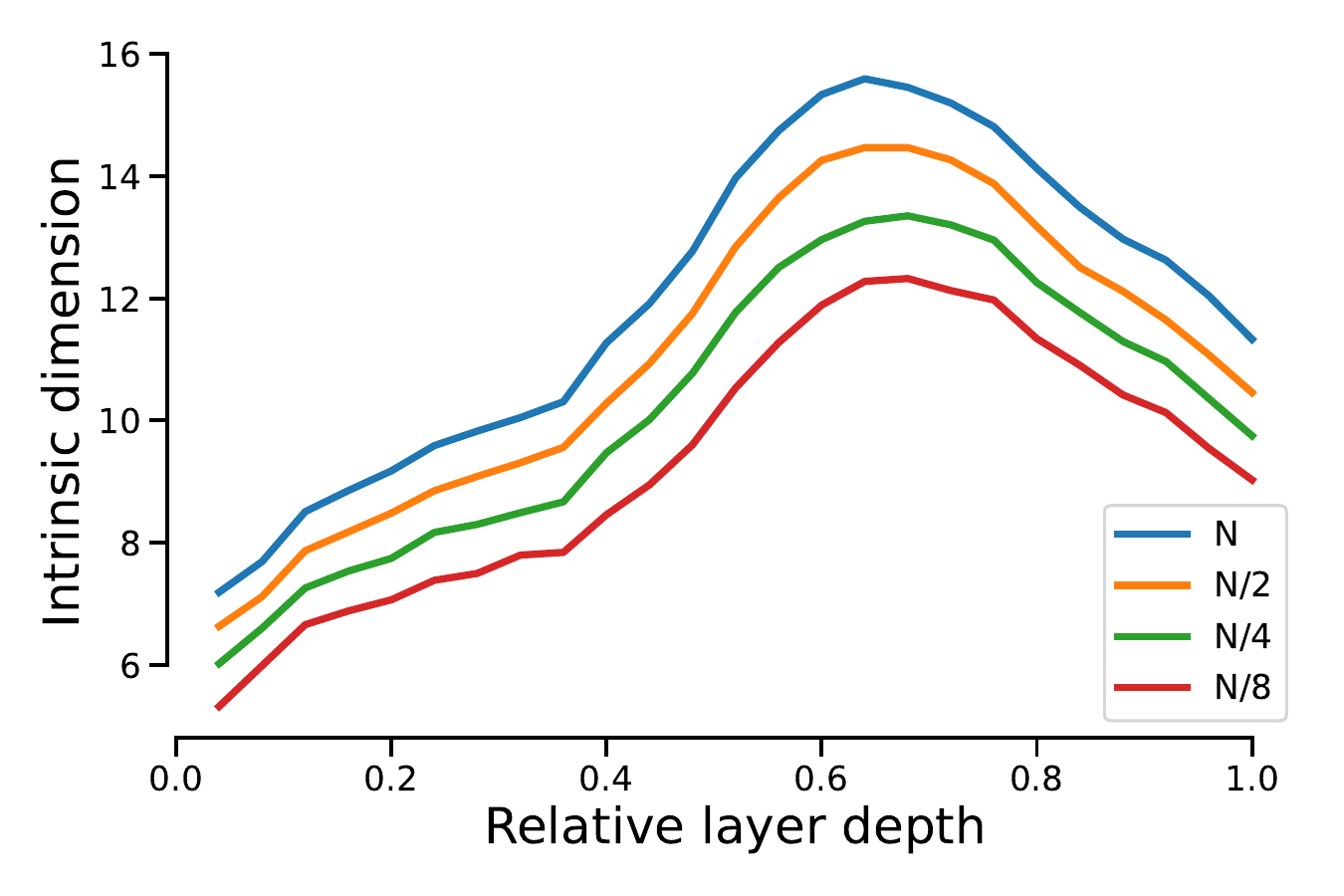}
\caption{TiViT-DINOv2}
\label{fig: dinov2 id}
\end{subfigure}%
\hfill
\begin{subfigure}[b]{0.33\textwidth}
\centering
\includegraphics[width=0.95\textwidth]{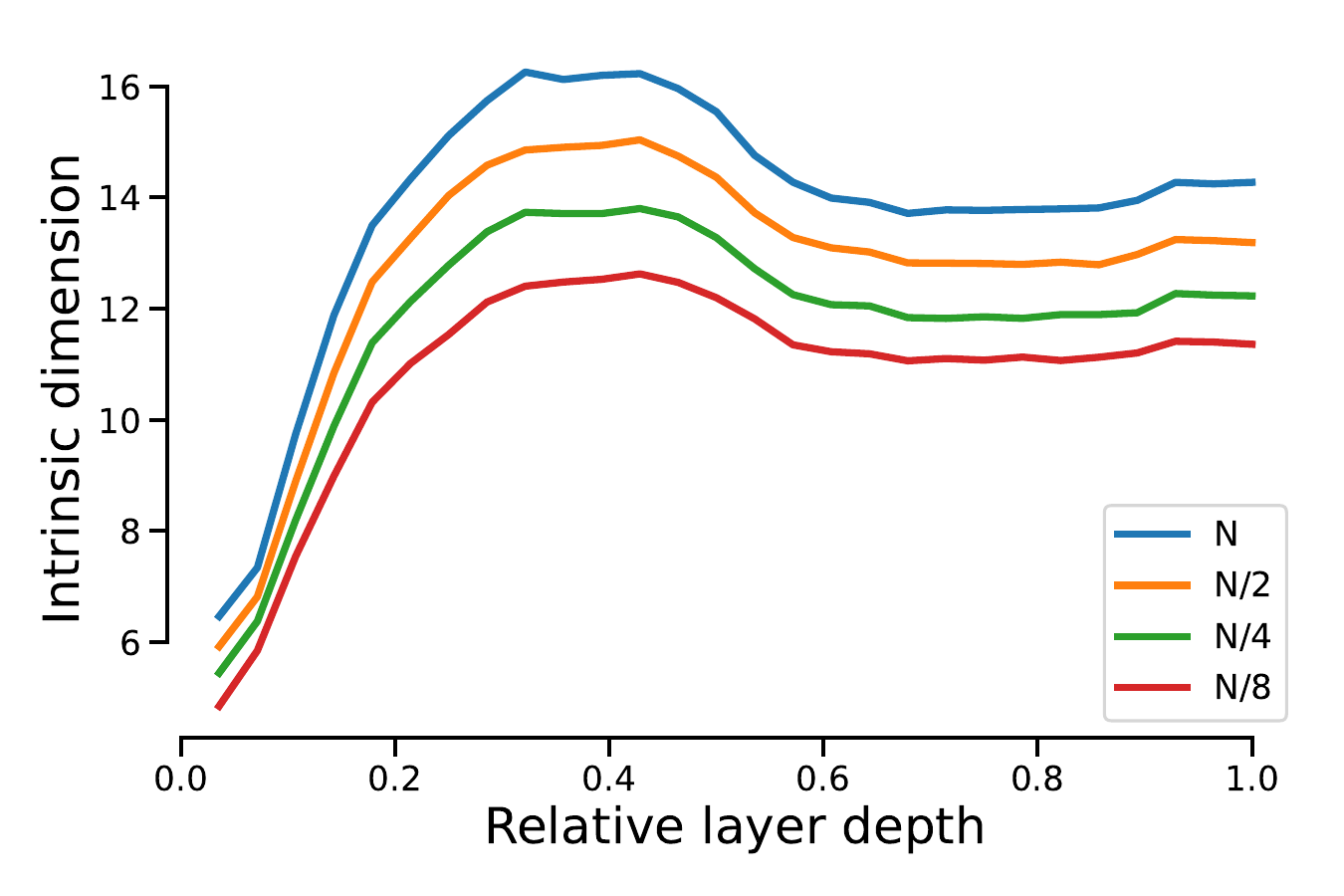}
\caption{TiViT-SigLIP 2}
\label{fig: siglip id}
\end{subfigure}%
\hfill
\caption{Intrinsic dimension of hidden representations per layer from CLIP, DINOv2, and SigLIP computed for subsamples of the dataset in $\{ {\color{NavyBlue}N}, {\color{Orange}\frac{N}{2}},{\color{ForestGreen}\frac{N}{4}}, {\color{BrickRed}\frac{N}{8}}\}$.}
\label{fig:intrinsic dimension subsampling}
\end{figure}

\subsection{Intrinsic Dimension and Principal Components of Hidden Representations}
\label{app: intrinsic dim and pca}

The intrinsic dimension quantifies the minimum number of variables required to represent a local neighborhood of samples in the representation space. To estimate the intrinsic dimension, the TWO-NN estimator introduced by Facco et al.~\cite{facco2017estimating} leverages the distance of each data point to its first and second nearest neighbor. As noted by the authors, a larger number of data points reduces the average distance to the second neighbor, and thus increases the intrinsic dimension. To mitigate this effect, they propose to subsample the dataset. Given a dataset of size $N$, we report the intrinsic dimension for $\frac{N}{4}$ subsamples in the main paper, which is in line with \cite{valeriani2023geometry}. In Figure \ref{fig:intrinsic dimension subsampling}, we compare the intrinsic dimension of average representations from hidden layers using $N$, $\frac{N}{2}$, $\frac{N}{4}$, and $\frac{N}{8}$ samples for estimation. The layer with the highest intrinsic dimension, which is central to our analysis, remains the same regardless of the subsampling ratio.

Since the intrinsic dimension only characterizes the local geometry of the representation space, we further provide a global analysis using principal components. Specifically, in Figure \ref{fig: principal components}, we determine the number of principal components that are necessary to cover 95\% of the variance in the data. For DINOv2, we observe a peak in the number of principal components in the middle layers that corresponds to the layers achieving the best classification accuracy. Interestingly, CLIP and SigLIP 2 exhibit two peaks in the number of principal components across the layers. The middle layers corresponding to the first peak yield the highest time series classification accuracy.

\begin{figure}
\centering
\includegraphics[width=0.6\textwidth]{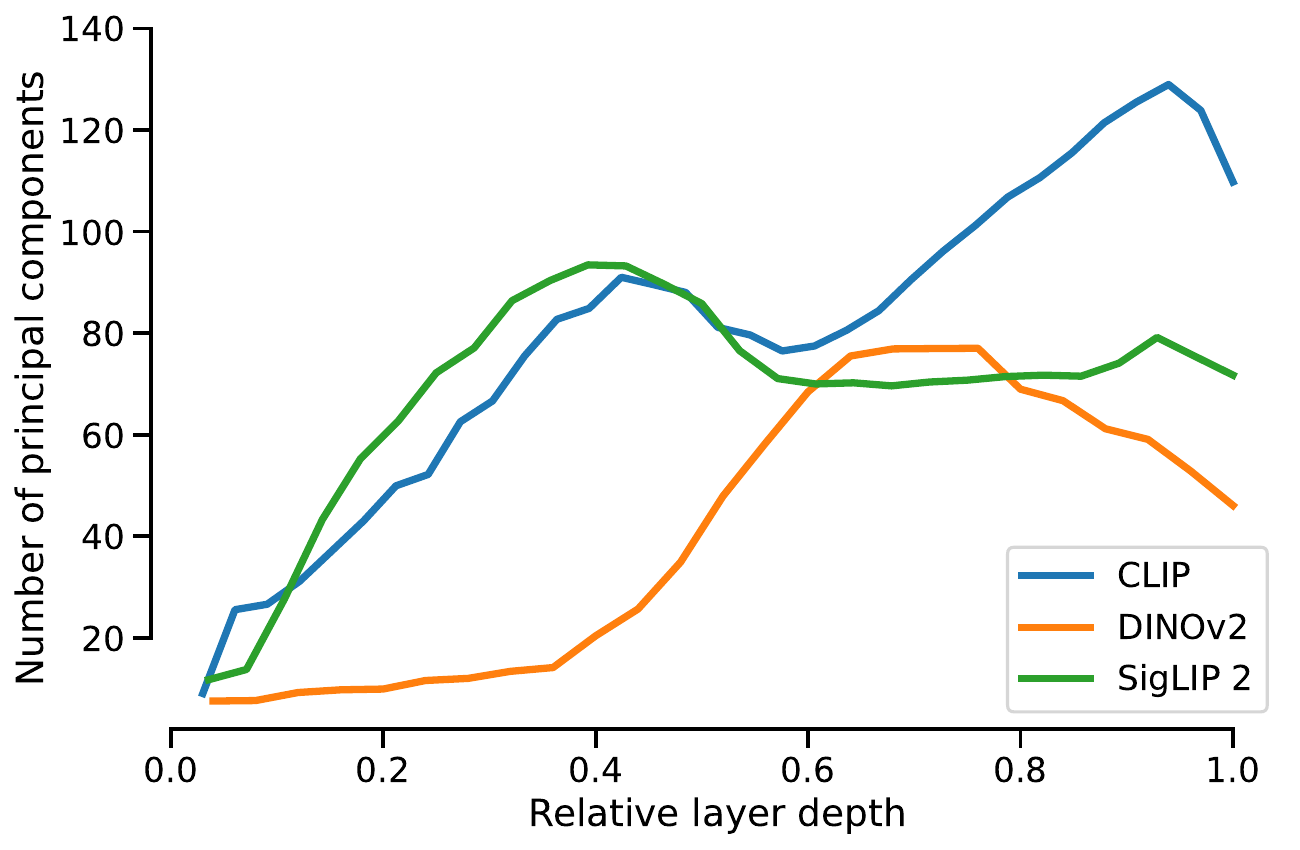}
\caption{Number of principal components necessary to cover 95\% of variance in the ViT representations per layer averaged across UCR datasets.}
\label{fig: principal components}
\end{figure}

\renewcommand{\arraystretch}{1.}
\begin{table}
\centering
\caption{Linear classification with TiViT on the UCR benchmark. For each model, we report the test accuracy achieved with the best performing hidden layer.}
\label{tab: hidden layers}
\begin{tabular}{@{}llcclc@{}}
\toprule
Model & Architecture & Layer (Max) & Parameters & Data & Accuracy\\
\midrule
TiViT-DINOv3 & ViT-L/14 & 17 (25) & 202 M & LVD-1689M & 80.2 \\
TiViT-SigLIP 2 & SoViT-400m/14 & 12 (28) & 184 M & WebLI (10B) & 80.6 \\
TiViT-CLIP & ViT-H/14 & 14 (33) & 257 M & LAION-2B  & \textbf{81.6} \\
\bottomrule
\end{tabular}
\end{table}

\begin{table}
\centering
\caption{Linear classification of TiViT-CLIP with varying size of the ViT backbone. For each model, we report the test accuracy on the UCR dataset achieved with the best performing hidden layer representation and the number of parameters up to this layer.}
\label{tab: vit size ablation}
\begin{tabular}{@{}lccc@{}}
\toprule
Architecture & Layer (total number) & Parameters & Accuracy\\
\midrule
ViT-B/32 & 8 (13) & 52 M & 79.8 \\
ViT-B/16 & 6 (13) & 36 M & 80.8 \\
ViT-L/14 & 10 (25) & 178 M & 80.3 \\
ViT-H/14 & 14 (33) & 257 M & \textbf{81.6} \\
\bottomrule
\end{tabular}
\end{table}

\subsection{Size of ViT Backbone}
We report the performance of TiViT with CLIP ViT-H backbone in Section 4.2 of the main paper. Table \ref{tab: vit size ablation} provides a detailed analysis of how the performance of TiViT varies with the size of the ViT backbone, including ViT-B (with two patch sizes), ViT-L, and ViT-H. Remarkably, with only 6 Transformer layers from ViT-B, TiViT achieves an accuracy of 80.8\%. While matching the number of Transformer layers in Mantis, TiViT surpasses Mantis (80.1\%) in classification accuracy. However, the hidden dimensionality is higher for the ViT-B backbone used in TiViT. By utilizing a larger backbone, specifically 14 hidden layers of ViT-H/14, we achieve the highest accuracy of 81.6\%, significantly outperforming conventional TSFMs.

\begin{table}[tb]
\centering
\caption{Comparison of CLIP-ViT-L-14 pretraining datasets on UCR benchmark.}
\label{tab:clip dataset size}
\begin{tabular}{@{}llccc@{}}
\toprule
Dataset   & Backbone & Layer & Accuracy \\
\midrule
Laion-400M & CLIP-ViT-L/14 & 10 & 81.6 \\
Laion-2B   & CLIP-ViT-L/14 & 10 & 80.5 \\
\bottomrule
\end{tabular}
\end{table}

\subsection{Size of Pretraining Dataset}

ViTs are pretrained on massive image datasets to learn rich and transferable features. These image datasets are orders of magnitude larger than the time series corpora used to pretrain models such as Mantis (2M samples) or Moment (13M samples). To investigate how the size of the ViT pretraining dataset affects the classification performance of TiViT, we compare TiViT with a CLIP-ViT-L backbone pretrained on 400M and 2B samples. As shown in Table \ref{tab:clip dataset size}, the model pretrained on 400M images outperforms the one pretrained on 2B images in time series classification. This suggests that dataset size alone does not guarantee superior performance in cross-domain tasks.

\begin{table}[tb]
\centering
\caption{Comparison of different backbones and feature extraction layers on the UCR benchmark.}
\label{tab:ucr_backbone_layers}
\begin{tabular}{@{}lcc@{}}
\toprule
Backbone          & Layer & Accuracy \\
\midrule
ViT-H/14          & 14      & 81.6 \\
ConvNeXt-XXLarge  & 15      & 82.1 \\
\bottomrule
\end{tabular}
\end{table}

\subsection{Convolutional Backbone}

We focus our study on ViTs because they are the most widely used vision backbones, trained on the largest datasets, and thus enable a comparison of different pretraining paradigms. Nonetheless, we also include a comparison with CNN-based methods. DINOv2, SigLIP 2, and MAE are exclusively built upon ViTs, and thus the only setting we can identify with a convolutional backbone (ConvNeXt) is OpenCLIP. We perform an ablation study for TiViT using different ConvNeXt layers in $\{10, 15, 20, 25\}$ and evaluate the classification accuracy on the UCR benchmark. As shown in Table \ref{tab:ucr_backbone_layers}, our method TiViT is fully compatible with pretrained convolutional models and can achieve even higher accuracies on the UCR benchmark when using a ConvNeXt backbone compared to the typical ViT.

\begin{table}
\centering
\caption{Linear classification accuracy of TiViT with varying MAE backbone size and aggregation of hidden representations per layer. We report the total number of layers including the output layer and the index of the best performing layer starting from 0.}
\label{tab: mae ablation}
\begin{tabular}{@{}lccccc@{}}
\toprule
\multirow{2}{*}{Architecture} & \multirow{2}{*}{\# Layers} & \multicolumn{2}{c}{Average of tokens} & \multicolumn{2}{c}{CLS token}\\
\cmidrule(lr){3-4}
\cmidrule(lr){5-6}
&&  Layer & Acc & Layer & Acc \\
\midrule
MAE Base & 13 & 8 & 72.7 & 9 & 73.8 \\
MAE Large & 25 & 14 & 74.3 & 18 & 75.6  \\
MAE Huge & 33 & 20 & 75.9 & 20 & \textbf{76.7}\\
\bottomrule
\end{tabular}
\end{table}

\begin{figure}
\centering
\includegraphics[width=0.6\textwidth]{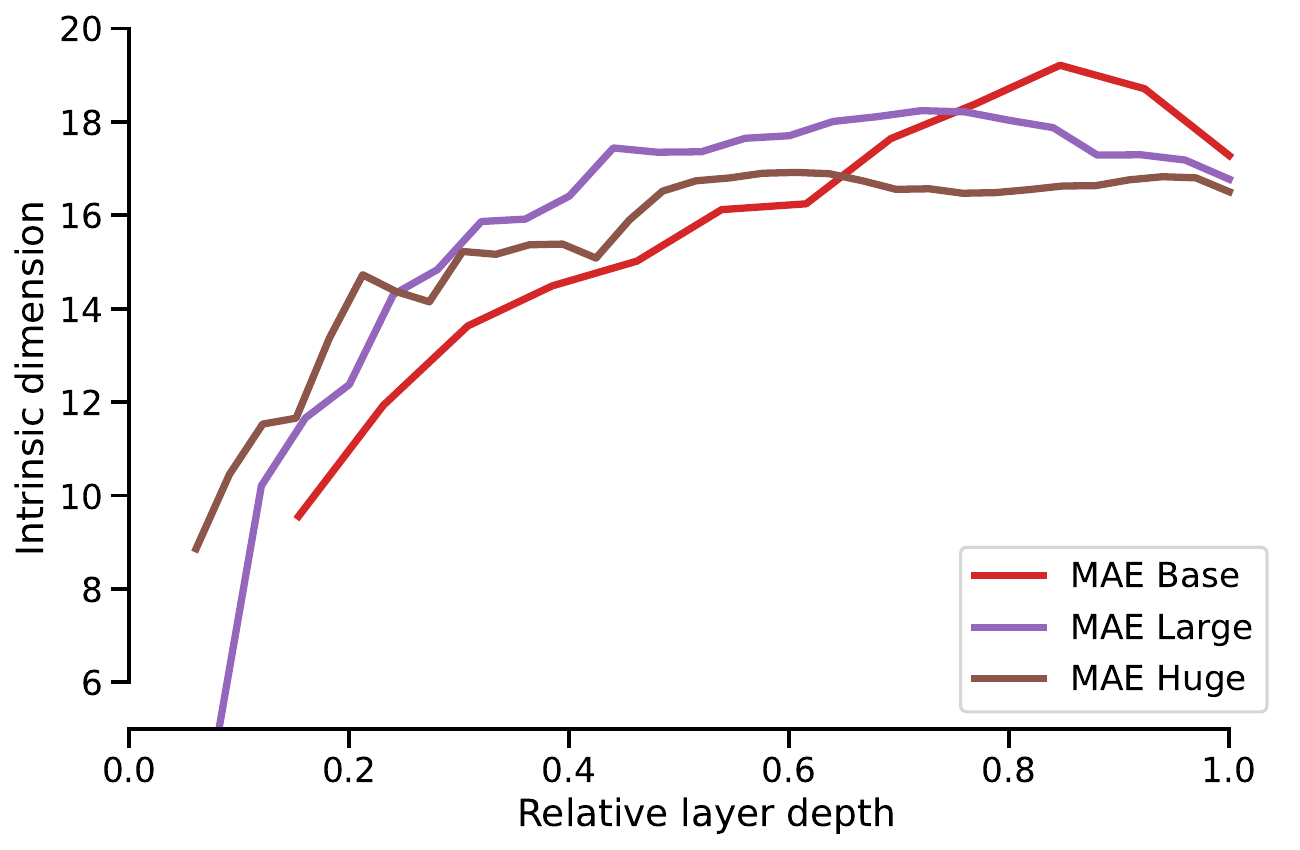}
\caption{Intrinsic dimensionality of CLS tokens per MAE layer averaged across UCR datasets.}
\label{fig: mae intrinsic dim}
\end{figure}

\subsection{Masked Autoencoder Backbone}
In the main paper, we analyze the reusability of ViT backbones from CLIP \cite{radford2021clip,schuhmann2022laion}, DINOv3 \cite{simeoni2025dinov3}, and SigLIP~2~\cite{tschannen2025siglip2} in time series classification. In contrast, Chen et al.~\cite{chen2024visionts} repurpose Masked Autoencoders (MAEs) \cite{he2022masked} for time series forecasting. To enable a direct comparison, we now utilize the hidden representations of MAE Base, Large, and Huge in time series classification.

Our analysis in Table \ref{tab: mae ablation} shows that for MAEs using the CLS token yields better performance in time series classification than averaging token representations. Moreover, Table~\ref{tab: mae ablation} presents a comparison across MAEs of different sizes, showing that larger backbones consistently achieve higher accuracy. Different from contrastively pretrained models, summarized in Table \ref{tab: hidden layers}, the best representations for time series classification with MAE lie in later layers. We further observe that the hidden representations of the later MAE layers up to the output layer perform similarly in time series classification, while there is a significant gap between hidden representations and output representations for TiViT-CLIP (see Figure 5 in the main paper). Figure \ref{fig: mae intrinsic dim} illustrates the intrinsic dimension of the CLS tokens per layer averaged across the UCR datasets. We observe that the intrinsic dimension increases up to 60\% of the layer depth, while the later layers mostly exhibit a similar intrinsic dimension, explaining their similar classification performance.

It is worth noting that MAE has only been pretrained on ImageNet-1k \cite{deng2009imagenet} with 1.5 million samples, whereas CLIP has been pretrained on the significantly larger LAION-2B \cite{schuhmann2022laion} dataset with 2 billion samples. We hypothesize that being exposed to a larger set of images during training enhances the capacity of a vision model to extract discriminative patterns from 2D time series representations.

\begin{figure}
    \centering
    \begin{subfigure}[b]{0.47\textwidth}
        \centering
        \includegraphics[width=\textwidth]{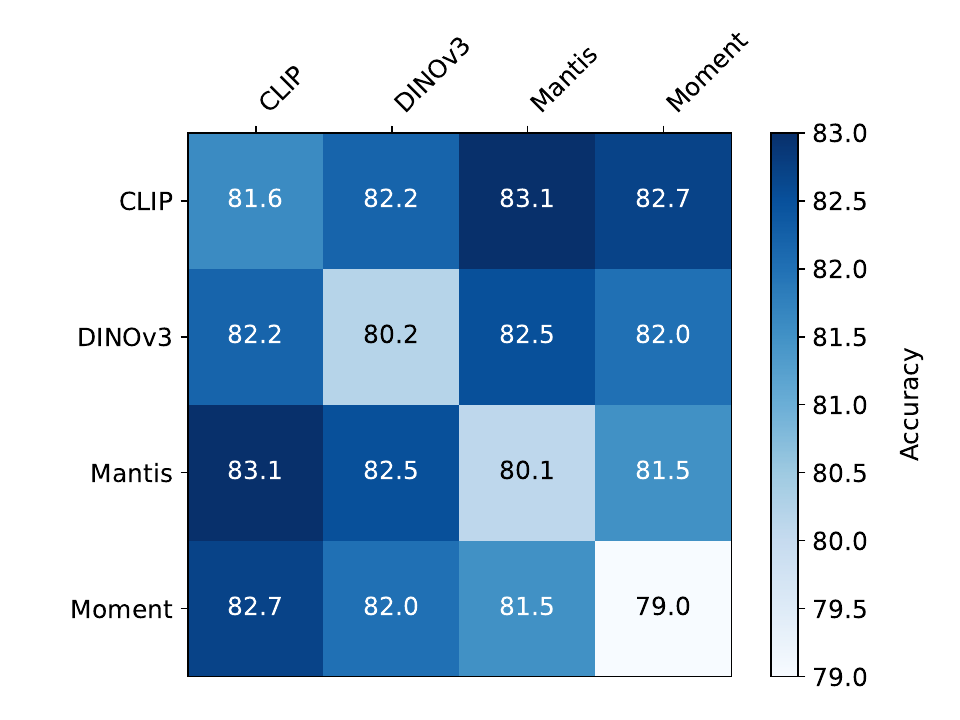}
        \caption{Pairwise joint classification accuracy.}
        \label{fig: accuracy heatmap}
    \end{subfigure}
    \hfill
    \begin{subfigure}[b]{0.47\textwidth}
        \centering
        \includegraphics[width=\textwidth]{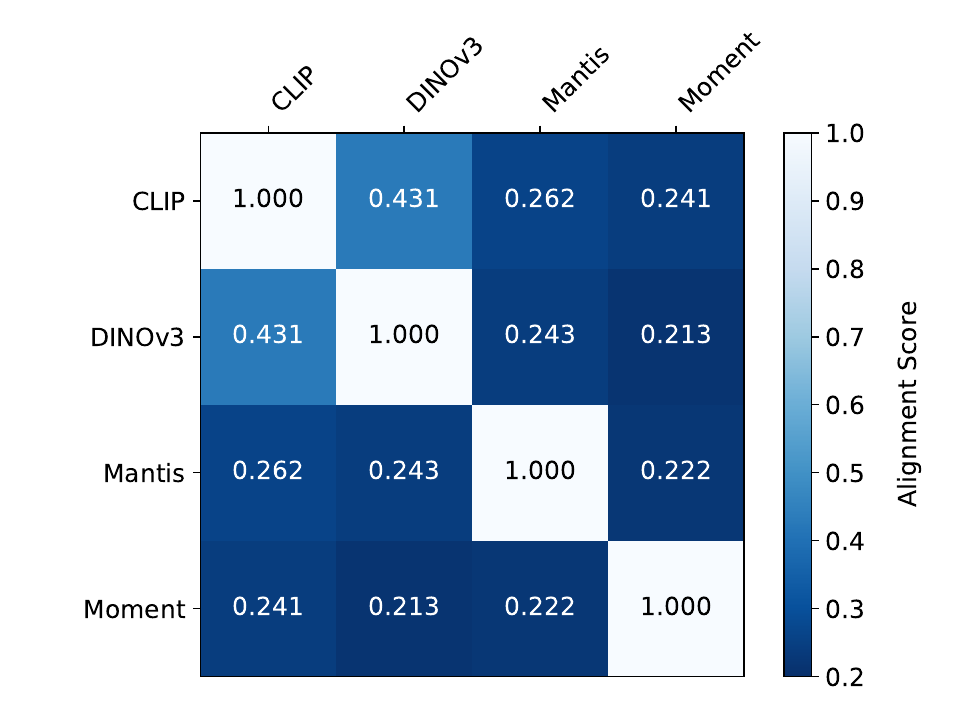}
        \caption{Pairwise alignment (mutual kNN).}
        \label{fig: alignment heatmap}
    \end{subfigure}
    \caption{The representations of frozen ViTs and TSFMs are concatenated and used in linear classification. Results are averaged over 128 datasets from the UCR benchmark.} 
    \label{fig:joint accuracy and alignment}
\end{figure}

\subsection{Representational Alignment of TiViT and TSFM}

In Table 2 of our main paper, we report the alignment and joint classification accuracy for TiViT and TSFMs. Figure \ref{fig:joint accuracy and alignment} is an additional visualization of the pairwise scores as heatmaps.

\subsection{Qualitative Feature Analysis}
\label{app: feature visualization}

In Section 4.6, we apply attention rollout to two samples from the ECG200 dataset, demonstrating that TiViT attends to salient regions of the time series images. Figure \ref{fig:attention_rollout_combined} further illustrates this behavior with three examples each from the AllGestureWiimoteX and ElectricDevices datasets, showing the original image, the corresponding attention rollout, and the overlay.

We further employ t-SNE to investigate the structure of the representations extracted by TiViT. Figure \ref{fig: additional tsne visualizations} presents t-SNE visualizations for 12 additional datasets.
The results underscore TiViT’s ability to uncover intrinsic cluster structures without access to labels and without being explicitly trained on time series.

Another way of understanding the features learned by ViTs is noise maximization. Ghiasi et al.~\cite{ghiasi2022vision} have generated images that highly activate a particular feature in ViTs starting from random noise. TiViT applies a frozen backbone and thus utilizes the exact same features of a ViT learned from natural images. Their visualizations underline that ViT-B captures general edges and textures in early layers, and more specialized objects in later layers. Please note that TiViT only uses the first six layers of ViT-B, where there are mostly patterns and less semantic components.

\begin{figure}
    \centering
    \begin{subfigure}{0.8\linewidth}
        \centering
        \includegraphics[width=0.8\linewidth]{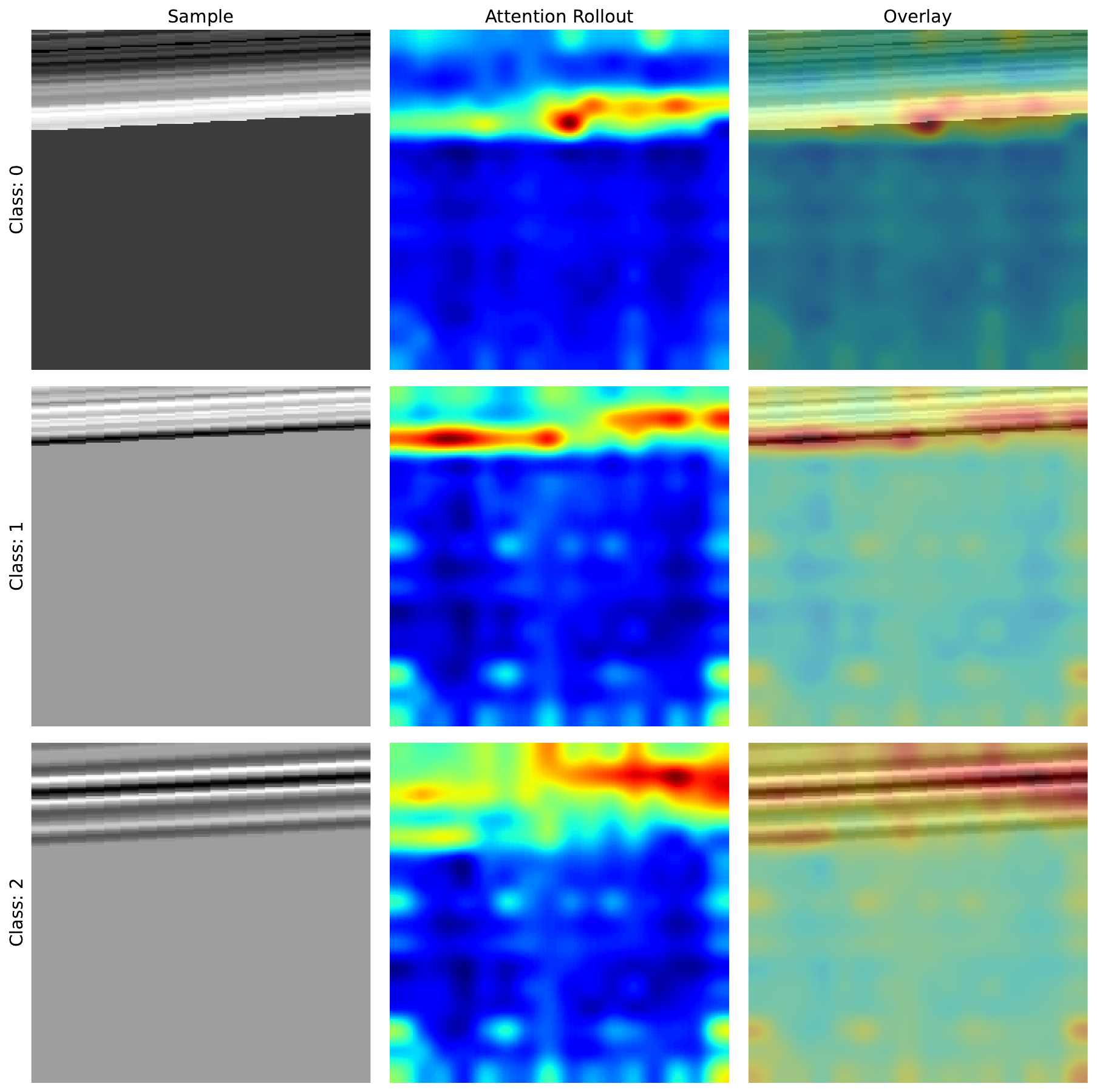}
        \caption{AllGestureWiimoteX}
        \label{fig:allgesturewiimotex_attention_rollout}
    \end{subfigure}
    \begin{subfigure}{0.8\linewidth}
        \centering
        \includegraphics[width=0.8\linewidth]{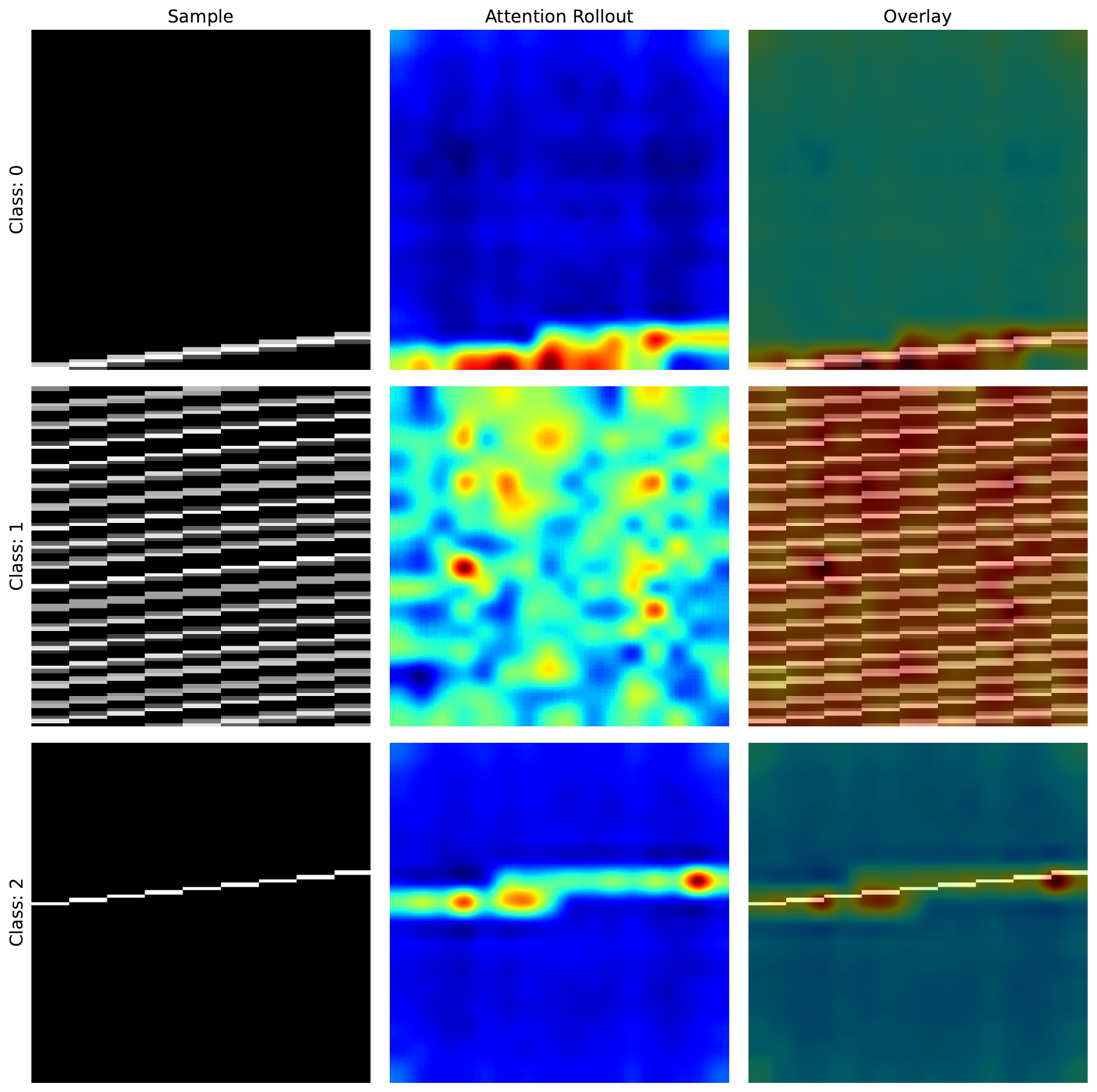}
        \caption{ElectricDevices}
        \label{fig:electricdevices_attention_rollout}
    \end{subfigure}
    \caption{Attention rollout.}
    \label{fig:attention_rollout_combined}
\end{figure}

\begin{figure}[htbp]
    \centering
    \begin{subfigure}{0.3\textwidth}
        \includegraphics[width=\linewidth]{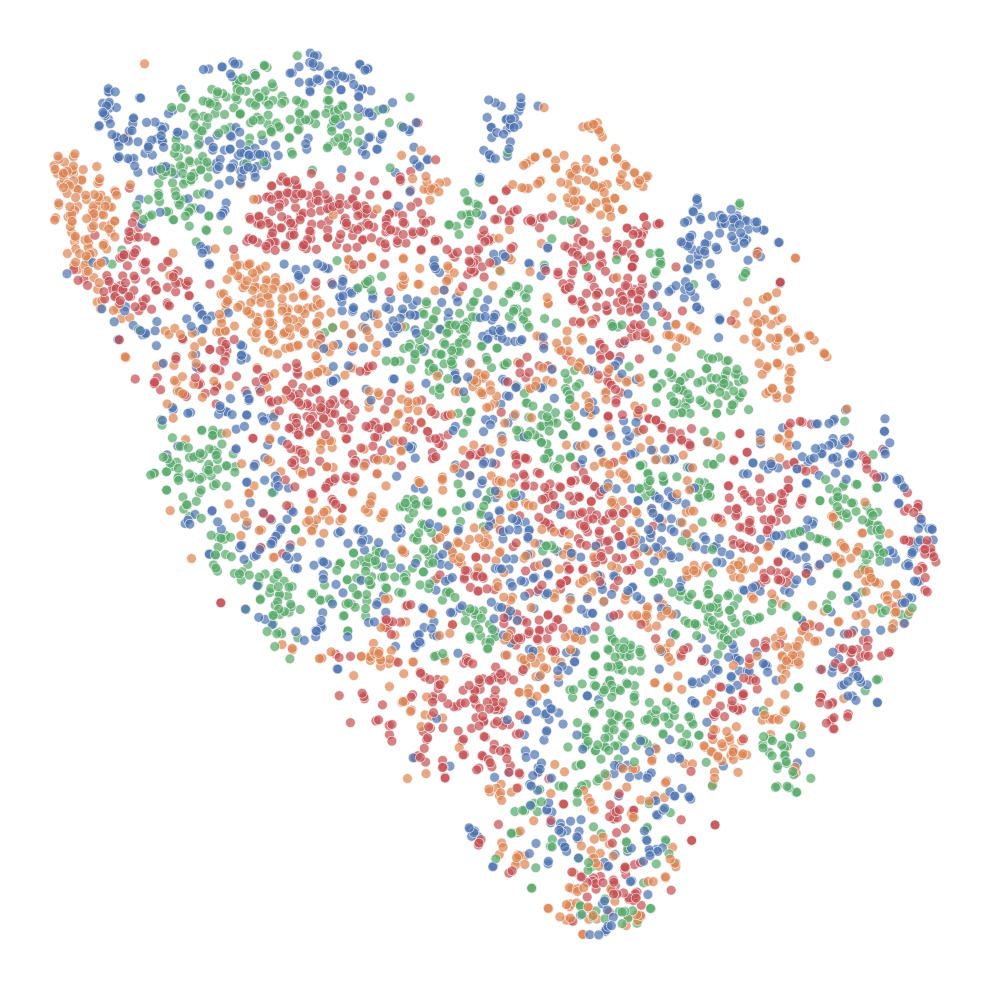}
        \caption{TwoPatterns}
    \end{subfigure}
    \begin{subfigure}{0.3\textwidth}
        \includegraphics[width=\linewidth]{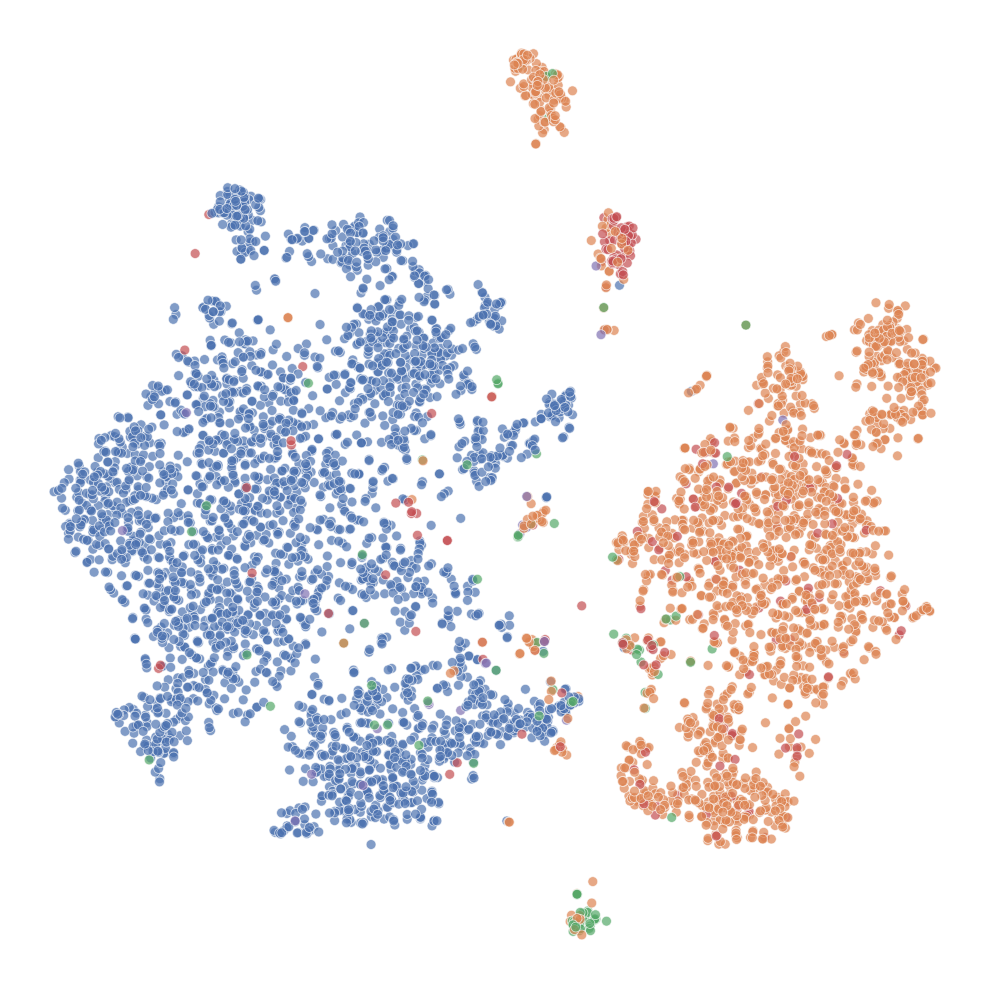}
        \caption{ECG5000}
    \end{subfigure}
    \begin{subfigure}{0.3\textwidth}
        \includegraphics[width=\linewidth]{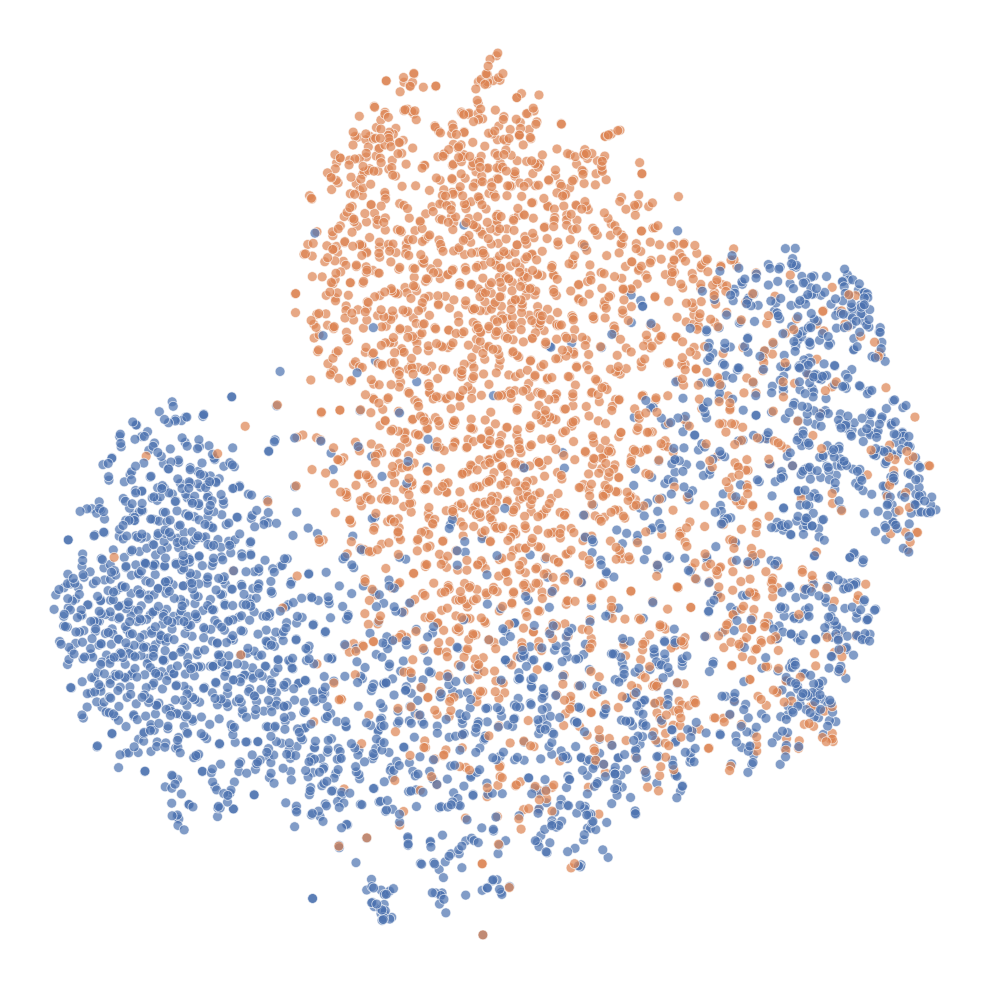}
        \caption{FordA}
    \end{subfigure}
    \begin{subfigure}{0.3\textwidth}
        \includegraphics[width=\linewidth]{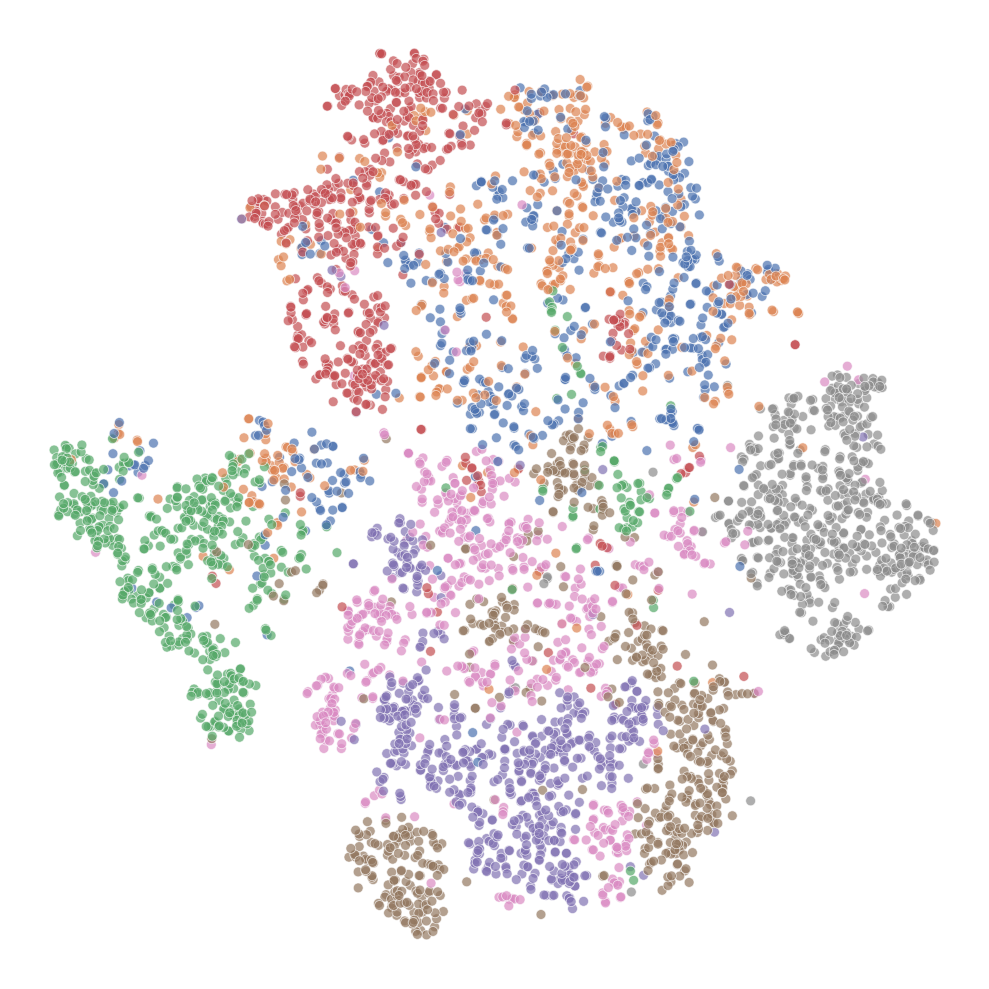}
        \caption{UWaveGestureLibX}
    \end{subfigure}
    \begin{subfigure}{0.3\textwidth}
        \includegraphics[width=\linewidth]{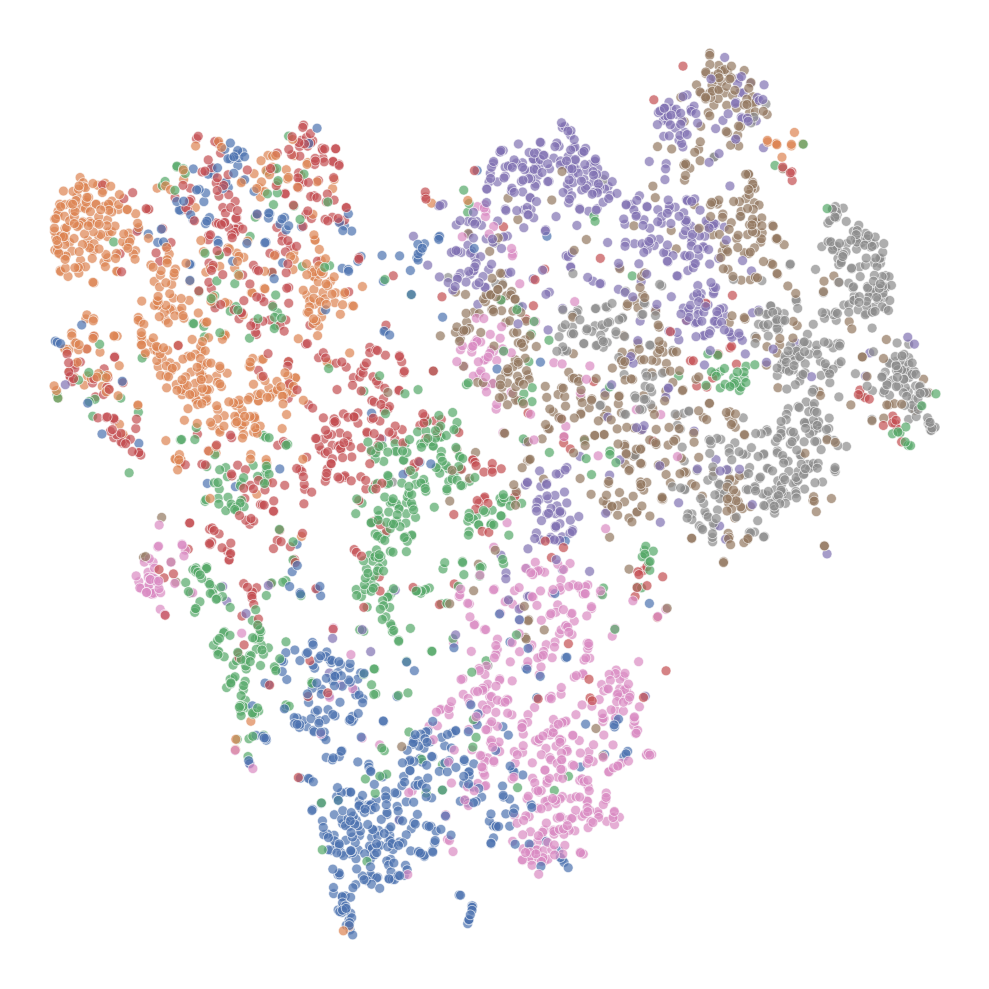}
        \caption{UWaveGestureLibY}
    \end{subfigure}
    \begin{subfigure}{0.3\textwidth}
        \includegraphics[width=\linewidth]{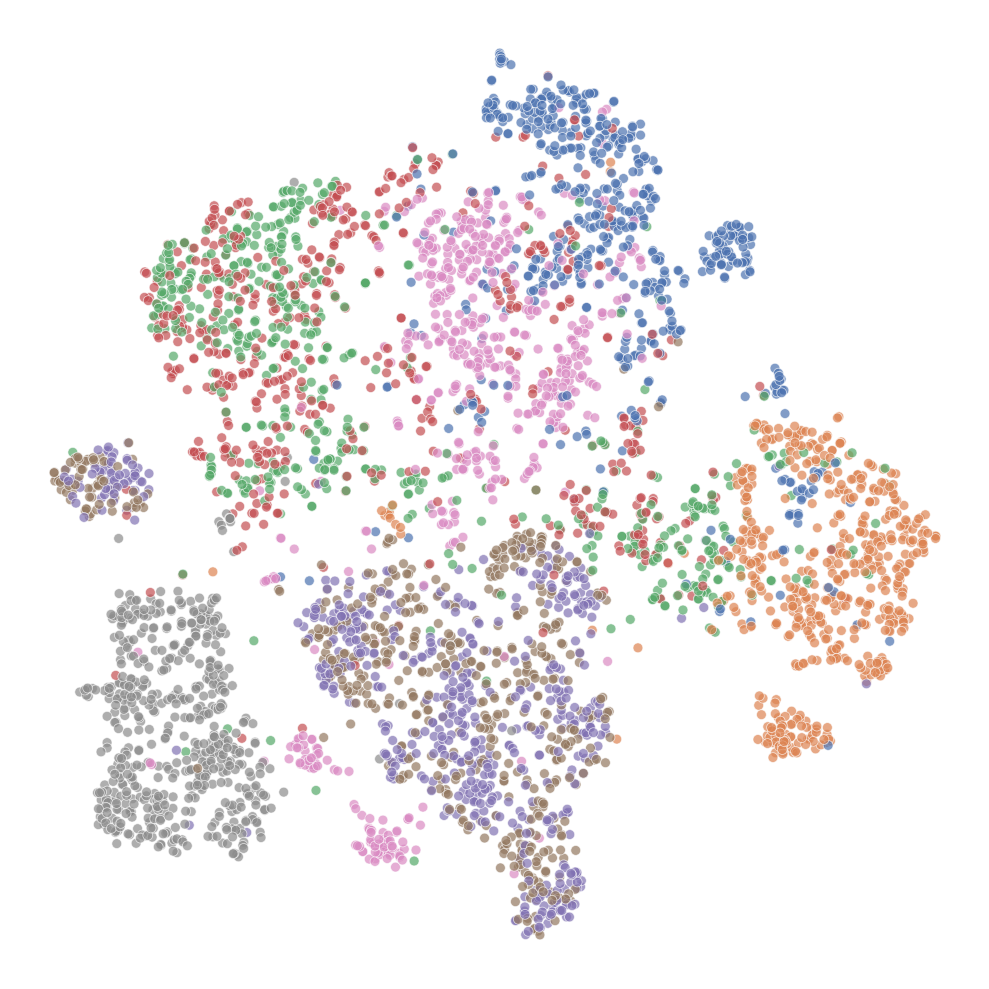}
        \caption{UWaveGestureLibZ}
    \end{subfigure}
    \begin{subfigure}{0.3\textwidth}
        \includegraphics[width=\linewidth]{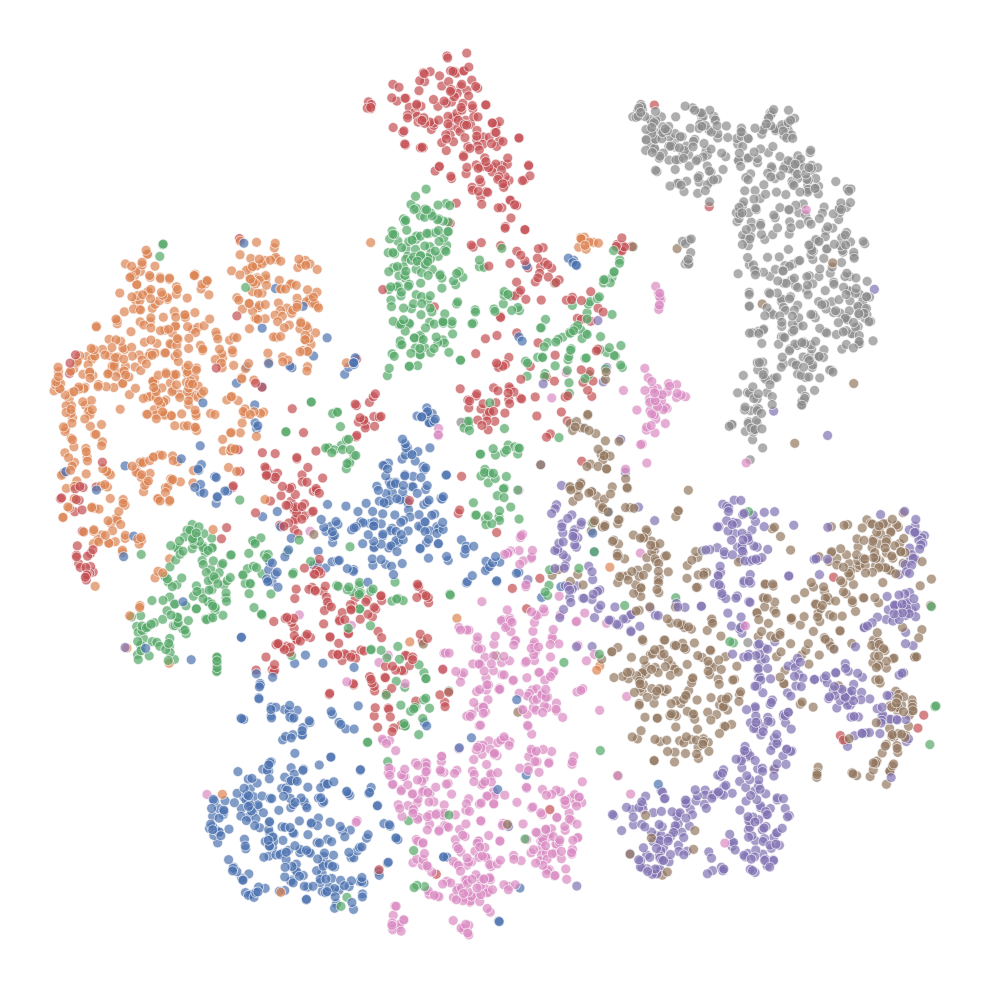}
        \caption{UWaveGestureLibAll}
    \end{subfigure}
    \begin{subfigure}{0.3\textwidth}
        \includegraphics[width=\linewidth]{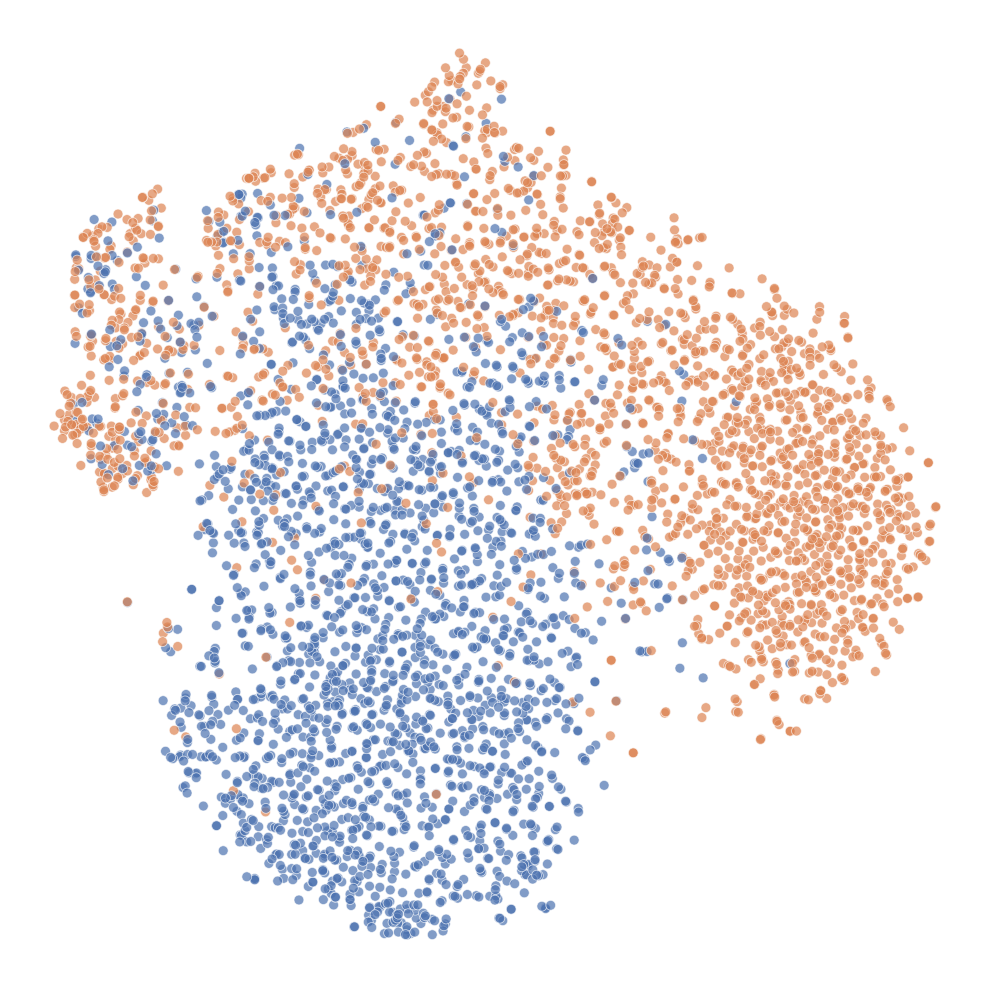}
        \caption{FordB}
    \end{subfigure}
    \begin{subfigure}{0.3\textwidth}
        \includegraphics[width=\linewidth]{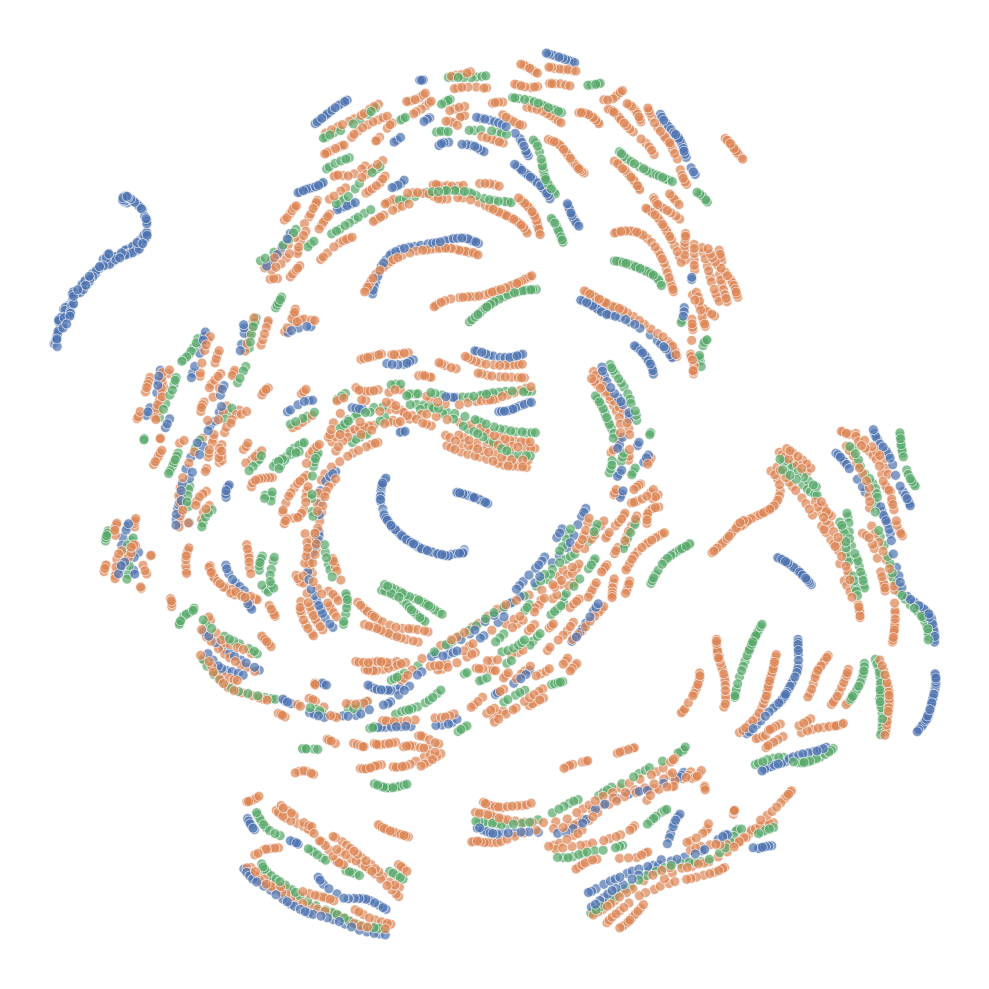}
        \caption{ChlorineConcentration}
    \end{subfigure}
    \begin{subfigure}{0.3\textwidth}
        \includegraphics[width=\linewidth]{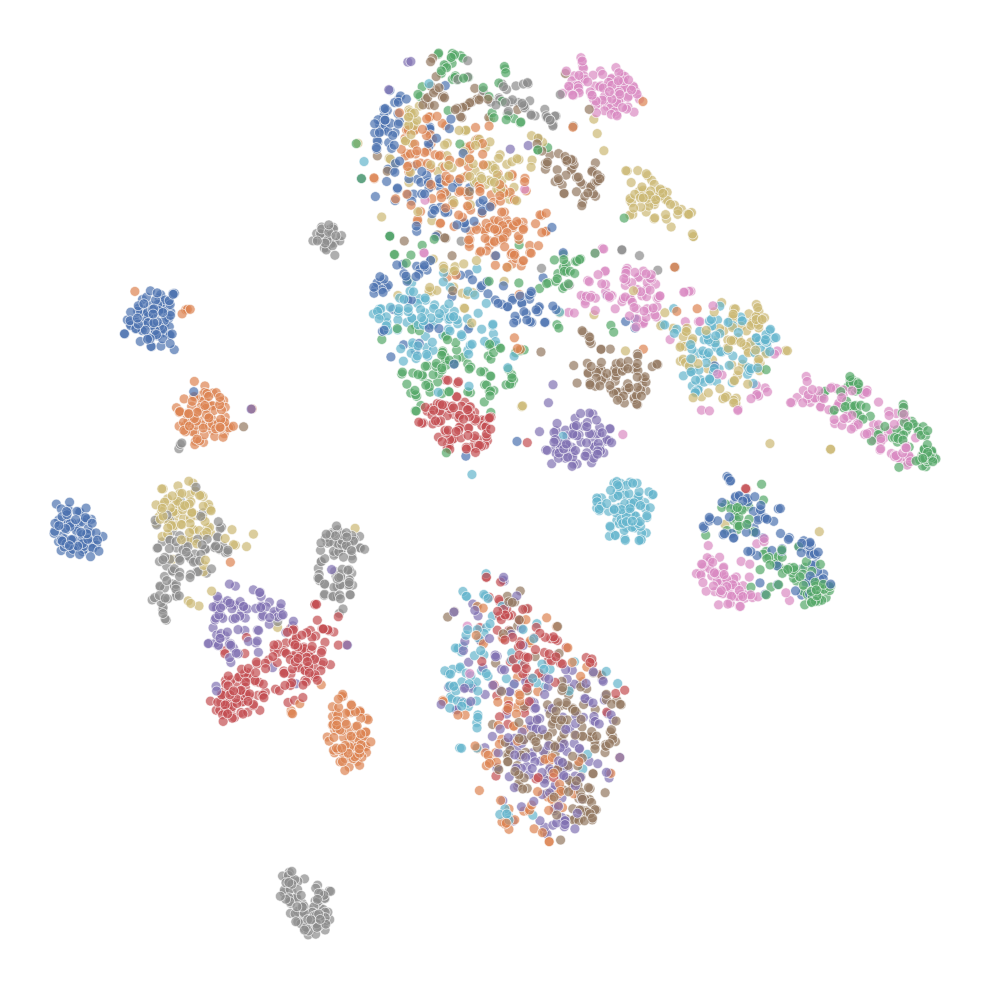}
        \caption{NonInvFetalECG1}
    \end{subfigure}
    \begin{subfigure}{0.3\textwidth}
        \includegraphics[width=\linewidth]{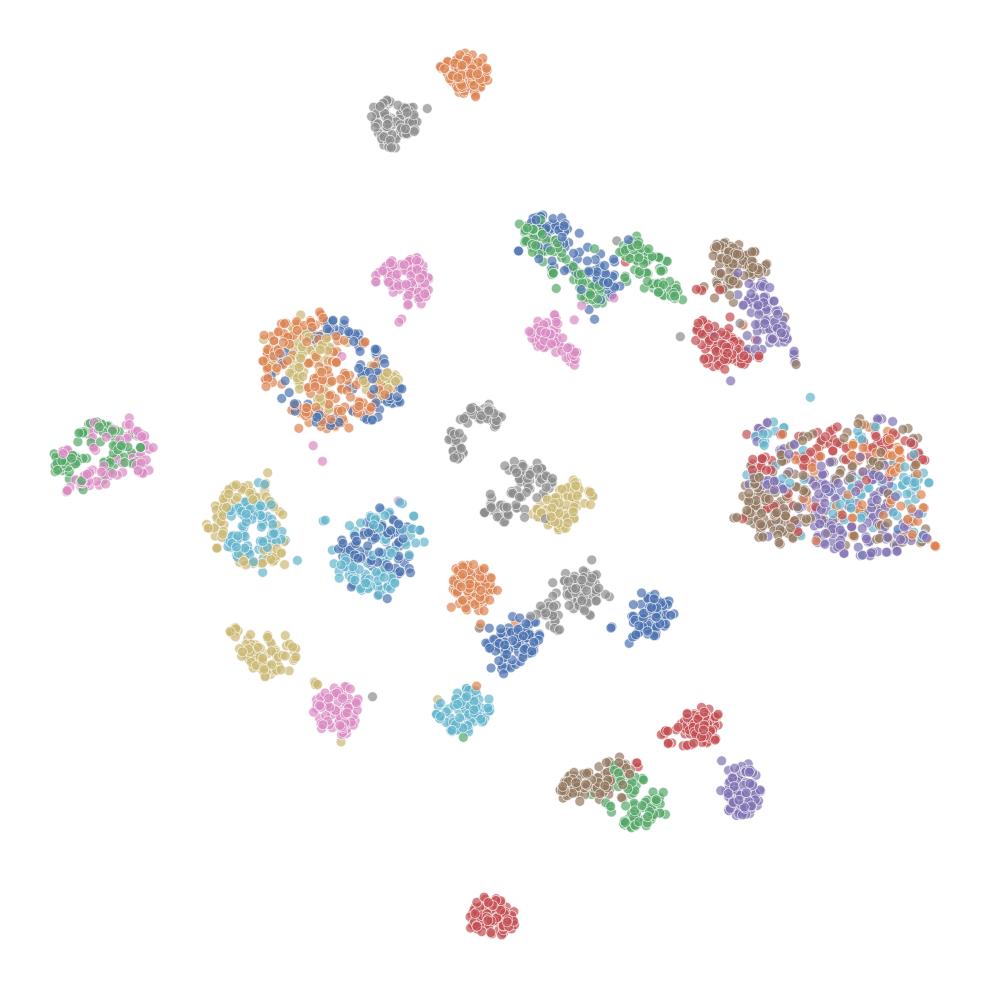}
        \caption{NonInvFetalECG2}
    \end{subfigure}
    \begin{subfigure}{0.3\textwidth}
        \includegraphics[width=\linewidth]{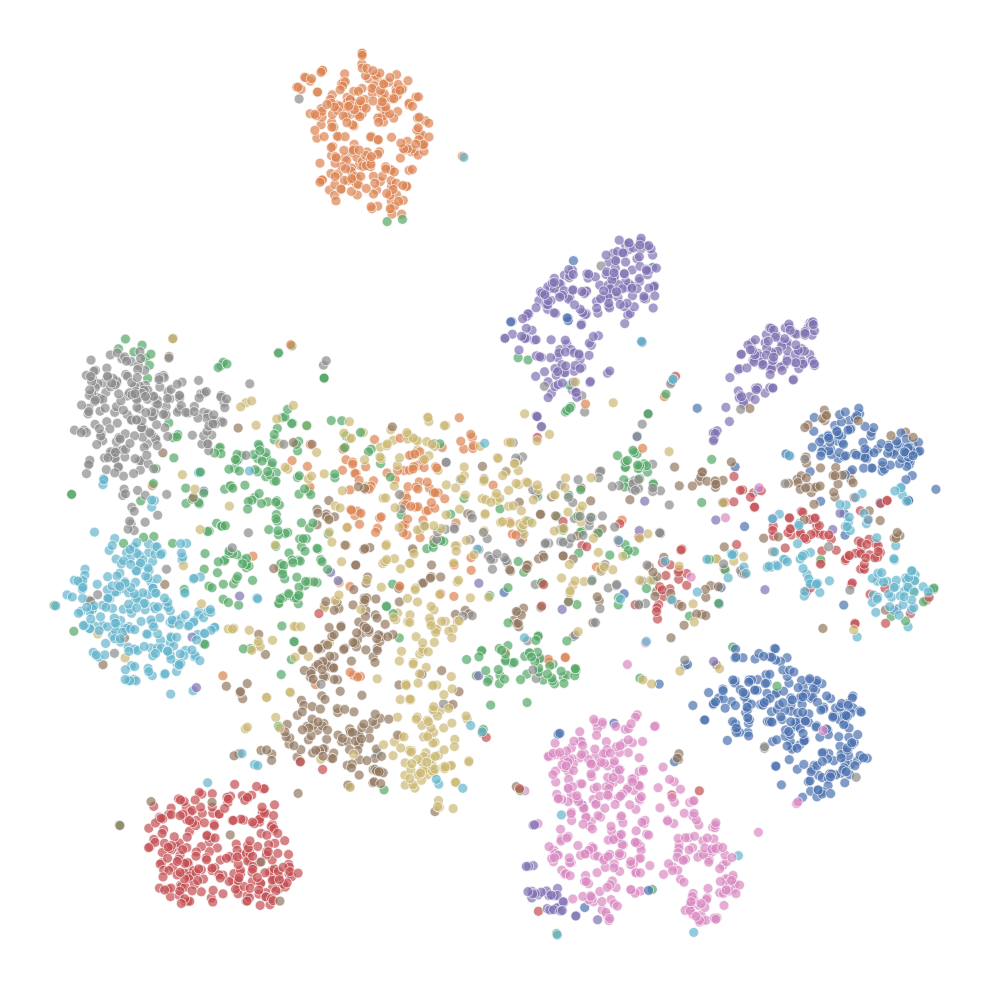}
        \caption{MelbournePedestrian}
    \end{subfigure}
    
    \caption{t-SNE visualizations.}
    \label{fig: additional tsne visualizations}
\end{figure}

\clearpage
\section{Detailed Benchmarking Results}
\label{app: ucr uea full benchmark}

In the main paper, we report the average accuracy of TiViT and TSFM across 128 univariate datasets from the UCR archive and 27 multivariate datasets from the UEA archive. Here, we report the full linear classification benchmark with accuracy scores for Mantis, Moment, TiViT, and their combinations on each dataset.
Table~\ref{tab:detailed-ucr-benchmark} presents the performance on the UCR dataset, Table~\ref{tab: detailed uea benchmark} reports the results on the UEA dataset, and Table~\ref{tab:woods_in_domain_ood} reports the in-domain and out-of-domain results on the WOODS dataset. Additionally, Table \ref{tab: mean rank} provides the mean rank of all five methods on both benchmarks. If multiple elements share the same rank, we assign them the lowest rank in the group.
Comparisons with supervised and self-supervised baselines are provided in Table \ref{tab:ucr_91_comparison_self_supervised} for the UCR benchmark and in Table \ref{tab:uea_29_comparison_self_supervised} for the UEA benchmark.

Furthermore, we assess the performance of TiViT against baseline methods in time series forecasting on 8 standard datasets (Table \ref{tab:long_horizon_forecasting}) and in time series anomaly detection on 248 datasets from the UCR Anomaly Archive (Table \ref{tab:anomaly_detection}).

\tiny
\begin{longtable}[c]{lccccc}
\caption{Classification accuracy for 128 univariate datasets from the UCR benchmark. We report the mean and standard deviation across three random seeds.\label{tab:detailed-ucr-benchmark}}\\
\toprule
Dataset & Moment & Mantis & TiViT & TiViT + Moment & TiViT + Mantis \\
\midrule
\endfirsthead

\toprule
\multicolumn{6}{c}{Continuation of Table \ref{tab:detailed-ucr-benchmark}}\\
\midrule
Dataset & Moment & Mantis & TiViT & TiViT + Moment & TiViT + Mantis \\
\midrule
\endhead

\midrule
\endfoot

\midrule
\multicolumn{6}{c}{End of Table}\\
\bottomrule
\endlastfoot

ACSF1 & 0.673 $\pm$ 0.012 & 0.667 $\pm$ 0.021 & \textbf{0.773} $\pm$ 0.015 & \textbf{0.773} $\pm$ 0.006 & 0.757 $\pm$ 0.015 \\
Adiac & 0.728 $\pm$ 0.004 & 0.728 $\pm$ 0.011 & 0.708 $\pm$ 0.009 & \textbf{0.732} $\pm$ 0.008 & 0.730 $\pm$ 0.012 \\
AllGestureWiimoteX & 0.686 $\pm$ 0.010 & 0.699 $\pm$ 0.003 & 0.685 $\pm$ 0.010 & 0.717 $\pm$ 0.009 & \textbf{0.726} $\pm$ 0.001 \\
AllGestureWiimoteY & 0.710 $\pm$ 0.006 & 0.742 $\pm$ 0.007 & 0.721 $\pm$ 0.015 & 0.750 $\pm$ 0.022 & \textbf{0.760} $\pm$ 0.014 \\
AllGestureWiimoteZ & 0.605 $\pm$ 0.007 & 0.673 $\pm$ 0.018 & 0.658 $\pm$ 0.015 & 0.690 $\pm$ 0.014 & \textbf{0.700} $\pm$ 0.014 \\
ArrowHead & 0.804 $\pm$ 0.012 & 0.745 $\pm$ 0.007 & 0.819 $\pm$ 0.049 & \textbf{0.851} $\pm$ 0.015 & 0.829 $\pm$ 0.035 \\
BME & 0.936 $\pm$ 0.010 & 0.991 $\pm$ 0.010 & 0.991 $\pm$ 0.015 & 0.987 $\pm$ 0.018 & \textbf{0.996} $\pm$ 0.008 \\
Beef & 0.667 $\pm$ 0.067 & 0.689 $\pm$ 0.019 & \textbf{0.800} $\pm$ 0.067 & \textbf{0.800} $\pm$ 0.000 & 0.789 $\pm$ 0.069 \\
BeetleFly & 0.850 $\pm$ 0.050 & 0.867 $\pm$ 0.058 & 0.917 $\pm$ 0.058 & 0.917 $\pm$ 0.058 & \textbf{0.950} $\pm$ 0.000 \\
BirdChicken & 0.883 $\pm$ 0.029 & \textbf{0.950} $\pm$ 0.000 & 0.917 $\pm$ 0.029 & 0.900 $\pm$ 0.000 & 0.933 $\pm$ 0.029 \\
CBF & 0.907 $\pm$ 0.030 & 0.990 $\pm$ 0.009 & \textbf{0.999} $\pm$ 0.001 & 0.997 $\pm$ 0.004 & \textbf{0.999} $\pm$ 0.001 \\
Car & 0.856 $\pm$ 0.035 & 0.828 $\pm$ 0.010 & 0.844 $\pm$ 0.010 & 0.878 $\pm$ 0.010 & \textbf{0.889} $\pm$ 0.025 \\
Chinatown & 0.962 $\pm$ 0.003 & \textbf{0.964} $\pm$ 0.006 & 0.950 $\pm$ 0.018 & 0.954 $\pm$ 0.025 & \textbf{0.964} $\pm$ 0.010 \\
ChlorineConcentration & 0.733 $\pm$ 0.010 & 0.643 $\pm$ 0.009 & 0.728 $\pm$ 0.008 & \textbf{0.744} $\pm$ 0.012 & 0.738 $\pm$ 0.000 \\
CinCECGTorso & 0.719 $\pm$ 0.056 & 0.727 $\pm$ 0.021 & \textbf{0.868} $\pm$ 0.034 & 0.837 $\pm$ 0.063 & 0.860 $\pm$ 0.039 \\
Coffee & \textbf{1.000} $\pm$ 0.000 & \textbf{1.000} $\pm$ 0.000 & \textbf{1.000} $\pm$ 0.000 & \textbf{1.000} $\pm$ 0.000 & \textbf{1.000} $\pm$ 0.000 \\
Computers & 0.712 $\pm$ 0.036 & 0.740 $\pm$ 0.012 & \textbf{0.785} $\pm$ 0.005 & 0.784 $\pm$ 0.011 & 0.781 $\pm$ 0.023 \\
CricketX & 0.706 $\pm$ 0.020 & 0.726 $\pm$ 0.015 & 0.753 $\pm$ 0.006 & 0.757 $\pm$ 0.013 & \textbf{0.765} $\pm$ 0.011 \\
CricketY & 0.693 $\pm$ 0.018 & 0.732 $\pm$ 0.017 & 0.765 $\pm$ 0.006 & 0.776 $\pm$ 0.008 & \textbf{0.783} $\pm$ 0.012 \\
CricketZ & 0.740 $\pm$ 0.016 & 0.721 $\pm$ 0.009 & 0.773 $\pm$ 0.017 & 0.779 $\pm$ 0.006 & \textbf{0.791} $\pm$ 0.012 \\
Crop & 0.709 $\pm$ 0.003 & 0.695 $\pm$ 0.001 & 0.675 $\pm$ 0.001 & \textbf{0.714} $\pm$ 0.003 & 0.707 $\pm$ 0.002 \\
DiatomSizeReduction & 0.900 $\pm$ 0.030 & 0.881 $\pm$ 0.032 & \textbf{0.949} $\pm$ 0.055 & 0.935 $\pm$ 0.048 & 0.944 $\pm$ 0.054 \\
DistalPhalanxOutlineAgeGroup & 0.743 $\pm$ 0.011 & \textbf{0.746} $\pm$ 0.017 & 0.703 $\pm$ 0.015 & 0.729 $\pm$ 0.011 & 0.717 $\pm$ 0.011 \\
DistalPhalanxOutlineCorrect & 0.762 $\pm$ 0.017 & 0.728 $\pm$ 0.007 & \textbf{0.769} $\pm$ 0.029 & 0.766 $\pm$ 0.008 & 0.757 $\pm$ 0.014 \\
DistalPhalanxTW & 0.643 $\pm$ 0.004 & \textbf{0.698} $\pm$ 0.007 & 0.640 $\pm$ 0.012 & 0.671 $\pm$ 0.011 & 0.626 $\pm$ 0.019 \\
DodgerLoopDay & 0.442 $\pm$ 0.014 & \textbf{0.517} $\pm$ 0.036 & 0.488 $\pm$ 0.043 & 0.467 $\pm$ 0.014 & 0.508 $\pm$ 0.040 \\
DodgerLoopGame & 0.691 $\pm$ 0.062 & 0.720 $\pm$ 0.018 & 0.797 $\pm$ 0.045 & 0.766 $\pm$ 0.073 & \textbf{0.802} $\pm$ 0.061 \\
DodgerLoopWeekend & \textbf{0.986} $\pm$ 0.013 & 0.978 $\pm$ 0.007 & 0.959 $\pm$ 0.011 & 0.981 $\pm$ 0.008 & 0.969 $\pm$ 0.015 \\
ECG200 & 0.843 $\pm$ 0.006 & 0.840 $\pm$ 0.017 & \textbf{0.863} $\pm$ 0.006 & 0.847 $\pm$ 0.031 & 0.847 $\pm$ 0.021 \\
ECG5000 & 0.934 $\pm$ 0.002 & 0.926 $\pm$ 0.005 & 0.934 $\pm$ 0.002 & \textbf{0.936} $\pm$ 0.003 & \textbf{0.936} $\pm$ 0.004 \\
ECGFiveDays & 0.919 $\pm$ 0.059 & 0.967 $\pm$ 0.012 & 0.953 $\pm$ 0.030 & \textbf{0.972} $\pm$ 0.032 & 0.959 $\pm$ 0.028 \\
EOGHorizontalSignal & 0.559 $\pm$ 0.012 & 0.542 $\pm$ 0.014 & 0.598 $\pm$ 0.008 & 0.634 $\pm$ 0.008 & \textbf{0.642} $\pm$ 0.012 \\
EOGVerticalSignal & 0.462 $\pm$ 0.021 & \textbf{0.530} $\pm$ 0.013 & 0.445 $\pm$ 0.006 & 0.476 $\pm$ 0.016 & 0.471 $\pm$ 0.008 \\
Earthquakes & \textbf{0.734} $\pm$ 0.025 & 0.707 $\pm$ 0.018 & 0.698 $\pm$ 0.007 & 0.717 $\pm$ 0.008 & 0.703 $\pm$ 0.017 \\
ElectricDevices & 0.626 $\pm$ 0.006 & 0.698 $\pm$ 0.003 & \textbf{0.757} $\pm$ 0.009 & 0.741 $\pm$ 0.003 & 0.748 $\pm$ 0.007 \\
EthanolLevel & \textbf{0.649} $\pm$ 0.008 & 0.433 $\pm$ 0.004 & 0.574 $\pm$ 0.008 & 0.617 $\pm$ 0.013 & 0.586 $\pm$ 0.008 \\
FaceAll & 0.724 $\pm$ 0.006 & \textbf{0.797} $\pm$ 0.007 & 0.741 $\pm$ 0.005 & 0.743 $\pm$ 0.005 & 0.762 $\pm$ 0.007 \\
FaceFour & 0.826 $\pm$ 0.076 & \textbf{0.958} $\pm$ 0.007 & 0.871 $\pm$ 0.029 & 0.909 $\pm$ 0.034 & 0.936 $\pm$ 0.035 \\
FacesUCR & 0.789 $\pm$ 0.010 & 0.888 $\pm$ 0.003 & 0.881 $\pm$ 0.007 & 0.881 $\pm$ 0.004 & \textbf{0.912} $\pm$ 0.004 \\
FiftyWords & 0.733 $\pm$ 0.015 & 0.736 $\pm$ 0.010 & 0.758 $\pm$ 0.013 & 0.788 $\pm$ 0.003 & \textbf{0.796} $\pm$ 0.006 \\
Fish & 0.949 $\pm$ 0.000 & 0.954 $\pm$ 0.000 & 0.952 $\pm$ 0.007 & 0.945 $\pm$ 0.020 & \textbf{0.968} $\pm$ 0.013 \\
FordA & 0.915 $\pm$ 0.002 & 0.910 $\pm$ 0.003 & 0.915 $\pm$ 0.003 & \textbf{0.927} $\pm$ 0.004 & 0.917 $\pm$ 0.000 \\
FordB & 0.801 $\pm$ 0.004 & 0.769 $\pm$ 0.002 & \textbf{0.812} $\pm$ 0.005 & 0.809 $\pm$ 0.007 & 0.800 $\pm$ 0.012 \\
FreezerRegularTrain & 0.973 $\pm$ 0.011 & 0.976 $\pm$ 0.012 & \textbf{0.997} $\pm$ 0.002 & 0.996 $\pm$ 0.005 & \textbf{0.997} $\pm$ 0.002 \\
FreezerSmallTrain & 0.840 $\pm$ 0.012 & 0.870 $\pm$ 0.020 & \textbf{0.992} $\pm$ 0.004 & 0.982 $\pm$ 0.006 & 0.990 $\pm$ 0.003 \\
Fungi & 0.753 $\pm$ 0.033 & 0.810 $\pm$ 0.025 & 0.787 $\pm$ 0.022 & 0.806 $\pm$ 0.014 & \textbf{0.812} $\pm$ 0.023 \\
GestureMidAirD1 & 0.659 $\pm$ 0.012 & 0.664 $\pm$ 0.027 & 0.746 $\pm$ 0.013 & 0.731 $\pm$ 0.023 & \textbf{0.756} $\pm$ 0.032 \\
GestureMidAirD2 & 0.567 $\pm$ 0.016 & 0.585 $\pm$ 0.040 & 0.667 $\pm$ 0.012 & 0.644 $\pm$ 0.032 & \textbf{0.669} $\pm$ 0.015 \\
GestureMidAirD3 & 0.359 $\pm$ 0.019 & 0.392 $\pm$ 0.013 & \textbf{0.472} $\pm$ 0.016 & 0.449 $\pm$ 0.016 & 0.464 $\pm$ 0.025 \\
GesturePebbleZ1 & 0.893 $\pm$ 0.015 & 0.917 $\pm$ 0.003 & 0.895 $\pm$ 0.006 & 0.924 $\pm$ 0.000 & \textbf{0.928} $\pm$ 0.003 \\
GesturePebbleZ2 & 0.846 $\pm$ 0.018 & \textbf{0.895} $\pm$ 0.007 & 0.840 $\pm$ 0.010 & 0.861 $\pm$ 0.035 & 0.892 $\pm$ 0.017 \\
GunPoint & 0.984 $\pm$ 0.027 & 0.987 $\pm$ 0.007 & \textbf{0.996} $\pm$ 0.004 & 0.987 $\pm$ 0.012 & \textbf{0.996} $\pm$ 0.004 \\
GunPointAgeSpan & 0.980 $\pm$ 0.008 & \textbf{0.998} $\pm$ 0.002 & 0.992 $\pm$ 0.002 & 0.993 $\pm$ 0.002 & 0.994 $\pm$ 0.000 \\
GunPointMaleVersusFemale & \textbf{1.000} $\pm$ 0.000 & 0.999 $\pm$ 0.002 & 0.996 $\pm$ 0.002 & \textbf{1.000} $\pm$ 0.000 & \textbf{1.000} $\pm$ 0.000 \\
GunPointOldVersusYoung & \textbf{1.000} $\pm$ 0.000 & \textbf{1.000} $\pm$ 0.000 & 0.988 $\pm$ 0.002 & \textbf{1.000} $\pm$ 0.000 & \textbf{1.000} $\pm$ 0.000 \\
Ham & \textbf{0.752} $\pm$ 0.025 & 0.667 $\pm$ 0.010 & 0.695 $\pm$ 0.000 & 0.721 $\pm$ 0.024 & 0.724 $\pm$ 0.019 \\
HandOutlines & 0.930 $\pm$ 0.007 & 0.931 $\pm$ 0.006 & 0.936 $\pm$ 0.007 & \textbf{0.945} $\pm$ 0.010 & 0.932 $\pm$ 0.007 \\
Haptics & 0.491 $\pm$ 0.026 & 0.462 $\pm$ 0.002 & 0.498 $\pm$ 0.007 & 0.535 $\pm$ 0.040 & \textbf{0.539} $\pm$ 0.009 \\
Herring & \textbf{0.698} $\pm$ 0.018 & 0.682 $\pm$ 0.024 & 0.599 $\pm$ 0.009 & 0.630 $\pm$ 0.039 & 0.625 $\pm$ 0.027 \\
HouseTwenty & 0.947 $\pm$ 0.010 & 0.961 $\pm$ 0.010 & 0.972 $\pm$ 0.005 & 0.972 $\pm$ 0.010 & \textbf{0.980} $\pm$ 0.005 \\
InlineSkate & 0.364 $\pm$ 0.019 & 0.334 $\pm$ 0.021 & 0.398 $\pm$ 0.015 & 0.401 $\pm$ 0.006 & \textbf{0.408} $\pm$ 0.015 \\
InsectEPGRegularTrain & 0.987 $\pm$ 0.014 & \textbf{1.000} $\pm$ 0.000 & \textbf{1.000} $\pm$ 0.000 & \textbf{1.000} $\pm$ 0.000 & \textbf{1.000} $\pm$ 0.000 \\
InsectEPGSmallTrain & 0.953 $\pm$ 0.008 & \textbf{1.000} $\pm$ 0.000 & 0.968 $\pm$ 0.007 & 0.973 $\pm$ 0.005 & 0.999 $\pm$ 0.002 \\
InsectWingbeatSound & 0.539 $\pm$ 0.003 & 0.470 $\pm$ 0.019 & 0.536 $\pm$ 0.015 & \textbf{0.560} $\pm$ 0.007 & 0.539 $\pm$ 0.010 \\
ItalyPowerDemand & \textbf{0.938} $\pm$ 0.005 & 0.910 $\pm$ 0.006 & 0.920 $\pm$ 0.018 & 0.936 $\pm$ 0.011 & 0.923 $\pm$ 0.018 \\
LargeKitchenAppliances & 0.859 $\pm$ 0.005 & 0.820 $\pm$ 0.010 & \textbf{0.883} $\pm$ 0.014 & 0.873 $\pm$ 0.018 & 0.879 $\pm$ 0.014 \\
Lightning2 & 0.760 $\pm$ 0.041 & 0.781 $\pm$ 0.025 & 0.803 $\pm$ 0.028 & \textbf{0.820} $\pm$ 0.028 & 0.803 $\pm$ 0.016 \\
Lightning7 & 0.836 $\pm$ 0.036 & 0.749 $\pm$ 0.021 & 0.831 $\pm$ 0.021 & \textbf{0.881} $\pm$ 0.008 & 0.822 $\pm$ 0.024 \\
Mallat & 0.915 $\pm$ 0.010 & 0.868 $\pm$ 0.028 & 0.956 $\pm$ 0.017 & \textbf{0.963} $\pm$ 0.016 & 0.958 $\pm$ 0.018 \\
Meat & 0.911 $\pm$ 0.038 & \textbf{0.939} $\pm$ 0.019 & 0.800 $\pm$ 0.000 & 0.900 $\pm$ 0.029 & 0.850 $\pm$ 0.044 \\
MedicalImages & 0.730 $\pm$ 0.003 & 0.707 $\pm$ 0.024 & 0.740 $\pm$ 0.006 & \textbf{0.780} $\pm$ 0.006 & 0.761 $\pm$ 0.014 \\
MelbournePedestrian & \textbf{0.933} $\pm$ 0.003 & 0.908 $\pm$ 0.005 & 0.862 $\pm$ 0.006 & 0.932 $\pm$ 0.005 & 0.925 $\pm$ 0.003 \\
MiddlePhalanxOutlineAgeGroup & 0.489 $\pm$ 0.029 & \textbf{0.587} $\pm$ 0.019 & 0.537 $\pm$ 0.036 & 0.530 $\pm$ 0.004 & 0.571 $\pm$ 0.023 \\
MiddlePhalanxOutlineCorrect & 0.816 $\pm$ 0.009 & \textbf{0.845} $\pm$ 0.009 & 0.789 $\pm$ 0.015 & 0.792 $\pm$ 0.016 & 0.805 $\pm$ 0.016 \\
MiddlePhalanxTW & 0.506 $\pm$ 0.019 & 0.442 $\pm$ 0.017 & 0.506 $\pm$ 0.023 & 0.498 $\pm$ 0.025 & \textbf{0.511} $\pm$ 0.010 \\
MixedShapesRegularTrain & 0.947 $\pm$ 0.004 & 0.955 $\pm$ 0.006 & 0.974 $\pm$ 0.002 & 0.973 $\pm$ 0.003 & \textbf{0.976} $\pm$ 0.002 \\
MixedShapesSmallTrain & 0.882 $\pm$ 0.004 & 0.904 $\pm$ 0.002 & 0.950 $\pm$ 0.002 & 0.937 $\pm$ 0.004 & \textbf{0.957} $\pm$ 0.003 \\
MoteStrain & 0.889 $\pm$ 0.028 & 0.895 $\pm$ 0.026 & 0.875 $\pm$ 0.021 & \textbf{0.918} $\pm$ 0.008 & 0.901 $\pm$ 0.025 \\
NonInvasiveFetalECGThorax1 & 0.919 $\pm$ 0.002 & 0.797 $\pm$ 0.006 & 0.884 $\pm$ 0.004 & \textbf{0.924} $\pm$ 0.003 & 0.885 $\pm$ 0.009 \\
NonInvasiveFetalECGThorax2 & 0.927 $\pm$ 0.002 & 0.817 $\pm$ 0.004 & 0.915 $\pm$ 0.001 & \textbf{0.934} $\pm$ 0.004 & 0.918 $\pm$ 0.005 \\
OSULeaf & 0.917 $\pm$ 0.004 & 0.899 $\pm$ 0.005 & 0.977 $\pm$ 0.006 & 0.972 $\pm$ 0.010 & \textbf{0.978} $\pm$ 0.009 \\
OliveOil & \textbf{0.856} $\pm$ 0.051 & 0.822 $\pm$ 0.107 & 0.656 $\pm$ 0.077 & 0.778 $\pm$ 0.019 & 0.711 $\pm$ 0.051 \\
PLAID & 0.775 $\pm$ 0.017 & 0.852 $\pm$ 0.001 & 0.888 $\pm$ 0.008 & 0.901 $\pm$ 0.011 & \textbf{0.928} $\pm$ 0.012 \\
PhalangesOutlinesCorrect & \textbf{0.795} $\pm$ 0.006 & 0.794 $\pm$ 0.008 & 0.789 $\pm$ 0.004 & \textbf{0.795} $\pm$ 0.008 & 0.787 $\pm$ 0.004 \\
Phoneme & 0.277 $\pm$ 0.003 & 0.293 $\pm$ 0.008 & 0.377 $\pm$ 0.006 & 0.372 $\pm$ 0.003 & \textbf{0.386} $\pm$ 0.006 \\
PickupGestureWiimoteZ & 0.713 $\pm$ 0.042 & 0.767 $\pm$ 0.023 & 0.887 $\pm$ 0.031 & 0.847 $\pm$ 0.046 & \textbf{0.893} $\pm$ 0.023 \\
PigAirwayPressure & 0.109 $\pm$ 0.007 & 0.588 $\pm$ 0.012 & 0.540 $\pm$ 0.006 & 0.447 $\pm$ 0.013 & \textbf{0.598} $\pm$ 0.010 \\
PigArtPressure & 0.780 $\pm$ 0.010 & 0.827 $\pm$ 0.017 & 0.817 $\pm$ 0.013 & 0.833 $\pm$ 0.019 & \textbf{0.846} $\pm$ 0.005 \\
PigCVP & 0.747 $\pm$ 0.027 & 0.753 $\pm$ 0.007 & 0.702 $\pm$ 0.019 & 0.761 $\pm$ 0.018 & \textbf{0.801} $\pm$ 0.012 \\
Plane & 0.997 $\pm$ 0.005 & \textbf{1.000} $\pm$ 0.000 & \textbf{1.000} $\pm$ 0.000 & \textbf{1.000} $\pm$ 0.000 & \textbf{1.000} $\pm$ 0.000 \\
PowerCons & 0.931 $\pm$ 0.006 & 0.933 $\pm$ 0.010 & 0.894 $\pm$ 0.022 & \textbf{0.943} $\pm$ 0.013 & 0.906 $\pm$ 0.020 \\
ProximalPhalanxOutlineAgeGroup & 0.802 $\pm$ 0.020 & \textbf{0.852} $\pm$ 0.007 & 0.833 $\pm$ 0.027 & 0.824 $\pm$ 0.005 & 0.828 $\pm$ 0.017 \\
ProximalPhalanxOutlineCorrect & 0.883 $\pm$ 0.010 & \textbf{0.885} $\pm$ 0.008 & 0.861 $\pm$ 0.020 & 0.871 $\pm$ 0.016 & 0.858 $\pm$ 0.023 \\
ProximalPhalanxTW & \textbf{0.767} $\pm$ 0.010 & 0.740 $\pm$ 0.015 & 0.751 $\pm$ 0.022 & 0.730 $\pm$ 0.010 & 0.759 $\pm$ 0.023 \\
RefrigerationDevices & 0.496 $\pm$ 0.017 & 0.526 $\pm$ 0.022 & 0.555 $\pm$ 0.007 & 0.531 $\pm$ 0.005 & \textbf{0.570} $\pm$ 0.014 \\
Rock & 0.727 $\pm$ 0.031 & 0.700 $\pm$ 0.060 & \textbf{0.873} $\pm$ 0.099 & \textbf{0.873} $\pm$ 0.115 & 0.853 $\pm$ 0.117 \\
ScreenType & 0.499 $\pm$ 0.020 & 0.468 $\pm$ 0.026 & 0.530 $\pm$ 0.014 & 0.516 $\pm$ 0.002 & \textbf{0.552} $\pm$ 0.027 \\
SemgHandGenderCh2 & 0.761 $\pm$ 0.018 & 0.883 $\pm$ 0.006 & 0.879 $\pm$ 0.001 & 0.878 $\pm$ 0.013 & \textbf{0.914} $\pm$ 0.006 \\
SemgHandMovementCh2 & 0.398 $\pm$ 0.010 & 0.654 $\pm$ 0.018 & 0.545 $\pm$ 0.016 & 0.538 $\pm$ 0.031 & \textbf{0.688} $\pm$ 0.024 \\
SemgHandSubjectCh2 & 0.648 $\pm$ 0.013 & 0.826 $\pm$ 0.005 & 0.840 $\pm$ 0.002 & 0.838 $\pm$ 0.012 & \textbf{0.895} $\pm$ 0.007 \\
ShakeGestureWiimoteZ & 0.887 $\pm$ 0.012 & 0.867 $\pm$ 0.012 & 0.827 $\pm$ 0.031 & \textbf{0.907} $\pm$ 0.031 & 0.840 $\pm$ 0.020 \\
ShapeletSim & 0.967 $\pm$ 0.010 & 0.919 $\pm$ 0.012 & \textbf{1.000} $\pm$ 0.000 & \textbf{1.000} $\pm$ 0.000 & \textbf{1.000} $\pm$ 0.000 \\
ShapesAll & 0.886 $\pm$ 0.003 & 0.844 $\pm$ 0.010 & 0.901 $\pm$ 0.003 & \textbf{0.913} $\pm$ 0.008 & 0.908 $\pm$ 0.007 \\
SmallKitchenAppliances & 0.733 $\pm$ 0.010 & 0.796 $\pm$ 0.013 & \textbf{0.830} $\pm$ 0.003 & 0.817 $\pm$ 0.018 & 0.812 $\pm$ 0.008 \\
SmoothSubspace & 0.898 $\pm$ 0.023 & \textbf{0.971} $\pm$ 0.004 & 0.956 $\pm$ 0.010 & 0.964 $\pm$ 0.010 & \textbf{0.971} $\pm$ 0.010 \\
SonyAIBORobotSurface1 & 0.834 $\pm$ 0.013 & 0.858 $\pm$ 0.015 & 0.890 $\pm$ 0.012 & 0.869 $\pm$ 0.009 & \textbf{0.896} $\pm$ 0.010 \\
SonyAIBORobotSurface2 & 0.855 $\pm$ 0.027 & 0.895 $\pm$ 0.012 & 0.911 $\pm$ 0.049 & 0.914 $\pm$ 0.049 & \textbf{0.923} $\pm$ 0.048 \\
StarLightCurves & 0.969 $\pm$ 0.003 & 0.968 $\pm$ 0.002 & 0.973 $\pm$ 0.002 & \textbf{0.976} $\pm$ 0.002 & \textbf{0.976} $\pm$ 0.002 \\
Strawberry & \textbf{0.972} $\pm$ 0.002 & 0.960 $\pm$ 0.004 & 0.959 $\pm$ 0.002 & 0.968 $\pm$ 0.006 & 0.959 $\pm$ 0.003 \\
SwedishLeaf & 0.915 $\pm$ 0.007 & 0.942 $\pm$ 0.006 & 0.955 $\pm$ 0.003 & \textbf{0.959} $\pm$ 0.006 & 0.958 $\pm$ 0.003 \\
Symbols & 0.957 $\pm$ 0.019 & 0.957 $\pm$ 0.031 & 0.966 $\pm$ 0.034 & \textbf{0.973} $\pm$ 0.020 & 0.967 $\pm$ 0.035 \\
SyntheticControl & 0.966 $\pm$ 0.004 & 0.992 $\pm$ 0.002 & 0.999 $\pm$ 0.002 & 0.993 $\pm$ 0.003 & \textbf{1.000} $\pm$ 0.000 \\
ToeSegmentation1 & \textbf{0.963} $\pm$ 0.007 & 0.952 $\pm$ 0.012 & 0.952 $\pm$ 0.012 & \textbf{0.963} $\pm$ 0.005 & 0.959 $\pm$ 0.009 \\
ToeSegmentation2 & 0.885 $\pm$ 0.015 & \textbf{0.954} $\pm$ 0.008 & 0.923 $\pm$ 0.008 & 0.895 $\pm$ 0.027 & 0.926 $\pm$ 0.004 \\
Trace & \textbf{1.000} $\pm$ 0.000 & \textbf{1.000} $\pm$ 0.000 & \textbf{1.000} $\pm$ 0.000 & \textbf{1.000} $\pm$ 0.000 & \textbf{1.000} $\pm$ 0.000 \\
TwoLeadECG & 0.901 $\pm$ 0.020 & 0.998 $\pm$ 0.002 & 0.997 $\pm$ 0.001 & 0.997 $\pm$ 0.001 & \textbf{1.000} $\pm$ 0.000 \\
TwoPatterns & 0.989 $\pm$ 0.001 & 0.946 $\pm$ 0.007 & 0.998 $\pm$ 0.000 & \textbf{0.999} $\pm$ 0.001 & 0.998 $\pm$ 0.001 \\
UMD & \textbf{0.993} $\pm$ 0.000 & \textbf{0.993} $\pm$ 0.000 & \textbf{0.993} $\pm$ 0.000 & \textbf{0.993} $\pm$ 0.000 & \textbf{0.993} $\pm$ 0.000 \\
UWaveGestureLibraryAll & 0.923 $\pm$ 0.002 & 0.874 $\pm$ 0.004 & 0.940 $\pm$ 0.001 & \textbf{0.950} $\pm$ 0.005 & 0.944 $\pm$ 0.003 \\
UWaveGestureLibraryX & 0.792 $\pm$ 0.001 & 0.779 $\pm$ 0.004 & 0.828 $\pm$ 0.004 & \textbf{0.838} $\pm$ 0.004 & \textbf{0.838} $\pm$ 0.002 \\
UWaveGestureLibraryY & 0.711 $\pm$ 0.006 & 0.678 $\pm$ 0.009 & 0.749 $\pm$ 0.004 & 0.758 $\pm$ 0.004 & \textbf{0.763} $\pm$ 0.006 \\
UWaveGestureLibraryZ & 0.731 $\pm$ 0.001 & 0.742 $\pm$ 0.009 & 0.770 $\pm$ 0.003 & 0.772 $\pm$ 0.004 & \textbf{0.786} $\pm$ 0.001 \\
Wafer & 0.992 $\pm$ 0.002 & 0.996 $\pm$ 0.000 & \textbf{1.000} $\pm$ 0.000 & \textbf{1.000} $\pm$ 0.000 & \textbf{1.000} $\pm$ 0.000 \\
Wine & \textbf{0.889} $\pm$ 0.019 & 0.796 $\pm$ 0.037 & 0.599 $\pm$ 0.065 & 0.747 $\pm$ 0.028 & 0.759 $\pm$ 0.049 \\
WordSynonyms & 0.655 $\pm$ 0.003 & 0.626 $\pm$ 0.017 & 0.649 $\pm$ 0.007 & \textbf{0.690} $\pm$ 0.005 & 0.681 $\pm$ 0.006 \\
Worms & 0.745 $\pm$ 0.033 & 0.710 $\pm$ 0.033 & 0.762 $\pm$ 0.027 & \textbf{0.805} $\pm$ 0.026 & 0.762 $\pm$ 0.052 \\
WormsTwoClass & 0.775 $\pm$ 0.037 & 0.745 $\pm$ 0.007 & 0.784 $\pm$ 0.020 & \textbf{0.792} $\pm$ 0.026 & 0.766 $\pm$ 0.022 \\
Yoga & 0.833 $\pm$ 0.008 & 0.771 $\pm$ 0.014 & 0.826 $\pm$ 0.009 & \textbf{0.852} $\pm$ 0.007 & 0.844 $\pm$ 0.007 \\

\end{longtable}

\begin{table}[ht]
\centering
\tiny
\caption{Classification accuracy for 27 multivariate datasets from the UEA benchmark. We report the mean and standard deviation across three random seeds.}
\label{tab: detailed uea benchmark}
\begin{tabular}{lccccc}
\toprule
Dataset & Moment & Mantis & TiViT & TiViT + Moment & TiViT + Mantis \\
\midrule
ArticularyWordRecognition & 0.988 $\pm$ 0.002 & \textbf{0.991} $\pm$ 0.002 & 0.977 $\pm$ 0.003 & 0.977 $\pm$ 0.003 & 0.974 $\pm$ 0.005 \\
BasicMotions & \textbf{1.000} $\pm$ 0.000 & \textbf{1.000} $\pm$ 0.000 & \textbf{1.000} $\pm$ 0.000 & \textbf{1.000} $\pm$ 0.000 & \textbf{1.000} $\pm$ 0.000 \\
CharacterTrajectories & \textbf{0.982} $\pm$ 0.001 & 0.973 $\pm$ 0.001 & 0.964 $\pm$ 0.005 & \textbf{0.982} $\pm$ 0.001 & 0.978 $\pm$ 0.005 \\
Cricket & \textbf{1.000} $\pm$ 0.000 & 0.986 $\pm$ 0.000 & \textbf{1.000} $\pm$ 0.000 & \textbf{1.000} $\pm$ 0.000 & \textbf{1.000} $\pm$ 0.000 \\
DuckDuckGeese & \textbf{0.467} $\pm$ 0.081 & 0.433 $\pm$ 0.023 & 0.393 $\pm$ 0.081 & 0.413 $\pm$ 0.064 & 0.433 $\pm$ 0.050 \\
ERing & 0.895 $\pm$ 0.022 & 0.905 $\pm$ 0.025 & 0.975 $\pm$ 0.014 & 0.977 $\pm$ 0.006 & \textbf{0.981} $\pm$ 0.007 \\
EigenWorms & 0.746 $\pm$ 0.022 & 0.746 $\pm$ 0.016 & \textbf{0.911} $\pm$ 0.016 & 0.880 $\pm$ 0.009 & \textbf{0.911} $\pm$ 0.012 \\
Epilepsy & \textbf{1.000} $\pm$ 0.000 & 0.990 $\pm$ 0.004 & \textbf{1.000} $\pm$ 0.000 & \textbf{1.000} $\pm$ 0.000 & \textbf{1.000} $\pm$ 0.000 \\
EthanolConcentration & 0.445 $\pm$ 0.013 & 0.269 $\pm$ 0.044 & \textbf{0.485} $\pm$ 0.012 & 0.473 $\pm$ 0.030 & 0.465 $\pm$ 0.019 \\
FaceDetection & 0.584 $\pm$ 0.007 & 0.592 $\pm$ 0.006 & 0.598 $\pm$ 0.004 & 0.584 $\pm$ 0.007 & \textbf{0.607} $\pm$ 0.005 \\
FingerMovements & \textbf{0.633} $\pm$ 0.045 & 0.593 $\pm$ 0.025 & 0.517 $\pm$ 0.040 & 0.620 $\pm$ 0.036 & 0.553 $\pm$ 0.050 \\
HandMovementDirection & \textbf{0.279} $\pm$ 0.051 & 0.212 $\pm$ 0.021 & 0.275 $\pm$ 0.016 & 0.257 $\pm$ 0.036 & 0.257 $\pm$ 0.027 \\
Handwriting & 0.296 $\pm$ 0.018 & \textbf{0.425} $\pm$ 0.013 & 0.307 $\pm$ 0.034 & 0.340 $\pm$ 0.002 & 0.385 $\pm$ 0.021 \\
Heartbeat & 0.735 $\pm$ 0.007 & \textbf{0.800} $\pm$ 0.017 & 0.732 $\pm$ 0.008 & 0.717 $\pm$ 0.022 & 0.769 $\pm$ 0.003 \\
InsectWingbeat & 0.231 $\pm$ 0.012 & \textbf{0.573} $\pm$ 0.017 & 0.355 $\pm$ 0.008 & 0.332 $\pm$ 0.018 & 0.443 $\pm$ 0.020 \\
JapaneseVowels & 0.918 $\pm$ 0.006 & \textbf{0.978} $\pm$ 0.003 & 0.940 $\pm$ 0.002 & 0.938 $\pm$ 0.012 & 0.933 $\pm$ 0.008 \\
LSST & 0.571 $\pm$ 0.005 & 0.607 $\pm$ 0.009 & 0.604 $\pm$ 0.005 & 0.610 $\pm$ 0.009 & \textbf{0.652} $\pm$ 0.003 \\
Libras & 0.861 $\pm$ 0.017 & 0.887 $\pm$ 0.026 & 0.907 $\pm$ 0.006 & \textbf{0.922} $\pm$ 0.022 & 0.920 $\pm$ 0.018 \\
MotorImagery & 0.530 $\pm$ 0.026 & \textbf{0.563} $\pm$ 0.012 & \textbf{0.563} $\pm$ 0.049 & 0.560 $\pm$ 0.044 & 0.553 $\pm$ 0.042 \\
NATOPS & 0.900 $\pm$ 0.029 & \textbf{0.931} $\pm$ 0.014 & 0.869 $\pm$ 0.006 & 0.889 $\pm$ 0.006 & 0.878 $\pm$ 0.006 \\
PEMS-SF & 0.705 $\pm$ 0.029 & \textbf{0.788} $\pm$ 0.029 & 0.709 $\pm$ 0.084 & 0.763 $\pm$ 0.044 & 0.742 $\pm$ 0.087 \\
PhonemeSpectra & 0.186 $\pm$ 0.004 & 0.272 $\pm$ 0.006 & 0.245 $\pm$ 0.007 & 0.265 $\pm$ 0.007 & \textbf{0.286} $\pm$ 0.008 \\
RacketSports & 0.829 $\pm$ 0.007 & \textbf{0.919} $\pm$ 0.004 & 0.846 $\pm$ 0.010 & 0.871 $\pm$ 0.008 & 0.879 $\pm$ 0.027 \\
SelfRegulationSCP1 & 0.762 $\pm$ 0.010 & 0.825 $\pm$ 0.022 & 0.858 $\pm$ 0.008 & 0.840 $\pm$ 0.003 & \textbf{0.891} $\pm$ 0.010 \\
SelfRegulationSCP2 & 0.509 $\pm$ 0.031 & 0.491 $\pm$ 0.018 & \textbf{0.526} $\pm$ 0.038 & 0.506 $\pm$ 0.017 & 0.517 $\pm$ 0.020 \\
SpokenArabicDigits & \textbf{0.981} $\pm$ 0.003 & 0.907 $\pm$ 0.006 & 0.969 $\pm$ 0.001 & 0.979 $\pm$ 0.003 & 0.972 $\pm$ 0.002 \\
UWaveGestureLibrary & 0.846 $\pm$ 0.010 & 0.879 $\pm$ 0.015 & 0.910 $\pm$ 0.005 & 0.902 $\pm$ 0.004 & \textbf{0.919} $\pm$ 0.009 \\
\bottomrule
\end{tabular}
\end{table}

\normalsize

\begin{table}[t]
\caption{Classification accuracy on EEG datasets from WOODS.}
\label{tab:woods_in_domain_ood}
\centering
\begin{subtable}[t]{0.48\linewidth}
\centering
\caption{In-domain}
\label{tab:woods_in_domain}
\resizebox{\linewidth}{!}{
\begin{tabular}{lcccc}
\toprule
Model & PCL & CAP & SEDFx & Mean \\
\midrule
Mantis         & 55.1 & 77.2 & 78.9 & 70.4 \\
Moment         & 55.6 & 76.4 & 78.2 & 70.1 \\
\midrule
TiViT          & 55.3 & 81.1 & 80.7 & 72.4 \\
\midrule
TiViT + Mantis & 57.3 & 82.2 & 81.5 & 73.7 \\
TiViT + Moment & 58.2 & 82.1 & 80.1 & 73.5 \\
\bottomrule
\end{tabular}
}
\end{subtable}
\hfill
\begin{subtable}[t]{0.48\linewidth}
\centering
\caption{Out-of-domain}
\label{tab:woods_out_of_domain}
\resizebox{\linewidth}{!}{
\begin{tabular}{lcccc}
\toprule
Model & PCL & CAP & SEDFx & Mean \\
\midrule
Mantis         & 53.4 & 65.2 & 76.5 & 65.0 \\
Moment         & 54.1 & 70.5 & 75.6 & 66.7 \\
\midrule
TiViT          & 54.7 & 70.5 & 77.8 & 67.7 \\
\midrule
TiViT + Mantis & 55.7 & 70.0 & 78.0 & 67.9 \\
TiViT + Moment & 55.5 & 72.8 & 76.7 & 68.3 \\
\bottomrule
\end{tabular}
}
\end{subtable}
\label{tab:woods_results}
\end{table}

\begin{table}
\centering
\caption{Mean rank of TiViT and TSFMs across datasets from the UCR and UEA archive.}
\label{tab: mean rank}
\begin{tabular}{lcc}
\toprule
Model & UCR & UEA \\
\midrule
Moment & 3.75 & 3.33\\
Mantis & 3.43 & 2.85\\
\midrule
TiViT \emph{(Ours)} & 2.97 & 2.85\\
\midrule
TiViT + Moment \emph{(Ours)} & 2.20 & 2.63\\
TiViT + Mantis \emph{(Ours)} & \textbf{1.95} & \textbf{2.22} \\
\bottomrule
\end{tabular}
\end{table}

\begin{table}
\centering
\caption{Classification accuracy across 91 UCR datasets. Baselines from \cite{goswami2024moment}.}
\label{tab:ucr_91_comparison_self_supervised}
\resizebox{.95\linewidth}{!}{
\begin{tabular}{lcccccccccc}
\toprule
Accuracy & TiViT + Mantis & TiViT & Mantis & MOMENT & TimesNet & GPT4TS & TS2Vec & T-Loss & TNC & TS-TCC \\
\midrule
Mean   & 0.848 & 0.834 & 0.826 & 0.794 & 0.572 & 0.566 & \bfseries 0.851 & 0.833 & 0.786 & 0.793 \\
Median & \bfseries 0.880 & 0.849 & 0.852 & 0.815 & 0.565 & 0.583 & 0.871 & 0.849 & 0.788 & 0.802 \\
Std.   & 0.133 & 0.136 & 0.143 & 0.147 & 0.238 & 0.234 & 0.134 & 0.136 & 0.168 & 0.176 \\
\midrule
\addlinespace[0.3cm]
\midrule
Accuracy & TST & CNN & Encoder & FCN & MCNN & MLP & ResNet & t-LeNet & TWIESN & DTW \\
\midrule
Mean   & 0.658 & 0.751 & 0.743 & 0.809 & 0.702 & 0.750 & 0.825 & 0.348 & 0.726 & 0.764 \\
Median & 0.720 & 0.773 & 0.753 & 0.837 & 0.718 & 0.766 & 0.852 & 0.333 & 0.724 & 0.768 \\
Std.   & 0.220 & 0.180 & 0.159 & 0.188 & 0.194 & 0.169 & 0.177 & 0.221 & 0.164 & 0.152 \\
\bottomrule
\end{tabular}
}
\end{table}

\begin{table}
\centering
\caption{Classification accuracy across 29 UEA datasets. Baselines from \cite{goswami2024moment}.}
\label{tab:uea_29_comparison_self_supervised}
\resizebox{.95\linewidth}{!}{
\begin{tabular}{lcccccccccc}
\toprule
Accuracy & TiViT + Mantis & TiViT & Mantis & MOMENT & TS2Vec & T-Loss & TNC & TS-TCC & TST & DTW \\
\midrule
Mean   & \bfseries 0.719 & 0.706 & 0.693 & 0.670 & 0.694 & 0.646 & 0.660 & 0.657 & 0.605 & 0.638 \\
Median & \bfseries 0.823 & 0.789 & 0.788 & 0.722 & 0.683 & 0.676 & 0.746 & 0.751 & 0.620 & 0.664 \\
Std.   & 0.260 & 0.266 & 0.266 & 0.274 & 0.255 & 0.296 & 0.267 & 0.263 & 0.294 & 0.296 \\
\bottomrule
\end{tabular}
}
\end{table}

\begin{table*}[t!]
\caption{Long-term forecasting. Baselines from \cite{goswami2024moment}.}
\centering
\resizebox{\linewidth}{!}{
\begin{tabular}{cc|cc|cc|cc|cc|cc|cc|cc|cc|cc}
\toprule
\multicolumn{2}{c}{\multirow{2}{*}{Method}} & \multicolumn{6}{c}{\textbf{Pretraining + linear probing}} & \multicolumn{12}{c}{\textbf{Supervised training}} \\
\cmidrule(lr){3-8} \cmidrule(lr){9-20}
\multicolumn{2}{c}{} & \multicolumn{2}{c}{\textbf{TiViT (Ours)}} & \multicolumn{2}{c}{\textbf{MOMENT}} & \multicolumn{2}{c}{\textbf{GPT4TS}} & \multicolumn{2}{c}{\textbf{PatchTST}} & \multicolumn{2}{c}{\textbf{DLinear}} & \multicolumn{2}{c}{\textbf{TimesNet}} & \multicolumn{2}{c}{\textbf{FEDFormer}} & \multicolumn{2}{c}{\textbf{Stationary}} & \multicolumn{2}{c}{\textbf{N-BEATS}} \\
\multicolumn{2}{c}{Metric} & MSE & MAE & MSE & MAE & MSE & MAE & MSE & MAE & MSE & MAE & MSE & MAE & MSE & MAE & MSE & MAE & MSE & MAE \\ \midrule
\multirow{4}{*}{Weather} & 96 & 0.153 & 0.211 & 0.154 & 0.209 & 0.162 & 0.212 & 0.149 & 0.198 & 0.176 & 0.237 & 0.172 & 0.220 & 0.217 & 0.296 & 0.173 & 0.223 & 0.152 & 0.210 \\
 & 192 & 0.196 & 0.247 & 0.197 & 0.248 & 0.204 & 0.248 & 0.194 & 0.241 & 0.220 & 0.282 & 0.219 & 0.261 & 0.276 & 0.336 & 0.245 & 0.285 & 0.199 & 0.260 \\
 & 336 & 0.248 & 0.285 & 0.246 & 0.285 & 0.254 & 0.286 & 0.245 & 0.282 & 0.265 & 0.319 & 0.280 & 0.306 & 0.339 & 0.380 & 0.321 & 0.338 & 0.258 & 0.311 \\
 & 720 & 0.321 & 0.337 & 0.315 & 0.336 & 0.326 & 0.337 & 0.314 & 0.334 & 0.333 & 0.362 & 0.365 & 0.359 & 0.403 & 0.428 & 0.414 & 0.410 & 0.331 & 0.359 \\ \midrule
 \multirow{4}{*}{ECL} & 96 & 0.140 & 0.240 & 0.136 & 0.233 & 0.139 & 0.238 & 0.129 & 0.222 & 0.140 & 0.237 & 0.168 & 0.272 & 0.193 & 0.308 & 0.169 & 0.273 & 0.131 & 0.228 \\
 & 192 & 0.152 & 0.251 & 0.152 & 0.247 & 0.153 & 0.251 & 0.157 & 0.240 & 0.153 & 0.249 & 0.184 & 0.289 & 0.201 & 0.315 & 0.182 & 0.286 & 0.153 & 0.248 \\
 & 336 & 0.168 & 0.267 & 0.167 & 0.264 & 0.169 & 0.266 & 0.163 & 0.259 & 0.169 & 0.267 & 0.198 & 0.300 & 0.214 & 0.329 & 0.200 & 0.304 & 0.170 & 0.267 \\
 & 720 & 0.204 & 0.297 & 0.205 & 0.295 & 0.206 & 0.297 & 0.197 & 0.290 & 0.203 & 0.301 & 0.220 & 0.320 & 0.246 & 0.355 & 0.222 & 0.321 & 0.208 & 0.298 \\ \midrule
\multirow{4}{*}{Traffic} & 96 & 0.384 & 0.274 & 0.391 & 0.282 & 0.388 & 0.282 & 0.360 & 0.249 & 0.410 & 0.282 & 0.593 & 0.321 & 0.587 & 0.366 & 0.612 & 0.338 & 0.375 & 0.259 \\
 & 192 & 0.398 & 0.280 & 0.404 & 0.287 & 0.407 & 0.290 & 0.379 & 0.256 & 0.423 & 0.287 & 0.617 & 0.336 & 0.604 & 0.373 & 0.613 & 0.340 & 0.403 & 0.274 \\
 & 336 & 0.407 & 0.285 & 0.414 & 0.292 & 0.412 & 0.294 & 0.392 & 0.264 & 0.436 & 0.296 & 0.629 & 0.336 & 0.621 & 0.383 & 0.618 & 0.328 & 0.426 & 0.285 \\
 & 720 & 0.443 & 0.303 & 0.450 & 0.310 & 0.450 & 0.312 & 0.432 & 0.286 & 0.466 & 0.315 & 0.640 & 0.350 & 0.626 & 0.382 & 0.653 & 0.355 & 0.508 & 0.335 \\
\bottomrule
\multirow{4}{*}{ETTh1} & 96 & 0.391 & 0.417 & 0.387 & 0.410 & 0.376 & 0.397 & 0.370 & 0.399 & 0.375 & 0.399 & 0.384 & 0.402 & 0.376 & 0.419 & 0.513 & 0.491 & 0.399 & 0.428 \\
 & 192 & 0.411 & 0.430 & 0.410 & 0.426 & 0.416 & 0.418 & 0.413 & 0.421 & 0.405 & 0.416 & 0.436 & 0.429 & 0.420 & 0.448 & 0.534 & 0.504 & 0.451 & 0.464 \\
 & 336 & 0.425 & 0.442 & 0.422 & 0.437 & 0.442 & 0.433 & 0.422 & 0.436 & 0.439 & 0.443 & 0.491 & 0.469 & 0.459 & 0.465 & 0.588 & 0.535 & 0.498 & 0.500 \\
 & 720 & 0.447 & 0.469 & 0.454 & 0.472 & 0.477 & 0.456 & 0.447 & 0.466 & 0.472 & 0.490 & 0.521 & 0.500 & 0.506 & 0.507 & 0.643 & 0.616 & 0.608 & 0.573 \\ \midrule
\multirow{4}{*}{ETTh2} & 96 & 0.319 & 0.375 & 0.288 & 0.345 & 0.285 & 0.342 & 0.274 & 0.336 & 0.289 & 0.353 & 0.340 & 0.374 & 0.358 & 0.397 & 0.476 & 0.458 & 0.327 & 0.387 \\
 & 192 & 0.363 & 0.406 & 0.349 & 0.386 & 0.354 & 0.389 & 0.339 & 0.379 & 0.383 & 0.418 & 0.402 & 0.414 & 0.429 & 0.439 & 0.512 & 0.493 & 0.400 & 0.435 \\
 & 336 & 0.372 & 0.418 & 0.369 & 0.408 & 0.373 & 0.407 & 0.329 & 0.380 & 0.448 & 0.465 & 0.452 & 0.452 & 0.496 & 0.487 & 0.552 & 0.551 & 0.747 & 0.599 \\
 & 720 & 0.407 & 0.447 & 0.403 & 0.439 & 0.406 & 0.441 & 0.379 & 0.422 & 0.605 & 0.551 & 0.462 & 0.468 & 0.463 & 0.474 & 0.562 & 0.560 & 1.454 & 0.847 \\ \midrule
\multirow{4}{*}{ETTm1} & 96 & 0.315 & 0.367 & 0.293 & 0.349 & 0.292 & 0.346 & 0.290 & 0.342 & 0.299 & 0.343 & 0.338 & 0.375 & 0.379 & 0.419 & 0.386 & 0.398 & 0.318 & 0.367 \\
 & 192 & 0.352 & 0.387 & 0.326 & 0.368 & 0.332 & 0.372 & 0.332 & 0.369 & 0.335 & 0.365 & 0.374 & 0.387 & 0.426 & 0.441 & 0.459 & 0.444 & 0.355 & 0.391 \\
 & 336 & 0.381 & 0.404 & 0.352 & 0.384 & 0.366 & 0.394 & 0.366 & 0.392 & 0.369 & 0.386 & 0.410 & 0.411 & 0.445 & 0.459 & 0.495 & 0.464 & 0.401 & 0.419 \\
 & 720 & 0.437 & 0.436 & 0.405 & 0.416 & 0.417 & 0.421 & 0.416 & 0.420 & 0.425 & 0.421 & 0.478 & 0.450 & 0.543 & 0.490 & 0.585 & 0.516 & 0.448 & 0.448 \\ \midrule
\multirow{4}{*}{ETTm2} & 96 & 0.189 & 0.277 & 0.170 & 0.260 & 0.173 & 0.262 & 0.165 & 0.255 & 0.167 & 0.269 & 0.187 & 0.267 & 0.203 & 0.287 & 0.192 & 0.274 & 0.197 & 0.271 \\
 & 192 & 0.252 & 0.318 & 0.227 & 0.297 & 0.229 & 0.301 & 0.220 & 0.292 & 0.224 & 0.303 & 0.249 & 0.309 & 0.269 & 0.328 & 0.280 & 0.339 & 0.285 & 0.328 \\
 & 336 & 0.301 & 0.351 & 0.275 & 0.328 & 0.286 & 0.341 & 0.274 & 0.329 & 0.281 & 0.342 & 0.321 & 0.351 & 0.325 & 0.366 & 0.334 & 0.361 & 0.338 & 0.366 \\
 & 720 & 0.382 & 0.405 & 0.363 & 0.387 & 0.378 & 0.401 & 0.362 & 0.385 & 0.397 & 0.421 & 0.408 & 0.403 & 0.421 & 0.415 & 0.417 & 0.413 & 0.395 & 0.419 \\ \midrule
\multirow{4}{*}{ILI} & 24 & 2.822 & 1.142 & 2.728 & 1.114 & 2.063 & 0.881 & 1.319 & 0.754 & 2.215 & 1.081 & 2.317 & 0.934 & 3.228 & 1.260 & 2.294 & 0.945 & 4.539 & 1.528 \\
 & 36 & 2.862 & 1.143 & 2.669 & 1.092 & 1.868 & 0.892 & 1.430 & 0.834 & 1.963 & 0.963 & 1.972 & 0.920 & 2.679 & 1.080 & 1.825 & 0.848 & 4.628 & 1.534 \\
 & 48 & 2.846 & 1.123 & 2.728 & 1.098 & 1.790 & 0.884 & 1.553 & 0.815 & 2.130 & 1.024 & 2.238 & 0.940 & 2.622 & 1.078 & 2.010 & 0.900 & 4.957 & 1.585 \\
 & 60 & 3.023 & 1.155 & 2.883 & 1.126 & 1.979 & 0.957 & 1.470 & 0.788 & 2.368 & 1.096 & 2.027 & 0.928 & 2.857 & 1.157 & 2.178 & 0.963 & 5.429 & 1.661 \\ \midrule
\end{tabular}}
\label{tab:long_horizon_forecasting}
\end{table*}

\begin{table}
\caption{Anomaly detection performance across 41 datasets from the UCR Anomaly Archive measured using adjusted best $F_1$ and VUS-ROC. Bold indicates the best performance per dataset/metric. Baselines from \cite{goswami2024moment}.}
\centering
\resizebox{\linewidth}{!}{
\begin{tabular}{lcccccccccc}
\toprule
& \multicolumn{5}{c}{Adjusted Best $F_1$} & \multicolumn{5}{c}{VUS-ROC} \\
\cmidrule(lr){2-6} \cmidrule(lr){7-11}
& TiViT & Anomaly TF & Moment & GPT4TS & TimesNet & TiViT & AnomalyTF & Moment & GPT4TS & TimesNet \\
\midrule
1sddb40 & \bfseries 0.935 & 0.030 & 0.540 & 0.190 & 0.680 & \bfseries 0.772 & 0.640 & 0.750 & 0.660 & 0.720 \\
BIDMC1 & \bfseries 1.000 & 0.990 & \bfseries 1.000 & \bfseries 1.000 & \bfseries 1.000 & 0.642 & 0.690 & 0.650 & 0.630 & \bfseries 0.740 \\
CHARISfive & 0.046 & 0.010 & \bfseries 0.130 & 0.020 & 0.080 & \bfseries 0.572 & 0.360 & 0.400 & 0.450 & 0.460 \\
CHARISten & \bfseries 0.851 & 0.020 & 0.110 & 0.100 & 0.030 & \bfseries 0.597 & 0.430 & 0.540 & 0.510 & 0.530 \\
CIMIS44AirTemperature3 & \bfseries 1.000 & 0.060 & 0.980 & 0.180 & 0.470 & \bfseries 0.843 & 0.640 & 0.750 & 0.620 & 0.740 \\
CIMIS44AirTemperature5 & \bfseries 1.000 & 0.390 & 0.990 & 0.200 & 0.710 & \bfseries 0.859 & 0.780 & 0.810 & 0.560 & 0.720 \\
ECG2 & \bfseries 1.000 & \bfseries 1.000 & \bfseries 1.000 & 0.900 & \bfseries 1.000 & 0.821 & 0.830 & \bfseries 0.840 & 0.780 & 0.600 \\
ECG3 & \bfseries 1.000 & 0.360 & 0.980 & 0.840 & 0.480 & \bfseries 0.808 & 0.540 & 0.770 & 0.450 & 0.610 \\
Fantasia & \bfseries 1.000 & 0.750 & 0.950 & 0.870 & 0.550 & \bfseries 0.786 & 0.730 & 0.640 & 0.650 & 0.610 \\
GP711MarkerLFM5z4 & \bfseries 1.000 & 0.930 & \bfseries 1.000 & 0.640 & 0.950 & \bfseries 0.886 & 0.540 & 0.730 & 0.620 & 0.720 \\
GP711MarkerLFM5z5 & \bfseries 1.000 & 0.760 & 0.970 & 0.480 & 0.900 & \bfseries 0.961 & 0.690 & 0.720 & 0.630 & 0.840 \\
InternalBleeding5 & \bfseries 1.000 & 0.940 & \bfseries 1.000 & 0.920 & \bfseries 1.000 & 0.932 & 0.460 & 0.690 & 0.630 & \bfseries 0.940 \\
Italianpowerdemand & 0.310 & 0.010 & \bfseries 0.740 & 0.010 & 0.440 & 0.709 & 0.450 & \bfseries 0.770 & 0.480 & 0.710 \\
Lab2Cmac011215EPG5 & \bfseries 1.000 & 0.990 & 0.980 & 0.600 & 0.990 & 0.739 & \bfseries 0.770 & 0.630 & 0.640 & 0.610 \\
Lab2Cmac011215EPG6 & 0.267 & \bfseries 0.410 & 0.100 & 0.100 & 0.170 & 0.554 & \bfseries 0.700 & 0.480 & 0.520 & 0.450 \\
MesoplodonDensirostris & \bfseries 1.000 & \bfseries 1.000 & 0.840 & \bfseries 1.000 & \bfseries 1.000 & 0.748 & \bfseries 0.850 & 0.720 & 0.690 & 0.790 \\
PowerDemand1 & \bfseries 0.994 & 0.870 & 0.440 & 0.760 & 0.950 & \bfseries 0.919 & 0.720 & 0.540 & 0.600 & 0.750 \\
TkeepFirstMARS & \bfseries 0.577 & 0.010 & 0.150 & 0.020 & 0.230 & 0.728 & 0.520 & 0.760 & 0.500 & \bfseries 0.790 \\
TkeepSecondMARS & \bfseries 1.000 & 0.830 & \bfseries 1.000 & 0.120 & 0.950 & \bfseries 0.989 & 0.720 & 0.910 & 0.810 & 0.980 \\
WalkingAceleration5 & 0.967 & 0.990 & \bfseries 1.000 & 0.870 & 0.930 & \bfseries 0.968 & 0.940 & 0.870 & 0.910 & 0.850 \\
apneaecg & \bfseries 0.814 & 0.400 & 0.200 & 0.310 & 0.260 & 0.608 & 0.580 & 0.690 & 0.580 & \bfseries 0.760 \\
apneaecg2 & \bfseries 1.000 & 0.650 & \bfseries 1.000 & \bfseries 1.000 & 0.650 & \bfseries 0.845 & 0.790 & 0.740 & 0.650 & 0.610 \\
gait1 & \bfseries 1.000 & 0.180 & 0.360 & 0.410 & 0.520 & \bfseries 0.887 & 0.630 & 0.570 & 0.580 & 0.600 \\
gaitHunt1 & \bfseries 0.596 & 0.080 & 0.430 & 0.100 & 0.300 & \bfseries 0.847 & 0.810 & 0.680 & 0.710 & 0.840 \\
insectEPG2 & \bfseries 0.962 & 0.120 & 0.230 & 0.810 & 0.960 & \bfseries 0.871 & 0.650 & 0.820 & 0.560 & 0.730 \\
insectEPG4 & 0.513 & 0.980 & \bfseries 1.000 & 0.210 & 0.850 & 0.691 & 0.690 & \bfseries 0.720 & 0.490 & 0.650 \\
ltstdbs30791AS & \bfseries 1.000 & \bfseries 1.000 & \bfseries 1.000 & \bfseries 1.000 & \bfseries 1.000 & \bfseries 0.959 & 0.780 & 0.810 & 0.740 & 0.670 \\
mit14046longtermecg & \bfseries 0.676 & 0.450 & 0.590 & 0.580 & 0.600 & 0.661 & 0.790 & 0.660 & 0.610 & \bfseries 0.840 \\
park3m & \bfseries 1.000 & 0.150 & 0.640 & 0.630 & 0.930 & \bfseries 0.875 & 0.630 & 0.780 & 0.540 & 0.780 \\
qtdbSel1005V & \bfseries 0.844 & 0.410 & 0.650 & 0.390 & 0.530 & 0.612 & 0.520 & \bfseries 0.640 & 0.610 & 0.540 \\
qtdbSel100MLII & \bfseries 1.000 & 0.420 & 0.840 & 0.600 & 0.870 & 0.573 & 0.620 & 0.620 & 0.580 & \bfseries 0.650 \\
resperation1 & \bfseries 0.308 & 0.000 & 0.150 & 0.010 & 0.030 & 0.725 & \bfseries 0.750 & 0.670 & 0.470 & 0.670 \\
s20101mML2 & \bfseries 1.000 & 0.690 & 0.710 & 0.050 & 0.080 & \bfseries 0.942 & 0.640 & 0.720 & 0.640 & 0.690 \\
sddb49 & \bfseries 1.000 & 0.890 & \bfseries 1.000 & 0.940 & \bfseries 1.000 & \bfseries 0.937 & 0.660 & 0.730 & 0.580 & 0.680 \\
sel840mECG1 & \bfseries 0.984 & 0.160 & 0.660 & 0.210 & 0.360 & 0.702 & 0.620 & \bfseries 0.720 & 0.650 & 0.600 \\
sel840mECG2 & \bfseries 0.984 & 0.150 & 0.390 & 0.280 & 0.210 & 0.683 & 0.590 & \bfseries 0.690 & 0.520 & 0.520 \\
tilt12744mtable & \bfseries 0.254 & 0.070 & 0.240 & 0.000 & 0.030 & \bfseries 0.761 & 0.480 & 0.740 & 0.510 & 0.640 \\
tilt12754table & 0.131 & 0.230 & \bfseries 0.640 & 0.060 & 0.050 & \bfseries 0.855 & 0.600 & 0.820 & 0.550 & 0.750 \\
tiltAPB2 & \bfseries 1.000 & 0.920 & 0.980 & 0.830 & 0.380 & \bfseries 0.844 & 0.770 & 0.770 & 0.600 & 0.700 \\
tiltAPB3 & 0.148 & 0.170 & \bfseries 0.850 & 0.050 & 0.090 & \bfseries 0.769 & 0.680 & 0.650 & 0.440 & 0.580 \\
weallwalk & \bfseries 0.706 & 0.000 & 0.580 & 0.130 & 0.170 & 0.849 & 0.730 & \bfseries 0.930 & 0.870 & 0.850 \\
\midrule
Mean & \bfseries 0.802 & 0.475 & 0.684 & 0.449 & 0.570 & \bfseries 0.789 & 0.659 & 0.711 & 0.605 & 0.695 \\
Median & \bfseries 0.994 & 0.410 & 0.740 & 0.390 & 0.550 & \bfseries 0.808 & 0.660 & 0.720 & 0.600 & 0.700 \\
Std & 0.300 & 0.379 & 0.321 & 0.358 & 0.355 & 0.122 & 0.124 & 0.106 & 0.107 & 0.118 \\
\bottomrule
\end{tabular}}
\label{tab:anomaly_detection}
\end{table}

\clearpage

\section{Insights on Modeling Time Series as Images}
\label{section: time series as images}
TiViT surpasses the performance of TSFMs in time series tasks by leveraging pretrained ViTs. This raises a key question: is its success solely due to the rich representations learned from billions of natural images, or is there an inherent advantage of the 2D patching strategy as well? 
We develop a theoretical insight at the patch level showing how the 2D representation of time series can enhance the classification performance of Transformer models. To empirically validate this, we compare the performance of Transformers pretrained on real-world data using 1D versus 2D patching.

\subsection{Theoretical Analysis of 1D and 2D Patching}
\label{subsection: theoretical motivation}

We consider a binary time series classification problem with \(N\) univariate training samples \(\{({\bm{t}}^n, y^n), y^n \in \{+1, -1\}\}_{n=1}^N\). 
Each time series $\bm{t}^n \in \mathbb{R}^T$ can be patched as follows:

\begin{figure}[!t]
\centering
    \begin{minipage}[b]{0.65\linewidth}
        \includegraphics[width=\linewidth]{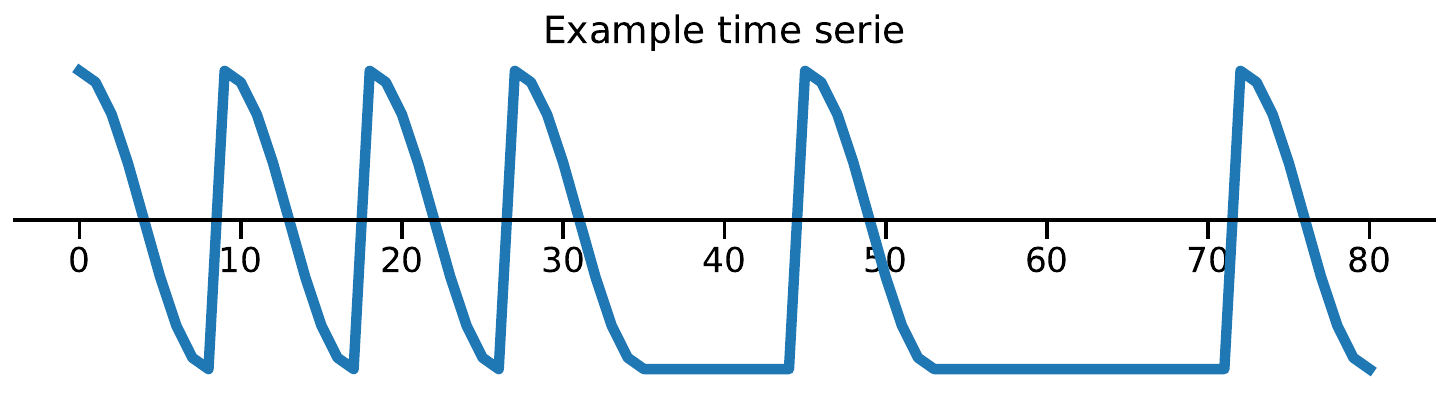}
        \includegraphics[width=\linewidth]{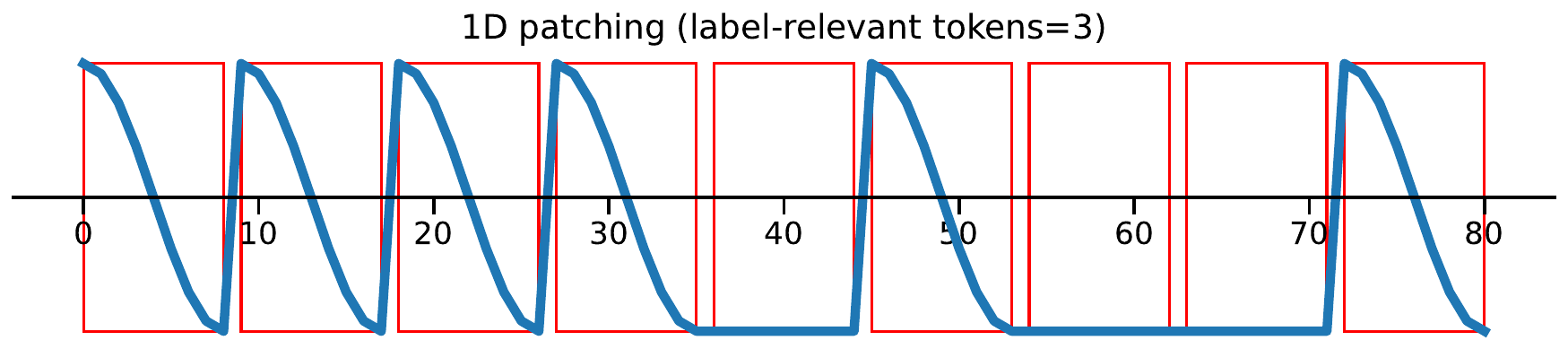}
    \end{minipage}
    \begin{minipage}[b]{0.3\linewidth}
        \includegraphics[width=0.9\linewidth]{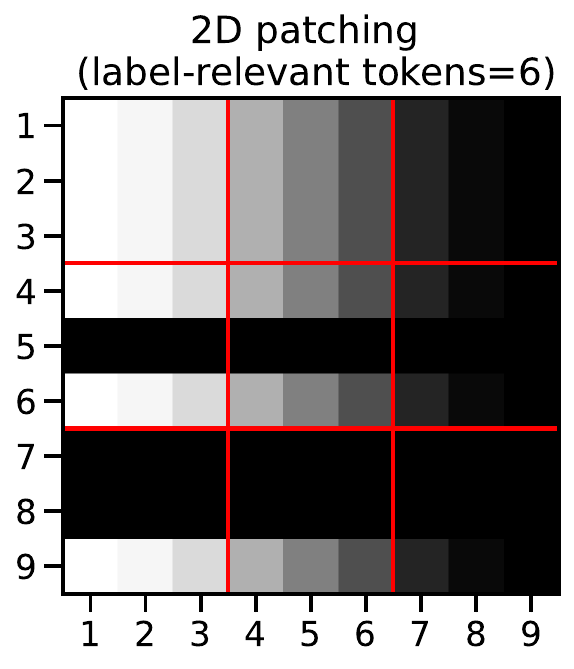}
    \end{minipage}
\caption{2D patching yields a higher number of label-relevant tokens (with constant negative signal) than 1D patching. This facilitates time series classification with Transformers.}
\label{fig:prop}
\end{figure}

\begin{itemize}
    \item 1D patching: The series $\bm{t}$ is split into $k$ contiguous, non-overlapping tokens $\bm{x}_l \in \mathbb{R}^k$.
    \item 2D patching: The series $\bm{t}$ is reshaped into a $k \times k$ matrix, then divided into $k$ non-overlapping $\sqrt{k} \times \sqrt{k}$ patches, which are flattened to form tokens $\bm{x}'_{(i,j)} \in \mathbb{R}^k$.
\end{itemize}

This setup ensures the same number of tokens for 1D and 2D patching. Our analysis builds on the notion of label-relevant tokens introduced by Li et al.~\cite{li2023a} (see Section \ref{app:sec_theory_background}). Following their data model, we consider each token to be a noisy version of distinct patterns. In binary classification, there exist two such patterns \(\{\bm{\mu}_1, \bm{\mu}_2\}\), $\bm{\mu}_i \in \mathbb{R}^{k}, \forall\ i$. For a time series $\bm{t}^n$ with label $y^n = 1$, tokens $\bm{x}$ that are noisy $\bm{\mu}_1$, i.e., $||\bm{x}-\bm{\mu_1}||\leq ||\bm{x}-\bm{\mu_2}||$, are label-relevant. Similarly, for a time series $\bm{t}^n$ with label $y^n = -1$, the noisy versions of $\bm{\mu}_2$ are label-relevant.

\myparagraph{Benefits of 2D patching.} Li et al.~\cite{li2023a} showed that the sample complexity of a Transformer scales as $\mathcal{O}(1/\alpha_*^2)$ where $\alpha_*$ denotes the fraction of label-relevant tokens in the training samples. In Section~\ref{app:sec_theory_proof}, we provide a constructive proof showing that under certain conditions, this fraction of label-relevant tokens is greater when the time series is transformed into a 2D representation compared to the conventional 1D representation. Therefore, 2D patching can lead to more efficient learning with Transformers than 1D patching.
Figure~\ref{fig:prop} illustrates our idea for an exemplary time series with $T=91$ and $k=9$. We set $\bm{\mu}_1=\cos(x)$ for $x \in [0,\pi]$ and define the label-relevant signal as $\bm{\mu}_2=-1$. In the 1D case, only three tokens carry the label-relevant information, whereas in the 2D case there are six such tokens. Following Li et al.~\cite{li2023a}, distributing the discriminative signal across a larger number of tokens makes it easier for a Transformer to detect and leverage it.

Below, we review the shallow ViT and data model introduced by Li et al.~\cite{li2023a} in their theoretical analysis of training a ViT. Their Theorem~\ref{trm:iclr2023} shows that the sample complexity for ViTs to achieve a zero generalization error is inversely correlated with the fraction of label-relevant tokens.
Building on this insight, we introduce and prove Proposition 1, showing that 2D patching can increase the number of label-relevant tokens compared to 1D patching.

\subsection{Theoretical Background}
\label{app:sec_theory_background}
\myparagraph{Model and setup.} Following the setup of Li et al.~\cite{li2023a}, we study a binary classification problem with \(N\) training samples \(\{({\bm{X}}^n, y^n)\}_{n=1}^N\). Each input \({\bm{X}}^n \in \mathbb{R}^{d \times L}\) contains \(L\) tokens $\{{\bm{x}}_1^n, \dots, {\bm{x}}_L^n\}$. Labels $y^n \in \{\pm 1\}$ are determined by majority vote over discriminative tokens. A simplified Vision Transformer (ViT)~\cite{dosovitskiy2021an} model is defined as:
\[
F({\bm{X}}^n) = \frac{1}{|\mathcal{S}^n|} \sum_{l \in \mathcal{S}^n} \bm{a}_{(l)}^{\top} \text{ReLU}\left({\bm{W}}_O {\bm{W}}_V {\bm{X}}^n \text{softmax}\left({\bm{X}}^{n^\top} {\bm{W}}_K^\top {\bm{W}}_Q {\bm{x}}_l^n\right)\right),
\]
where \(\psi = ({\bm{A}}\!=\!\{\bm{a}_{(l)}\}_l, {\bm{W}}_O, {\bm{W}}_V, {\bm{W}}_K, {\bm{W}}_Q)\) are trainable parameters. The empirical risk minimization problem is:
\[
\min_{\psi} f_N(\psi) = \frac{1}{N} \sum_{n=1}^N \max\left\{1 - y^n \cdot F({\bm{X}}^n), 0\right\}.
\]
Training uses mini-batch SGD with fixed output layer weights \({\bm{A}}\), following standard NTK initialization practices.

\myparagraph{Data model.}
Tokens \({\bm{x}}_l^n\) are noisy versions of \(M\) patterns \(\{\bm{\mu}_1, \dots, \bm{\mu}_M\}\), where \(\bm{\mu}_1, \bm{\mu}_2\) are discriminative. Label \(y^n\) depends on majority vote over tokens closest to \(\bm{\mu}_1/\bm{\mu}_2\). Noise level \(\tau\) satisfies \(\tau < \kappa/4\), with \(\kappa - 4\tau = \Theta(1)\).

\myparagraph{Generalization of ViT.}
We now recap the main results from Li et al.~\cite{li2023a} from which we derive our result, along with the main notations in Table \ref{tab:notations}.
\begin{assumption}[Initial Model Conditions, \cite{li2023a}]
Initial weights ${\bm{W}}_V^{(0)}$, ${\bm{W}}_K^{(0)}$, ${\bm{W}}_Q^{(0)}$ satisfy:
\[
\|{\bm{W}}_V^{(0)}\bm{\mu}_j - {\bm{p}}_j\| \leq \sigma, \quad \|{\bm{W}}_K^{(0)}\bm{\mu}_j - {\bm{q}}_j\| \leq \delta, \quad \|{\bm{W}}_Q^{(0)}\bm{\mu}_j - {\bm{r}}_j\| \leq \delta,
\]
for orthonormal bases \(\mathcal{P}, \mathcal{Q}, \mathcal{R}\) and \(\sigma = O(1/M), \delta < 1/2\).
\end{assumption}

\begin{theorem}[Generalization of ViT, \cite{li2023a}]
\label{trm:iclr2023}
Under Assumption 1, with sufficient model width \(m \gtrsim \epsilon^{-2}M^2\log N\), fraction 
$$\alpha_* \geq \alpha_\#/(\epsilon_S e^{-(\delta+\tau)}(1 - (\sigma + \tau)),$$ 
and sample size 
$$N \geq \Omega\left((\alpha_* - c'(1 - \zeta) - c''(\sigma + \tau))^{-2}\right),$$ 
SGD achieves zero generalization error after 
$$T = \Theta\left(\frac{1}{(1 - \epsilon - (\sigma + \tau)M/\pi)\eta \alpha_*}\right)$$ iterations.
\end{theorem}

\begin{proposition}[Generalization without Self-Attention, \cite{li2023a}]
Without self-attention, achieving zero error requires \(N \geq \Omega\left((\alpha_*(\alpha_* - \sigma - \tau))^{-2}\right)\), demonstrating ViT's sample complexity reduction by \(1/\alpha_*^2\).
\end{proposition}

\begin{table}[!t]
\centering
\caption{Key Notations}\label{tab:notations}
\begin{tabular}{ll}
\toprule
Notation & Description \\
\midrule
\(\alpha_*\) & Fraction of label-relevant tokens \\
\(\sigma, \delta, \tau\) & Initialization/token noise parameters \\
\(\kappa\) & Minimum pattern distance \\
\(M\) & Total number of patterns \\
\bottomrule
\end{tabular}
\end{table}

\subsection{Proof of Label Relevance in 2D Patches}

We introduce Proposition 1 that formalizes our theoretical analysis of 1D and 2D patching from Section \ref{subsection: theoretical motivation} and provide a detailed proof.

\setcounter{proposition}{0}

\label{app:sec_theory_proof}
\begin{proposition}
For an arbitrary $\bm{\mu}_1, \bm{\mu}_2 \in \mathbb{R}^k$, let $\bm{t} = [\bm{x}_1 \ \ \bm{x}_2 \ \ \cdots \ \ \bm{x}_k ]^\top \in \mathbb{R}^T \text{where } \forall i \in [k], \bm{x}_i \in \mathbb{R}^k$ and either $\bm{x}_i = \bm{\mu}_1$ or $\bm{x}_i = \bm{\mu}_2$ with $\bm{\mu}_2$ being a label-relevant pattern. Let $\left| \left\{ i : \bm{x}_i = \boldsymbol{\mu}_2 \right\} \right| = n'$ and assume that $2\bm{x}'\cdot (\bm{\mu}_1 - \bm{\mu}_2) \leq ||\bm{\mu}_1 ||^2 - ||\bm{\mu}_2||^2$ whenever $\left| \left\{ i : x'_i \in \boldsymbol{\mu}_2 \right\} \right| \geq \sqrt{k}$. Then, it holds that  $$\alpha_*^{\text{2D}} \geq \alpha_*^{\text{1D}}=\frac{n'}{k},$$ and the inequality is strict if $n' \text{ mod } \sqrt{k} > 0.$
\end{proposition}

\begin{proof}
For a token \(\bm{x}'^n\) to be label-relevant (aligned with \(\bm{\mu}_2\)), it must satisfy:
\[
\|\bm{x}'^n - \bm{\mu}_2\| \leq \|\bm{x}'^n - \bm{\mu}_1\|.
\]

Expanding both sides, we have that:
\[
\|\bm{x}'^n\|^2 + 2\bm{x}'^n \cdot \bm{\mu}_1 + \|\bm{\mu}_1\|^2 \leq \|\bm{x}'^n\|^2 - 2\bm{x}'^n \cdot \bm{\mu}_2 + \|\bm{\mu}_2\|^2.
\]
Regrouping the terms gives us the desired condition:
\[
2\bm{x}'^n \cdot (\bm{\mu}_1 - \bm{\mu}_2) \leq ||\bm{\mu}_1||^2 - ||\bm{\mu}_2||^2. \tag{1}
\]

Recall that $n'$ denotes the number of segments of $\bm{\mu}_2$ in time series $\bm{t}$. Each such segment spans $\sqrt{k}$ tokens, contributing at least $\sqrt{k}$ elements to each of them. Under the assumption of the proposition, it implies (1) and makes each of these $\sqrt{k}$ tokens label-relevant. 

We now need to carefully consider how the $\bm{\mu}_2$ segments can be placed within $\bm{t}$ to understand how many tokens become label-relevant thanks to each $\bm{\mu}_2$. We consider two cases: 1) $n' = c\sqrt{k}$ for some $c \in \mathbb{N}$ satisfying $n' \in (0,k]$, and 2) $n' = c\sqrt{k} + b$ for some $a,b \in \mathbb{N}$, $\sqrt{k}>b>0$ such that $n' \in (0,k]$. In the first case, $\alpha_*^{\text{1D}}=c\sqrt{k}/k$. In the case of 2D patching, in the worst case, $\bm{\mu}_2$ segments can be placed such that they will contribute to $c\sqrt{k}$ tokens. In this case, $\alpha_*^{\text{2D}}\geq c\sqrt{k}/k$ and $\alpha_*^{\text{1D}} \leq \alpha_*^{\text{2D}}$. If $n'$ is not a multiple of $\sqrt{k}$, the same analysis applies for the $c\sqrt{k}$ segments of $\bm{\mu}_2$. To account for the remainder $b$, we note that for any $b>0$, in 2D case, it adds $\sqrt{k}$ label-relevant tokens to the fraction $\alpha_*^{\text{2D}}$ so that $\alpha_*^{\text{2D}}\geq \frac{c\sqrt{k}+\sqrt{k}}{k}$. In the case of 1D patching, $\alpha_*^{\text{1D}} = \frac{c\sqrt{k}+b}{k}$. Given that $b<\sqrt{k}$, this concludes the proof. 
\end{proof} 

\label{app:sec_theory_examples}
To illustrate the benefits of 2D modeling and patching, we present several examples of time series in Figure \ref{fig:ex_prop1}. We define $\bm{\mu}_1$ using functions such as log, cosine, and sine. We then set $\bm{\mu}_2=\bm{1}_k$, $n'=3$ and randomly shuffle $\bm{\mu}_1$ and $\bm{\mu}_2$ segments within the generated input time series.

\subsection{Interpretability Scores}

We illustrate the spread of information achieved with 2D modeling on samples from a real-world dataset from the UCR repository. In particular, we show that a model trained on 2D representations of time series has more regions that it deems relevant for predicting the class membership of the time series. To this end, we follow \cite{Early2024} and use MILLET: a framework that provides interpretability scores for timestamps within a time series given a pretrained model. For this, we train two shallow ViTs, ViT1D and ViT2D, on the BirdChicken dataset from the UCR repository. ViT1D takes as input a raw 1D time series, while ViT2D is trained on square 2D images of the time series. The only difference between these two models is their patching strategy: ViT1D patches the time series using a 1D convolutional filter, while ViT2D applies a 2D convolutional filter. The obtained results for two samples from different classes are presented in Figure~\ref{fig:millet_1d_2d}, for 1D (left) and 2D (right) cases, respectively.
Note that the 1D heatmap for the sample with ground-truth class 0 highlights discriminative signals for class 1 at the beginning and end of the time series, while the corresponding 2D heatmap displays no such signals in these areas. The interpretability scores w.r.t. the ground-truth class of a sample are generally more homogeneous for ViT2D which facilitates classification.

\begin{figure}[t]
\centering
\includegraphics[width=0.45\linewidth]{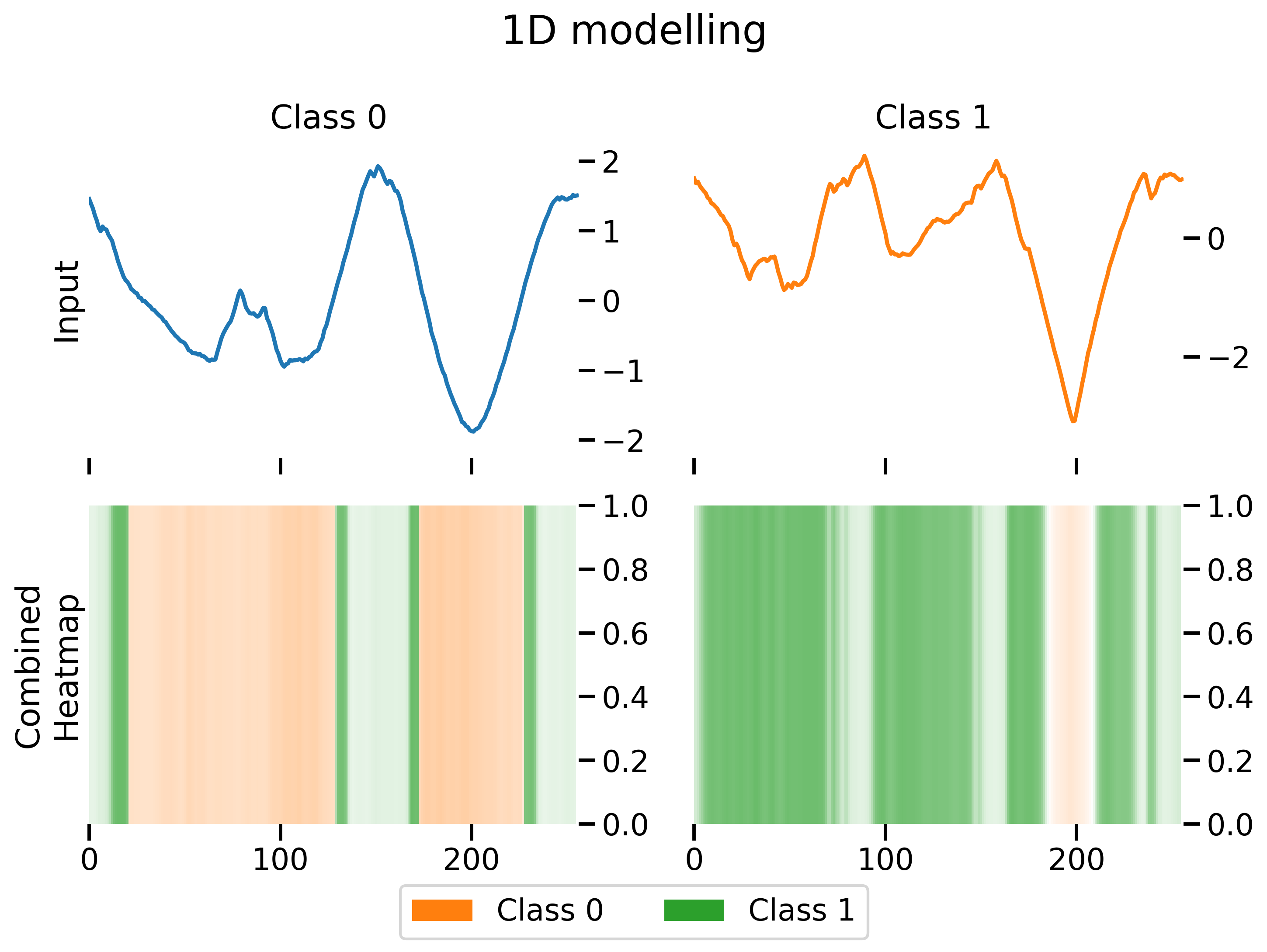}
\hspace{1cm}
\includegraphics[width=0.37\linewidth]{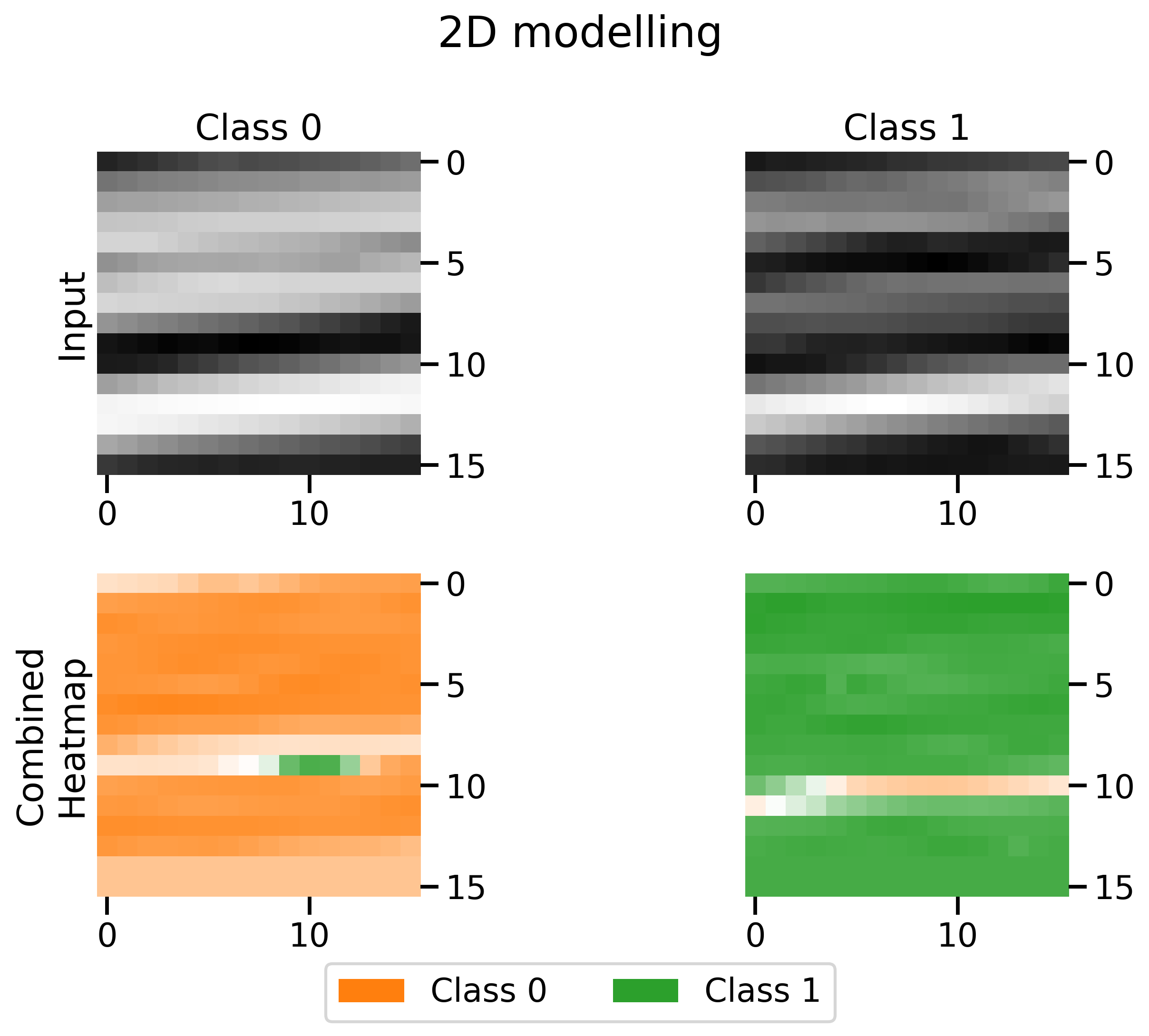}
\caption{Comparison of interpretability heatmaps for two ViTs trained on 1D (left) and 2D (right) representations of time series from the BirdChicken dataset, respectively. The interpretability scores for the correct class of a sample are more homogeneous in the 2D case.}
\label{fig:millet_1d_2d}
\end{figure}

\subsection{Pretraining Transformers with 1D and 2D Patching}

\setlength\intextsep{2pt}
\begin{wraptable}[10]{r}{0.42\linewidth}
\centering
\caption{Evaluation of models pretrained with different patching strategies on UCR.}
\label{tab:matrix-patching-exp-res}
\begin{tabular}{@{}lcccc@{}}
\toprule
\multirow{2}{*}{Patching} & \multicolumn{2}{c}{Non-overlap} & \multicolumn{2}{c}{Overlap} \\
\cmidrule(lr){2-3}
\cmidrule(lr){4-5}
& 1D & 2D & 1D & 2D \\
\midrule
Accuracy & 76.4 & 76.8 & 76.6 & \textbf{77.4} \\
\bottomrule
\end{tabular}
\end{wraptable}

To validate our hypothesis that 2D patching improves representation learning for time series, we study the effect of different patching strategies while keeping the Transformer architecture and pretraining objective fixed. 

\myparagraph{Model.} We adopt the setup of Feofanov et al.~\cite{feofanov2025mantis}, whose Transformer block implementation is closely related to the classical ViT (ViTUnit class \href{https://github.com/vfeofanov/mantis/blob/main/src/mantis/architecture/architecture.py}{here}). Specifically, the model consists of 6 Transformer layers, each with 8 attention heads and an embedding dimension of 256. We pretrain all models with contrastive learning following \cite{feofanov2025mantis,he2020momentum}. The augmentation technique to generate positive pairs is RandomCropResize with a crop rate varying within $[0,0.2]$. All time series are resized to a fixed length $T = 512$ using interpolation.

\myparagraph{1D and 2D patching.} We compare 1D and 2D patching in both non-overlapping and overlapping settings. For non-overlapping 1D patching, we generate 32 patches of size 16. For non-overlapping 2D patching, we first arrange the 1D patches in a matrix of size $32 \times 16$ and then extract 32 patches of size $2 \times 8$. After flattening, we obtain 32 patches of size 16, similar to the 1D setting, but semantically different. For overlapping 1D patching, we apply a stride of 8, which yields 64 patches of size 16. For overlapping 2D patching, we rearrange these 1D patches again in a matrix of size $64 \times 16$ and then extract 32 patches of size $4 \times 8$. Flattening yields 32 patches of size 32.

\myparagraph{Data.} For pretraining, we construct a corpus of 100{,}000 samples from publicly available time series datasets that do not overlap with the evaluation benchmark. Specifically, the corpus is assembled from ECG~\cite{clifford2017ecgdataset}, EMG~\cite{goldberger2000emgdataset}, Epilepsy~\cite{andrzejak2001epilepsydataset}, FD-A and FD-B~\cite{lessmeier2016fdafdbdataset}, Gesture~\cite{liu2009gesturedataset}, HAR~\cite{anguita2013hardataset}, and SleepEEG~\cite{kemp2000sleepeegdataset}. To reduce computational cost, we use a balanced subset of the full collection. Detailed sample counts per source dataset are reported in Table~\ref{tab:pretrainig-data-tsfm}.

\myparagraph{Classification results.} We evaluate the learned representations on the UCR benchmark. As shown in Table~\ref{tab:matrix-patching-exp-res}, 2D patching consistently outperforms 1D patching, with overlapping 2D patches achieving the highest classification accuracy. These results indicate that transforming time series into image-like structures is beneficial not only when leveraging pretrained ViTs, but also when pretraining Transformers from scratch on time series data.

\begin{table}
    \centering
    \caption{Data used to pretrain Transformers for comparison of 1D and 2D patching.}
    \begin{tabular}{lccc}
    \toprule
     Dataset & Number of examples & Prop. of taken examples \\
    \midrule
    ECG & 20835 & $45.7\%$ \\
    EMG & 163 & $100\%$ \\
    Epilepsy & 11480 & $100\%$ \\
    FD-A & 10912 & $100\%$ \\
    FD-B & 13619 & $100\%$ \\
    Gesture & 1320 & $100\%$ \\
    HAR & 20835 & $78.7\%$ \\
    SleepEEG & 20836 & $4.5\%$ \\
    \bottomrule
    \end{tabular}
    \label{tab:pretrainig-data-tsfm}
\end{table}

\begin{figure}[t]
\centering
\begin{minipage}[b]{.65\linewidth}
\includegraphics[width=\linewidth]{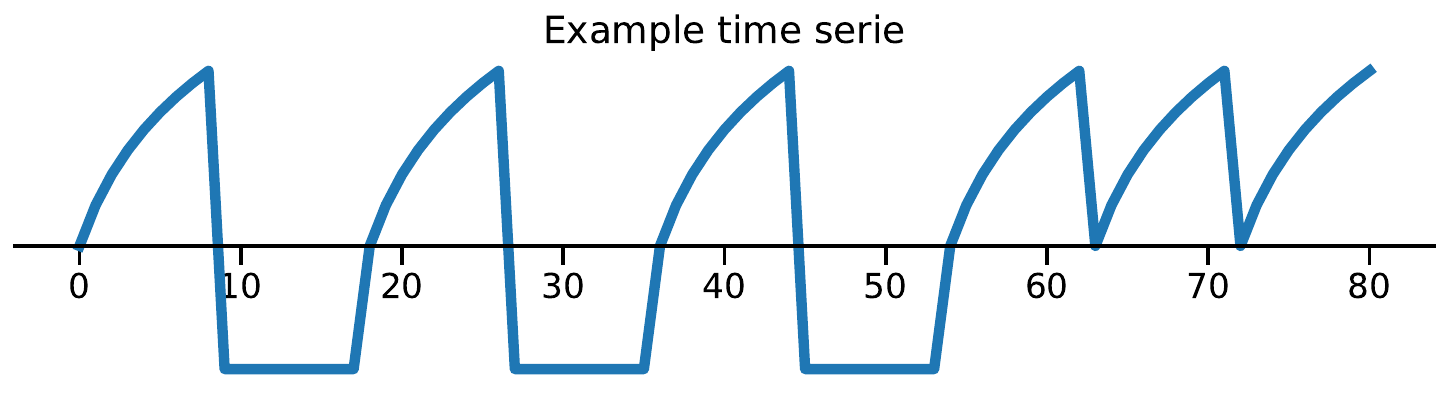}
\includegraphics[width=\linewidth]{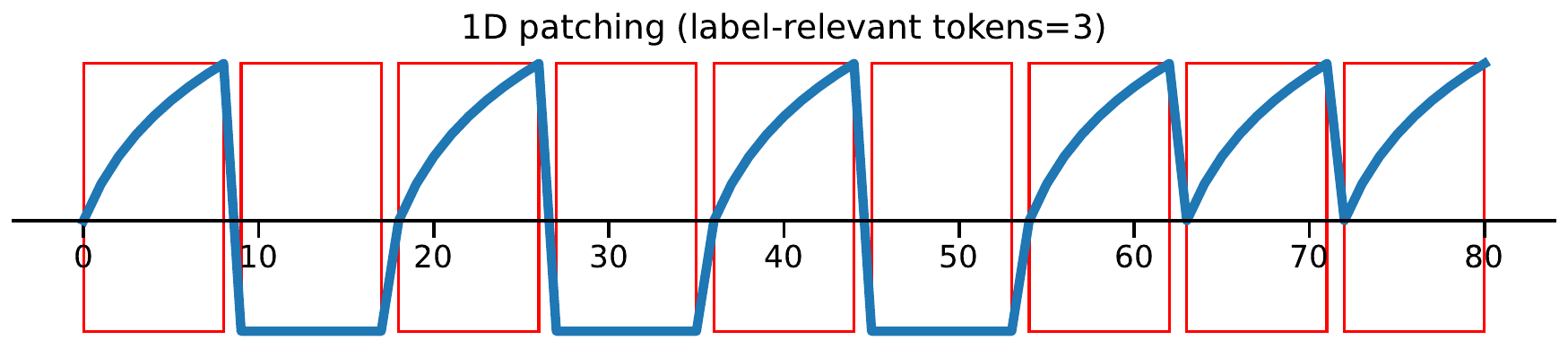}
\end{minipage}
\begin{minipage}[b]{.3\linewidth}
\includegraphics[width=.9\linewidth]{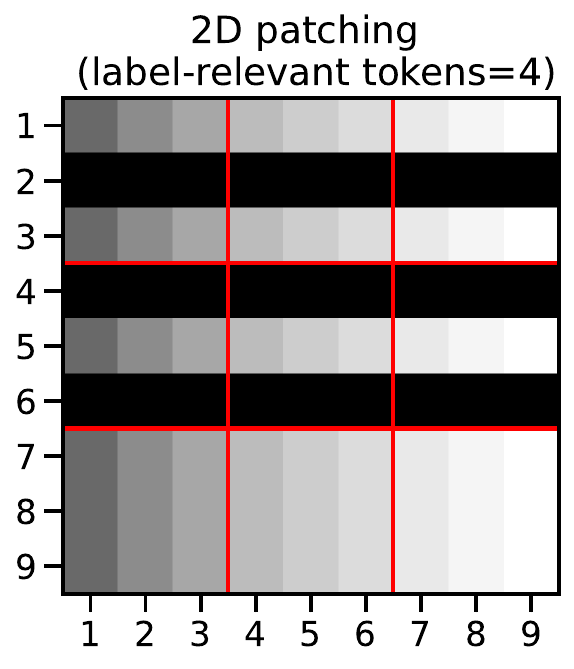}
\end{minipage}
\vspace{.5cm}

\begin{minipage}[b]{.65\linewidth}
\includegraphics[width=\linewidth]{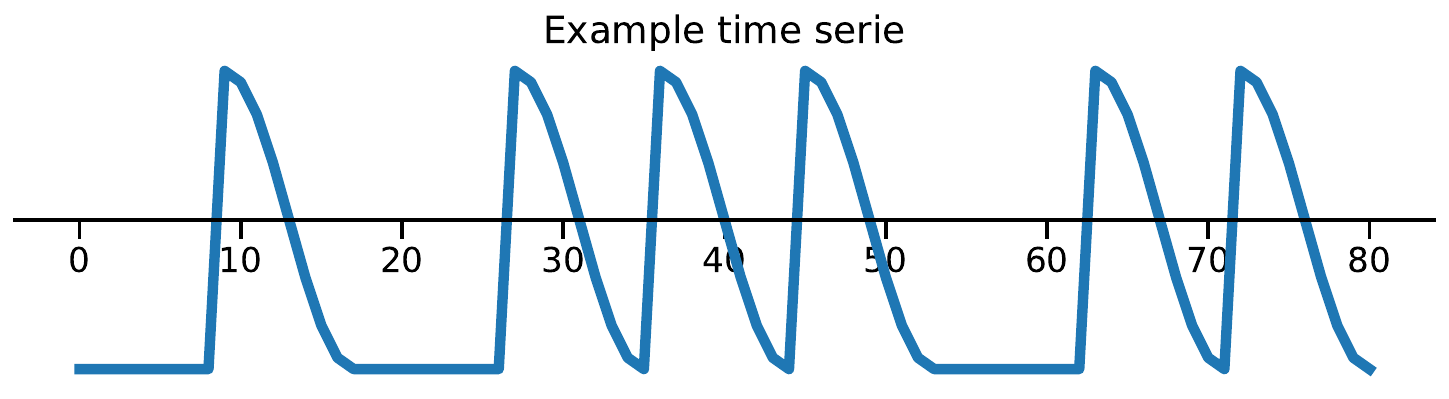}
\includegraphics[width=\linewidth]{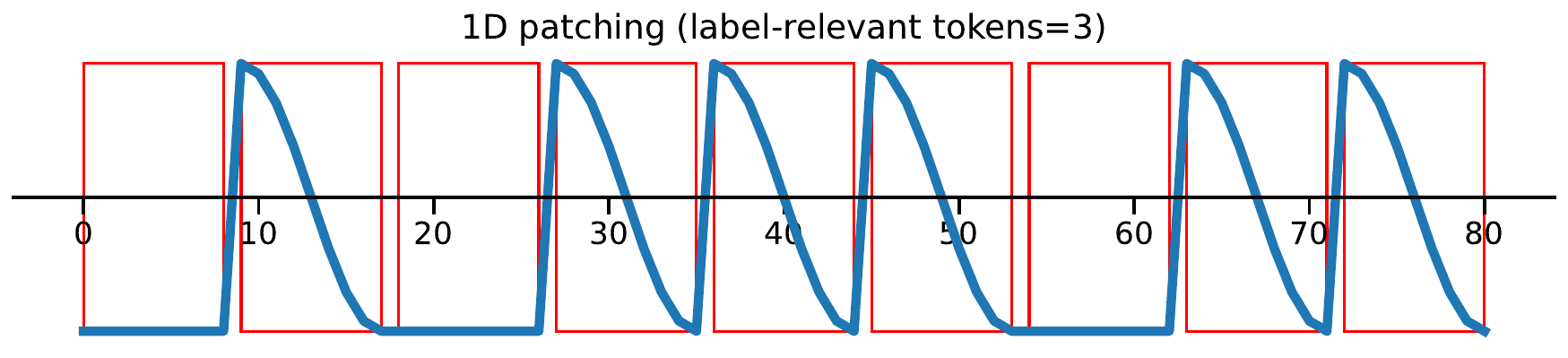}
\end{minipage}
\begin{minipage}[b]{.3\linewidth}
\includegraphics[width=.9\linewidth]{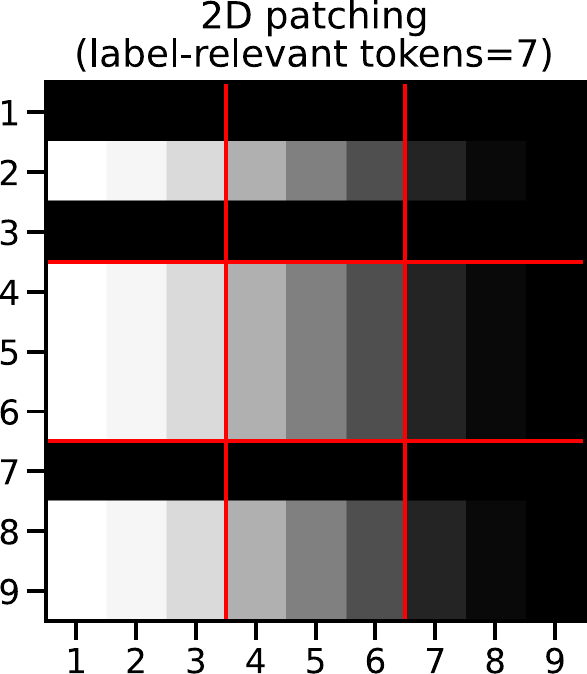}
\end{minipage}
\vspace{.5cm}

\begin{minipage}[b]{.65\linewidth}
\includegraphics[width=\linewidth]{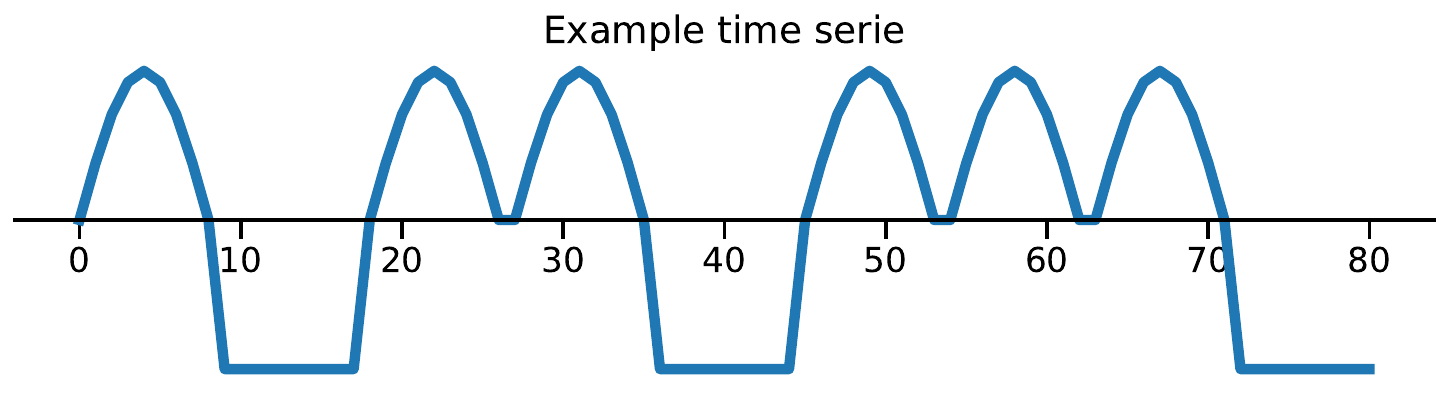}
\includegraphics[width=\linewidth]{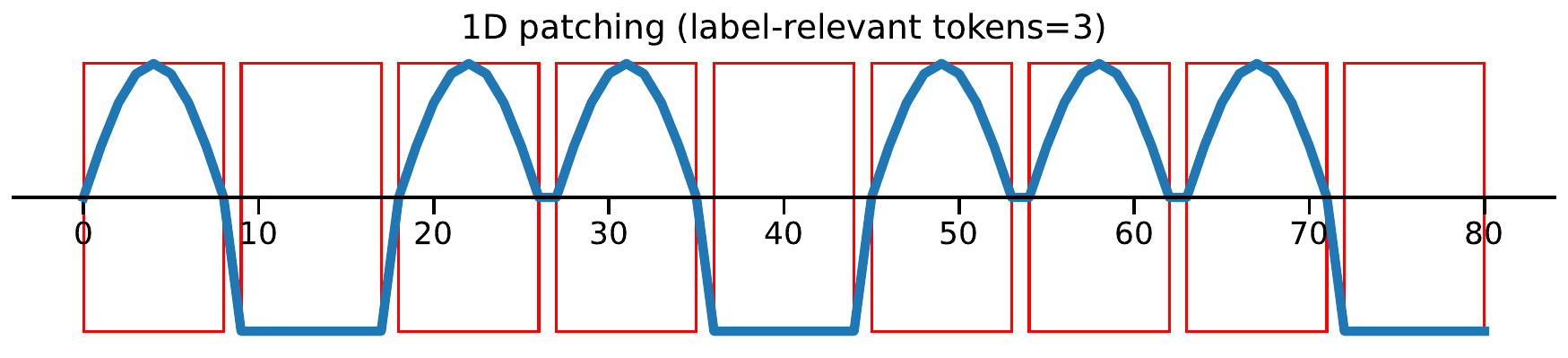}
\end{minipage}
\begin{minipage}[b]{.3\linewidth}
\includegraphics[width=.95\linewidth]{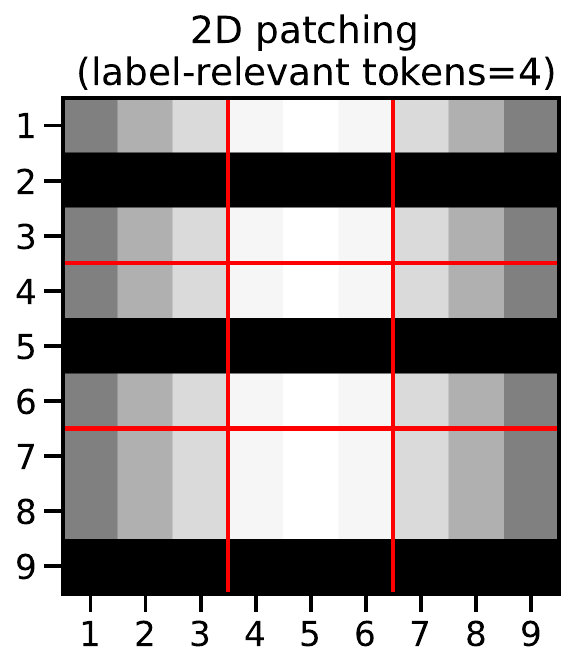}
\end{minipage}
\vspace{.5cm}

\caption{Illustration of Proposition 1 on more generated time series. In each example considered, 2D patching is more beneficial due to the higher number of label-relevant tokens.}
\label{fig:ex_prop1}
\end{figure}

\section{Broader Impact}
\label{app: broader impacts}
Since this paper presents foundational machine learning research, we do not see any direct societal risks. The broader impact of our work will depend on its specific application.

We demonstrate that our method TiViT significantly improves classification accuracy. This advancement can be beneficial in healthcare where the analysis of physiological signals is crucial for early diagnosis and treatment or in industry where the accurate monitoring of sensor data enables predictive maintenance and reduces downtime.

However, deep learning models including TiViT operate as black boxes with limited interpretability. In safety-critical domains or applications directly impacting humans, such models necessitate careful deployment and oversight. Further research into interpretability and human-in-the-loop frameworks is essential to make deep learning models trustworthy for real-world settings.

\end{document}